\documentclass[journal]{IEEEtran}
\usepackage{times}
\usepackage{epsfig}
\usepackage{graphbox} 

\usepackage{amsmath}
\usepackage{amssymb}
\usepackage{amsthm}

\newtheorem{theorem}{Theorem}[section]
\newtheorem{lemma}[theorem]{Lemma}
\newtheorem{proposition}{Proposition}[section]

\ifCLASSINFOpdf
\else
\fi
\begin{document}
%
\title{Understanding the Distributions of Aggregation Layers in Deep Neural Networks}
%
%
%

\author{Eng-Jon~Ong,
        Sameed~Husain,
        and~Miroslaw~Bober,
\thanks{E. Ong, S. Husain and M. Bober is with the Centre for Vision, Speech and Signal Processing, University of Surrey, Guidlford GU27XH, UK. e-mail: e.ong,s.husain,m.bober@surrey.ac.uk}
}
\maketitle

\begin{abstract}
The process of aggregation is ubiquitous in almost all deep nets models. It functions as an important mechanism for consolidating deep features into a more compact representation, whilst increasing robustness to overfitting and providing spatial invariance in deep nets. In particular, the proximity of global aggregation layers to the output layers of DNNs mean that aggregated features have a direct influence on the performance of a deep net. A better understanding of this relationship can be obtained using information theoretic methods. However, this requires the knowledge of the distributions of the activations of aggregation layers. 
To achieve this, we propose a novel mathematical formulation for analytically modelling the probability distributions of output values of layers involved with deep feature aggregation. 
An important outcome is our ability to analytically predict the KL-divergence of output nodes in a DNN. We also experimentally verify our theoretical predictions against empirical observations across a range of different classification tasks and datasets.
\end{abstract}


%
\IEEEpeerreviewmaketitle

\section{Introduction}
In recent years, deep neural networks (DNN) have shown excellent performance across a wide range of tasks, for example, image classification and retrieval.
One element that is almost ubiquitous in state-of-the-art DNN architectures is the use of aggregation processes, where deep features are combined together into a more compact representation, with the intuition that stronger (i.e. higher valued) features represent important information. 
The aggregation process can occur from the local level to global levels. One of the earliest use of aggregation was by Fukushima \cite{Fukushima1980}, where average pooling was introduced to reduce the size of convolutional images. Later, Weng et al. \cite{maxpool_weng} introduced the max-pooling layer in the Cresceptron framework, whereby activation values of a patch is replaced by its maximum value. The aggregation process was then extended by Lin et al. \cite{lin2014network} to act on a more global level, where the all the convolutional activations of a filter were replaced by a single aggregated average value. This brought the advantage of improved robustness to overfitting and spatial translations of the input. Greater emphasis can be placed on strong features in the aggregation process by using more generalised activation pooling methods \cite{GEM,HusainPAMI}, where the global average pooling layer was sandwiched between two non-linear transformations. This was found to have significantly improved the accuracy of the DNN when applied on image retrieval tasks. 

It can be argued that the pooling process, is as important as the feature extracting convolutional filters in a deep net. This applies especially to the global pooling since it determines the separability of features that ultimately defines the accuracy of a DNN. As an example, it was found in \cite{HusainPAMI} that the retrieval or classification performance of DNNs correlates strongly with an increase in the KL divergence of its outputs.
%
In order to better understand the usefulness of aggregation, we need to understand the quality of the aggregated features. One way of achieving this is by using information theoretic methods such as KL divergence. In order to do this, we need knowledge about the distributions of activations throughout the aggregation process. To our knowledge, this paper presents a first attempt in the theoretical study of the distributions of {\em aggregated} features. We will also show that these models agree well with results from DNN models used in practice.
\section{Related Work}
The theoretical study of the statistics of layers in DNN is currently an active area of research.
The connection between Gaussian Processes (GP) and a single fully connected layer of infinite neurons was initially established by Neal \cite{nealthesis}. 
Recently, the link between Gaussian Processes and neural networks with multiple hidden layers was shown by Lee et al. \cite{lee18} and Mathews et al. \cite{matthews2018gaussian}. This was further extended to cover deep nets with convolutional layers by Novak et al. \cite{Novak19}. However, the validity of the above work rests on the assumption of infinite channels and hidden units. Unfortunately, practical DNNs in use at present contain finite hidden units and convolutional channels. Recently, Brock et al. \cite{brock2021characterizing} investigated how the mean and variances of layer outputs changes across a DNN. However, they did not explicitly model the distribution itself.
Since the number of filters used in convolutional layers are often large (thousands), the law of large numbers can often be used to approximate pre-activation values distributions using Gaussians. This is because the linear combination of a large number of distributions tends to result in a Gaussian distribution. However, this does not apply in a non-linear aggregation block, where responses from filters are individually transformed non-linearly. Our work directly addresses this issue by approaching it in terms of transformation of probability distributions. 
%
%
%
\subsection{Contributions and Overview}
To our knowledge, there does not exist any work on theoretically modelling the output distributions of DNN layers involved in the aggregation process. 
\begin{itemize}
\item We carry out a theoretical study how the distribution of information changes across the aggregation layers as illustrated in Fig. \ref{fig:nn_overview} (Section \ref{sec:agg_blocks}), where we propose a novel mixed distribution for modelling the outputs of the last convolutional layer.  
\item We show that the convolutional layer outputs are non-Gaussian due to ReLU activation. These output distributions are typically immediately consolidated into a Gaussian distribution by subsequent convolutional layers. However, this does not occur in an aggregation block. As such, our proposed distribution plays an important role in determining the distributions of DNN layers in the aggregation block.
\item We also propose a formulation for analytically predicting the KL-divergence of the outputs of a DNN based solely on the last convolutional layer distributions.
\item Our novel mathematical formulation allows us to theoretically link the role the covariance between convolutional filters and features plays in the KL-divergence of class specific distributions of a DNN output nodes. We also show this connection experimentally.
\end{itemize}
We perform experiments and show that our mathematical formulations can be used to accurately predict the distributions for layers involved in the aggregation of deep features (Section \ref{sec:exp}), before concluding in Section \ref{sec:conclusions}.

\begin{figure*}[t!]
\begin{center}
\includegraphics[width = 1\linewidth]{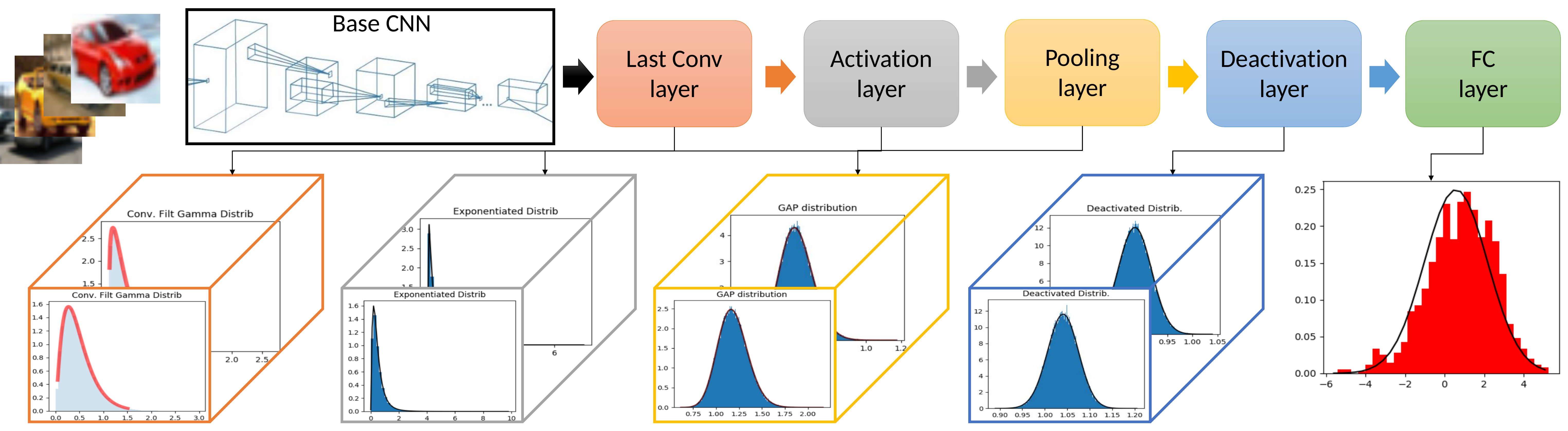} 
\end{center}
\caption{An overview of how the distribution of values from the last convolutional layer (for images of a specific label) is transformed by a nonlinear activation layer, followed by the pooling layer, deactivation layer, before being combined together into an output distribution.}
\label{fig:nn_overview}
\end{figure*}

\section{Aggregation DNN Classifiers}
In this paper, we consider DNN architectures that have the following structure: base-network $\rightarrow$ aggregation block $\rightarrow$ fully connected classification layer.

The base network typically consists of a sequence of convolutional blocks, and can be somewhat shallow (AlexNet \cite{krizhevsky2012imagenet}, VGG \cite{simonyan2014very}) to very deep (ResNet101 \cite{resnet}, Xception \cite{Xception}). 
The final convolutional layer then produces an output tensor that is passed to the aggregation block, where features are aggregated together to form a compact global image signature, also denoted as a {\em global descriptor}. The popular feature aggregation methods are Global max pooling (GMP) \cite{MAC}, Global average pooling (GAP) \cite{lin2014network}, NetVLAD \cite{NetVLAD}, Generalized mean pooling (GEMP) \cite{GEM}, Region of interest pooling (ROIP) \cite{ROIP} and REMAP \cite{REMAP}. In a classification DNN, this global descriptor is then passed to a fully connected classification layer for producing the output result.

Our work is concerned with understanding the statistics of the aggregation block and how information is transformed within it. We consider the case where the final convolutional layer features undergo non-linear amplification before the aggregation process, and also non-linear ``deactivation'' after aggregation.

\subsection{Non-linear Activated Aggregation}
\label{sec:nonlin_act_agg}
Let $F$ be the number of filters in the last convolutional layer. Given an input image, the last convolutional layer will produce an output tensor $\mathbf{T}$ of size $W\times H \times F$. The elements of $\mathbf{T}$ is denoted as $T_{ijk}$ where $i \in \{1,2,...,W\}$, $j \in \{1,2,...,H\}$ and $k \in \{1,2,...,F\}$. Additionally, the total number of output pixels in the convolutional image of each filter is written as $R = WH$.

In most state of the art DNNs, the last convolutional layer features are aggregated, typically using global average pooling (GAP), giving a $F$-dimensional ``global descriptor'' vector $\mathbf{d} = (d_1,d_2....d_F)$ of the input image, where its elements are:
\begin{equation}
    d_f = \frac{1}{R} \sum^N_i\sum^N_j T_{ijf}
\end{equation}

The vector $\mathbf{g}$ is given as input to a fully connected layer for classification.

In our work, we consider the case where the non-linear exponential transformation is applied {\em before} aggregation using GAP. Specifically, each element of the last convolutional layer tensor $\mathbf{T}$ is independently transformed as follows:
\begin{equation}
g(x) = \alpha ( \exp(\beta x ) - 1 )
\label{eq:scaled_exp_trans}
\end{equation}

The result is then passed through the GAP operation to produce an exponentially amplified global descriptor $\mathbf{e} = (e_1,e_2,...,e_F)$, where the elements are:
\[
e_f = \frac{1}{R} \sum^N_i\sum^N_j g( T_{ijf} )
\]

The exponential amplification of tensor elements can produce values that are very large. To overcome this issue, a ``deactivation'' process is performed by raising the amplified global descriptor to a small power $\gamma$ (less than 1): 
\begin{equation}
    s(x) = x^\gamma
\end{equation}

This produces the following non-linearly aggregated descriptor $\mathbf{e'} = (e'_1, e'_2,...,e'_F)$ that is used for classification, where:
\[
e'_f = s( e_f )
\]
An illustration of the aggregation DNN described above is shown in Figure \ref{fig:nn_overview}. We also show how the activation distribution changes across the different aggregation block layers. For the rest of the paper, we will use the above DNN architecture.

\section{Statistics of Aggregation Blocks} \label{sec:agg_blocks}

In this section, we propose a novel formulation for the statistics of different operations in an aggregation block. This allows us to study how the distributions of deep features changes as they move through the different aggregation block processes.
This will provide an important link between features of the last convolutional layer and the KL-divergence of two classes at the output layer (Section \ref{sec:class_kl}).
We start by proposing a suitable statistical model that describes the distribution of output values of the last convolutional layer of a DNN.

\subsection{Last Convolutional Layer Mixed-Distribution}
\label{sec:last_conv_mix_distrib}
In this section, we propose a probabilistic model capturing how the last convolutional layer outputs are distributed, particularly when ReLU is used as an activation function. We have found that when a convolutional layer uses ReLU, its output values will not be Gaussian.
In our case, when these values are passed on to subsequent amplification, the result will remain non-Gaussian. As we will see in Section \ref{sec:agg_blocks}, the distribution of output values throughout other layers in the aggregation block tend to remain non-Gaussian.

Based on observations, we propose to model each filter's responses in the last convolutional layer that are greater than 0 using the Gamma distribution:
 \begin{equation}
     f_\Gamma(x) = \frac{1}{\Gamma(a)s^a}x^{a-1}\exp(-x/s)
     \label{eq:gamma_pdf}
 \end{equation}
where $a$ and $s$ are the shape and scale parameters respectively. Here, $\Gamma(z)$ denotes the complex valued Gamma function, a continuous analogue of the factorial function, which we use extensively in this paper, and is defined as:
\[
 \Gamma(z) = \int^\infty_0 x^{z-1}e^{-x} dx
\]

%
The use of ReLU also has the effect of concentrating all the mass of the pdf at 0 and negative values of $x$ to a single point of $x = 0$. In order to capture such a behaviour, we use the following mixed probability distribution defined in the range $[0,\infty)$:
\begin{equation}
f(x) = \left\{ \begin{array}{cc}
     p &, \textrm{if}\;\; x = 0 \\
     (1-p)\frac{1}{\Gamma(a)s^a}x^{a-1}\exp(-x/s)& , \textrm{if}\;\; x > 0
\end{array}\right.
\label{eq:zero_gamma_pdf}
\end{equation}
where $p$ is the probability of values less or equal to 0. The Gamma distribution used for values $x > 0$ with a location parameter of 0 was found to give a good fit to observed data whilst being simpler to work with. 

\subsubsection{Mean and Variance}
To use this distribution, it is important that its mean and variance is known. To start, let $X$ be a random variable with the probability distribution function (pdf) of $f(x)$ (Eq. \ref{eq:zero_gamma_pdf}).
We find that the mean of $X$ is the weighted mean of the Gamma distributed component:
\begin{eqnarray}
\mu_f & = & \mathbb{E}[X] 
 = \int^\infty_0 f(x)x dx \nonumber\\
& = & \int^\infty_0 (1-p)\frac{1}{\Gamma(a)s^a}x^{a-1}\exp(-x/s)x dx \nonumber \\
& = & (1-p)as 
\label{eq:zero_gamma_mean}
\end{eqnarray}
The last line above uses the fact that the mean of a Gamma distribution is $as$. Next, to find the variance, denoted as $\sigma_f^2$, we first note that:
\begin{equation}
\sigma^2 = \mathbb{E}[X^2] - (\mathbb{E}[X])^2
\label{eq:gen_var}
\end{equation}

To proceed, we will need the second moment of $X$:
\begin{eqnarray}
\mathbb{E}[X^2] & = & \int^\infty_0 f(x) x^2 dx \nonumber \\
& = & \int^\infty_0 (1-p)\frac{1}{\Gamma(a)s^a}x^{a-1}\exp(-x/s)x^2 dx \nonumber \\
& = & (1-p)a(a+1)s^2
\label{eq:zero_gamma_2nd_mom}
\end{eqnarray}
The last equation above is obtained by using moment generating functions on the Gamma distribution (details of the derivation in Appendix \ref{app:mix_gamma_scnd_mnt_deriv}). Then, by substituting Eq. \ref{eq:zero_gamma_2nd_mom} and Eq. \ref{eq:zero_gamma_mean} into Eq. \ref{eq:gen_var}, we have the variance as:
\begin{eqnarray}
\sigma^2_f& = & (1-p)a(a+1)s^2 - (1-p)^2a^2s^2 \nonumber \\
& = & (1-p)as^2( a - (1-p)a ) \nonumber \\
& = & (1-p)as^2(1+ap)
\end{eqnarray}

\subsection{Distribution of Non-linear Activation}
We now describe how a non-linear activation on mixed-Gamma tensor values will transform its distribution. In particular, we consider the scaled exponential transform: $g(x) = \alpha( \exp(\beta x ) - 1$ (Eq. \ref{eq:scaled_exp_trans}).

We have seen in Section \ref{sec:last_conv_mix_distrib} how the distribution of each last convolutional layer filter output is modelled as a mixture of a discrete point distribution at 0 and Gamma distribution for non-zero positive values. Consequently, if this zero-Gamma mixed distribution is exponentiated, the result will be a new mixed distribution, with a discrete component at 1 and an exponentiated Gamma distribution $h(x)$ (which we describe next in Section \ref{sec:exp_gamma}) for all values greater than 1. In other words, we have the new distribution $h_m$, 
which we call the ``Zero-ExpGamma'' distribution and is defined for $x \geq 0$:
\begin{eqnarray}
h_m(x) = \left\{
\begin{array}{cc}
     p & ,\textrm{if}\;\; x = 1 \\
     (1-p)h(x) & ,\textrm{if}\;\; x > 1
\end{array}
\right.
\label{eq:zero_exp_gamma_pdf}
\end{eqnarray}

\subsubsection{Exponentiated Gamma Distributions}
\label{sec:exp_gamma}
We start by determining the distribution for an exponentially transformed Gamma distribution. Suppose a random variable $X$ follows the Gamma distribution described in Eq. \ref{eq:gamma_pdf}. We would like to know what the distribution of this random variable will be when it is exponentially transformed.
\begin{proposition}
Let $X$ be a random variable with distribution $f_\Gamma(x)$ in Eq. \ref{eq:gamma_pdf}. Let $g(x) = \alpha( \exp(\beta x) - 1)$ with $\alpha, \beta \in \mathbb{R}$, then $g(X)$ follows the following distribution:
\[
h(x) = \frac{1}{\Gamma(a)\beta^as^a}(\ln (1+x/\alpha))^{a-1}(1+x/\alpha)^{-\left(\frac{\beta s +1}{\beta s}\right)}
\]
\label{thm:exp_gamma}
\end{proposition}
The proof of this theorem is a straightforward application of the transformations of variables method commonly used to determine the distributions of transformed random variables. It requires that the transformation function has a differentiable inverse, which is the case here. Details of the proof can be found in Appendix \ref{app:exp_gamma_pdf}.
The mean of $h(x)$ is:
\[
\mu_h = \alpha\left(\frac{1}{(1- \beta s)^a} - 1\right)
\]
The variance of this distribution is:
\begin{equation*}
\sigma^2_h = \frac{\alpha^2}{(1-2\beta s)^a} - \frac{\alpha^2}{(1-\beta s)^{2a}}
\end{equation*}
The derivation of the above mean and variance can be found in Appendix \ref{app:zero_expgamma_meanvar}. 

\subsubsection{Zero-ExpGamma Mean and Variance}
The mean and variance of the mixed zero-ExpGamma distribution can be obtained in a similar manner to that of the zero-Gamma distribution. 
The mean of $h_m(x)$ is:
\begin{eqnarray}
\mu_{m} 
& = & \int^\infty_0 h_m(x)x dx 
 =  \int^\infty_0 (1-p)h(x) x dx \nonumber\\
& = & (1-p)\mu_h = \alpha(1-p)\left(\frac{1}{(1 - \beta s)^a} - 1\right)
\label{eq:zero_expgamma_mean}
\end{eqnarray}

The variance of $h_m(x)$ is found in the same manner as the previous distributions, by finding the second moment and subtracting the squared mean from it. The result is (derivation in Appendix \ref{app:zero_expgamma_var}):
\begin{eqnarray}
\sigma_m^2& = & \alpha^2(1-p)\left[ \frac{1}{(1-2\beta s)^a} - \frac{1-p}{(1-\beta s)^{2a}} - \right. \nonumber \\
&& \left. \frac{2p}{(1-\beta s)^a} + p \right]
\label{eq:zero_expgamma_var}
\end{eqnarray}


\begin{figure*}
    \centering
    \includegraphics[width = 0.95\linewidth]{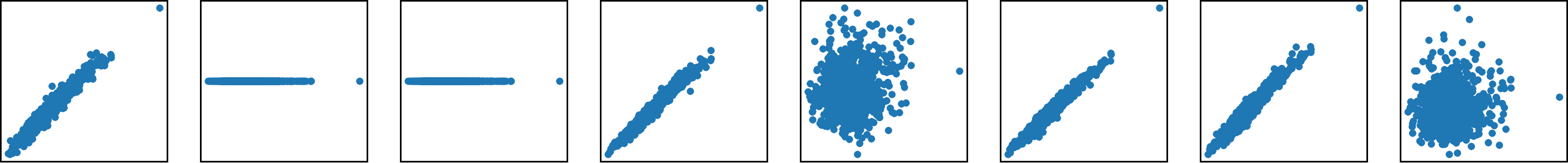}
    \caption{Scatter plots of values from 8 random pairs of pixels from the last convolutional outputs from the CIFAR dataset. This shows that the last convolutional layer output values can be strongly correlated and therefore not independent.  }
    \label{fig:gap_corr}
\end{figure*}
\subsection{Probability Distribution of GAP Features}
We are now in position to determine the distribution of global average pooled features. We assume that the scaled exponential activation function was used on the last convolutional layer outputs, so that the resulting ``pixels'' each follow the zero-ExpGamma distribution.
Importantly, we cannot assume that convolutional image pixels are independent. In fact, we have found that the convolutional image pixels are correlated both within the same filter and between different filters. This is illustrated in Fig. \ref{fig:gap_corr}. Therefore, accounting for these correlations is crucial when we are calculating the variance of the GAP features.

The convolutional output of a particular filter can be modelled as a collection of $R$ (i.e. total pixels in convolutional image) dependent random variables. The GAP operation has the effect of taking the average of $R$ random variables.

We denote the above $R$ random variables as $X_1,X_2,...,X_R$. 
The zero-expGamma distribution of these random variables is given in Eq. \ref{eq:zero_exp_gamma_pdf}. We denote their respective means and variance using $\mu_m$ and $\sigma^2_m$.
The GAP operation will produce a new random variable, $X_S$, that is the average of $X_1,X_2...,X_R$: 
\[
X_S = \frac{1}{R}\sum^R_{i=1}X_i
\]

The exact distribution of the exponentially activated GAP features can be very complicated. To overcome this, we describe (Section \ref{sec:gap_pdf_approx}) how it is possible to approximate the ``GAP feature'' distribution as a Gamma distribution. However, to determine the shape and scale parameter of this distribution we will need the mean and variance of $X_S$. We start with the mean, denoted as $\mu_{X_S}$, which is obtained using the linearity property of the expectation operator:
\begin{eqnarray}
\mu_{X_S} & = & \mathbb{E}[X_S] 
=\mathbb{E}\left[\frac{1}{R}\sum^R_{i=1}X_i \right]\nonumber \\
& = & \frac{1}{R} \sum^R_{i=1} \mathbb{E}[X_i] 
= \mu_m
\label{eq:gapsum_mean}
\end{eqnarray}
In other words, the mean is unchanged by the averaging operation. 

\subsubsection{GAP Feature Mean and Variance Prediction}
\label{sec:gap_feat_var_pred}
Since the pixels of each filter are correlated, the variance of the GAP distribution will also depend on the covariance matrix of the convolutional pixels of a filter. 
The presence of correlation between output pixels in the last convolutional layer means that the variance of the GAP distribution will depend on the covariance matrix of these pixel values.
Unfortunately, determining the exact covariance matrix of a sum of correlated R.V.s will require knowing their joint distribution. 
However, we have found that a 2nd order Taylor expansion of the covariance matrix circumvents this issue and provides a good approximation. 

To start, let the elements of the $R\times R$ covariance matrix between the random variables $X_1,X_2,...,X_R$ be denoted as $Cov_S(i,j)$, with $i,j\in\{1,2,...,R\}$ where:
\begin{eqnarray}
Cov_S(i,j) & = & \mathbb{E}[(X_i - \mu_{X_i})(X_j - \mu_{X_j})] \nonumber \\
&=& \mathbb{E}[X_iX_j] - \mu_{X_i}\mu_{X_j}
\label{eq:gap_cov_mat}
\end{eqnarray}

Our task now is to calculate $\mathbb{E}[X_iX_j]$. 
Suppose that $X_i$ is the result of the non-linear activation of the random variable $W_i$ (e.g. $W_i$ is some pixel of the last convolutional layer) using the function $g(x)$, that is: $X_i = g(W_i)$. Next, the 2nd order Taylor expansion of $g(x)$ at $x = 0$ gives $g(x) \approx g(0) + g'(0)x + g''(0)x^2/2$. 
Inserting this expansion into $\mathbb{E}[X_iX_j]$ gives:
\begin{eqnarray}
\mathbb{E}[X_iX_j] & = & \mathbb{E}[g(W_i)g(W_j)] \\
& \approx & \mathbb{E}\left[\left(g(0) + g'(0)W_i + \frac{g''(0)W_j^2}{2}\right)\times \right.\nonumber\\
&&\left. \left(g(0) + g'(0)W_j + \frac{g''(0)W_j^2}{2}\right)\right] \nonumber\\
& = & \mathbb{E}\left[g^2(0) + A(W_i + W_j) + BW_iW_j + \right.\nonumber\\
&& \left.C(W_i^2W_j + W_iW_j^2) + DW_i^2W_j^2\right] \nonumber\\
& = & g^2(0) + A(\mathbb{E}[W_i] + \mathbb{E}[W_j]) + B\mathbb{E}[W_iW_j] + \nonumber\\
&& C(\mathbb{E}[W_i^2W_j] + \mathbb{E}[W_iW_j^2] ) + D\mathbb{E}[W_i^2W_j^2]
\label{eq:exp_xi_xj}
\end{eqnarray}
with constants $A = g(0)g'(0)$, $B = (g'(0))^2$, $C = g'(0)g''(0)/2$ and $D = g''(0)^2/4$. 

We are now faced with the issue of calculating the expectation: $\mathbb{E}[W_i^a,W_j^b]$ where $a,b \in {0,1,2}$. We can achieve this {\em without} knowledge of the joint distribution between $W_i$ and $W_j$ by using the relationship: $\mathbb{E}[W_i^aW_j^b] = Cov(W_i^a,W_j^b) - \mu_{W_i}\mu_{W_j}$. Here, the covariance matrix $Cov(W_i^a,W_j^b)$ and means $\mu_{W_i},\mu_{W_j}$ can all be directly computed using the data from the last convolutional layer outputs. Since we will be using the approximation of $\mathbb{E}[X_iX_j]$, we should also use the approximation of $\mu_{X_i}$ and $\mu_{X_j}$ too, where:
\begin{equation}
\mu_{X_i} \approx g(0) + g'(0)W_i \frac{g''(0)}{2}W_i^2
\label{eq:gap_mu_approx}
\end{equation}
with $\mu_{X_j}$ obtained in a similar manner. In our case where $g(x) = \alpha( e^{\beta x} - 1)$, we have: $g'(x) = \alpha\beta e^{\beta x}$ and $g''(x) = \alpha\beta^2e^{\beta x}$, so that $g(0) = 0, g'(0) = \alpha\beta$ and $g''(0) = \alpha\beta^2$.

Finally, plugging the approximations of $\mathbf{E}[X_iX_j]$ (Eq. \ref{eq:exp_xi_xj} and $\mu_{X_i}, \mu_{X_j}$ (Eq. \ref{eq:gap_mu_approx}) back into Eq. \ref{eq:gap_cov_mat} will give us an estimate of the $R\times R$ covariance matrix of different pixels in a particular filter.


We can now calculate the variance. This is the sum of the variances of $X_i$ scaled by a factor of $1/R$, with a correction factor based on the covariance matrix added in:
\begin{eqnarray}
\sigma^2_{X_S} &=& Var(X_S) \nonumber \\
&= & Var\left(\frac{1}{R}\sum^R_{i=1}X_i \right) + \sum^R_{\substack{i,j=1\\ i\neq j}}Cov_S(i,j) \nonumber \\
& = & \sum^R_{i=1}\frac{1}{R^2} Var(X_i) + \sum^R_{\substack{i,j=1\\ i\neq j}}Cov_S(i,j) \nonumber \\
& = &\frac{1}{R}\sigma^2_m + \sum^R_{\substack{i,j=1\\ i\neq j}}Cov_S(i,j)
\label{eq:gapsum_var}
\end{eqnarray}

\subsubsection{GAP Feature Distribution Approximation}
\label{sec:gap_pdf_approx}
We now find an approximation for the distribution of $X_S$. 
First, the zero values of $X_S$ will have a probability of $p^N$. 
We find that the Gamma distribution provides a good approximation when we are considering the sum of exponentiated Gamma distributions. In fact, the higher the resolution of the convolutional image, the better the approximation.
Thus, we will assume that the non-zero values of $X_S$ will follow a Gamma distribution. 

To obtain the shape and scale parameters of this Gamma distribution, denoted as $a_S$ and $s_S$ respectively,
we will use $\mu_{X_S}$ (Eq. \ref{eq:gapsum_mean}) and $\sigma^2_{X_S}$(Eq. \ref{eq:gapsum_var}).
We find that $\mu_{X_S} = a_Ss_S$ and $\sigma^2_{X_S} = a_Ss^2_S$. 
Algebraic rearrangement then yields:
\[
s_S = \frac{\sigma_{X_S}^2}{\mu_{X_S}},\;\;\textrm{and}\;\;a_S = \frac{\mu_{X_S}^2}{\sigma_{X_S}^2}
\]

%

Thus, each dimension of the GAP feature vector will have the following distribution:
\[
f_S(x) = \left\{
\begin{array}{cc}
     p^N &, \textrm{if}\;\; x = 0  \\
     \frac{1-p^N}{\Gamma(a_S)s_S^{a_S}}x^{a_S -1}\exp(-x/s_S)& , \textrm{if}\;\; x > 0 
\end{array}
\right.
\]
At this point, we note that $p^N$ will be extremely small (since $N$ tends to be tens to hundreds for typical convolutional image sizes), and we will take it to be 0, and $1-p^N$ will be taken to be 1. Consequently, we can ignore the $x=0$ branch of the above distribution, giving a simpler form of:
\begin{equation}
    f_S(x) = \frac{1}{\Gamma(a_S)s_S^{a_S}}x^{a_S-1}\exp(-x/s_S)
\label{eq:gap_pdf}
\end{equation}
\subsection{GAP Features Covariance Matrix}
In order to predict the variances of DNN layers after the GAP layer, we will require the covariance matrix of the GAP features. 
We approximate this covariance matrix using a 2nd order Taylor expansion.

Let the number of GAP features be $N$, giving an $N\times N$ GAP feature covariance matrix, and let the random variable of each GAP feature be denoted to be $Y_i$ where $i = 1,2,...,N$. Then, each GAP feature, $Y_i$ say, is the result of averaging $R$ number of convolutional pixel random variables $(X_{i,1},X_{i,2},...,X_{i,R})$ that have been transformed into $(g(X_{i,1}),g(X_{i,2}),...,g(X_{i,R}))$ using some function $g(x)$, so that:
\[
Y_i = \frac{1}{R}\sum^R_{k=1} g\left(X_{i,k}\right)
\]
The covariance between two GAP features, $Y_i$ and $Y_j$ is:
\begin{eqnarray}
Cov_G( i, j ) & = & \mathbb{E}[(Y_i - \mu_{Y_i})(Y_j - \mu_{Y_j})]\\ 
& = & \mathbb{E}\left[\left(\frac{1}{R}\sum^R_{k=1} g\left(X_{i,k}\right)- \mu_{Y_i}\right) \right.\times \nonumber \\
&& \left.\left(\frac{1}{R}\sum^R_{k=1} g\left(X_{j,k}\right)- \mu_{Y_j}\right) \right] \nonumber\\
&=& \frac{1}{R^2}\mathbb{E}\left[\left(\sum^R_{k=1}g(X_{i,k})\right) \left(\sum^R_{k=1}g(X_{j,k})\right)\right] - \nonumber\\
&&\mu_{Y_i}\mu_{Y_j} \nonumber\\
& = & \frac{1}{R^2}\mathbb{E}\left[\sum^R_{k=1}\sum^R_{l=1}g(X_{i,k})g(X_{j,l})\right] - \mu_{Y_i}\mu_{Y_j} \nonumber\\
& = & \frac{1}{R^2}\left(\sum^R_{k=1}\sum^R_{l=1}\mathbb{E}\left[g(X_{i,k})g(X_{j,l})\right]\right) - 
\mu_{Y_i}\mu_{Y_j} \nonumber
\label{eq:gap_cov}
\end{eqnarray}

We now need to find an accurate estimate of
$\mathbb{E}[g(X_{i,k})g(X_{j,l})]$. For this, we again perform a second order Taylor expansion on $g(x)$ around $x=0$, so that $g(x) \approx a + bx + cx^2$, where for conciseness, we have written the expansion coefficients as: $a = g(0), b = g'(0)$ and $c = g''(0)/2$. For more clarity, we will denote $U = X_{i,k}$ and $W = X_{j,l}$. This gives:
\begin{eqnarray*}
\mathbb{E}[g(U)g(W)] & \approx & \mathbb{E}\left[ (a + bU + cU^2)(a + bW + cW^2)
\right] \\
& = & \mathbb{E}\left[ a^2 + ab(U+W) + ac(U^2 + W^2) +\right. \\
&& \left. b^2UW + bc(U^2W + UW^2) + c^2U^2W^2 \right]\\
& = & a^2 + ab(\mathbb{E}[U]+\mathbb{E}[W]) +\\
&&ac(\mathbb{E}[U^2] + \mathbb{E}[W^2]) + \\
&& b^2\mathbb{E}[UW] + bc(\mathbb{E}[U^2W] + \mathbb{E}[UW^2]) + \\ && c^2\mathbb{E}[U^2W^2]
\end{eqnarray*}
From the last equation above, we can see that a second order approximation of the GAP feature covariances can be obtained through first and second moments, as well as cross-covariances of the random variables of last convolutional layer output pixels. As such, they can be directly computed from a training set of last convolutional layer responses. The cross-covariances of the form $\mathbb{E}[U^aW^b]$, with $a,b \in \{0,1,2\}$ can be computed in the similar manner as those used in the GAP feature variance prediction (Section \ref{sec:gap_feat_var_pred}).

\subsection{Deactivation Layer Distribution}
The non-linear exponential operator used to amplify the tensor values before GAP can result in very large values after GAP. To address this issue, the deactivation layer is usually used. 
We start by describing the deactivation layer output values distribution.
\begin{proposition}
Let $X$ be a random variable with distribution $f_S(x)$ in Eq. \ref{eq:gap_pdf}. Let $g(x) = x^\gamma$ with $\gamma \in \mathbb{R}$, 
then $g(X)$ follows the following distribution (proof in Appendix \ref{app:deact_pdf_proof}):
\[
w(x) = \frac{1}{\gamma\Gamma(a)s^a}x^{\frac{a-1}{\gamma}}\exp\left(-\frac{x^\gamma}{s}\right)
\]

\label{thm:deact_pdf}
\end{proposition}

The mean ($\mu_w$) and variance $\sigma_w^2$ of the deactivation layer are as follows (derivations in Appendix \ref{app:deact_mean_var}):
\begin{eqnarray}
    \mu_w & = & \frac{s^\gamma}{\Gamma(a)}\Gamma( \gamma + a ) \\
\label{eq:deact_mu}
    \sigma^2_w & = & \frac{s^{2\gamma}}{\Gamma(a)}\left( \Gamma(a + 2\gamma) - \frac{\Gamma(a+\gamma)^2}{\Gamma(a)}\right)
\label{eq:deact_sigma}
\end{eqnarray}

\subsubsection{Deactivation Feature Covariance}
We find that determining the variances of outputs in layers after the deactivation layer will require the deactivation feature covariances, especially in the case of a fully connected layer.
For this, we give an approximation of the covariance matrix for the $N$ deactivated GAP features. Let $X,Y$ denote two random variables of two GAP features. In the case where our deactivation function is $g(x) = x^\gamma$, we are seeking an approximation of $Cov( X^\gamma, Y^\gamma )$, which can be written as:
\begin{eqnarray}
Cov_D(X^\gamma, Y^\gamma) & = & \mathbb{E}[X^\gamma Y^\gamma] - \mathbb{E}[X^\gamma]\mathbb{E}[Y^\gamma]
\label{eq:deact_cov}
\end{eqnarray}
We have found that in this case, a first order Taylor approximation sufficed, with the advantage of only requiring expectations that can be directly computed from training examples from the last convolutional layer outputs. Thus, taking the expansion of $g(x) = (x + \epsilon)^\gamma$ at $x = \epsilon$ gives us:
$g(x) \approx g(\epsilon) + \gamma(2\epsilon)^{\gamma - 1}(x-\epsilon)$, and for brevity, we will write $g(x) \approx c + dx$, with constants $d = \gamma(2\epsilon)^{\gamma - 1}, c = g(\epsilon)=(2 - d)\epsilon, $. 
The first order approximation for the mean of the deactivated random variable $X$ is:
\[
\mathbb{E}[X^\gamma] \approx c + d\mathbb{E}[X]
\]
where $\mathbb{E}[X]$ is the mean of the GAP feature associated with random variable $X$. In a similar manner, we can also obtain the approximation for $\mathbb{E}[Y]$.
Next, we can approximate the joint expectation $\mathbb{E}[X^\gamma Y^\gamma]$ as follows:
\begin{eqnarray*}
\mathbb{E}[X^\gamma Y^\gamma] & = & \mathbb{E}[(c + dX)(c + dY)] \\
& = & c^2 + cd( \mathbb{E}[X] + \mathbb{E}[Y] ) + d^2\mathbb{E}[XY]
\end{eqnarray*}
where $\mathbb{E}[XY]$ can be calculated using the covariance matrix from the GAP variables given in Eq. \ref{eq:gap_cov}, where suppose $X$ and $Y$ corresponds to the $i^{th}$ and $j^{th}$ GAP features respectively:
\[
\mathbb{E}[XY] = Cov_G( i, j ) - \mu_X\mu_Y
\]
and where $\mu_X, \mu_Y$ are the mean of the GAP feature random variables $X$ and $Y$ respectively.

\section{Application: Predicting Classification KL-Divergence of DNNs}
\label{sec:class_kl}
In this section, we detail the distribution of a DNN output node, and in the process mathematically link it with the aggregation layer distributions. This gives us the important ability to analytically predict the KL divergence of the DNN output nodes from the last convolutional layer outputs.  

\subsection{Output Layer PDF Approximation as Normal Distribution}
We first determine the distribution of values from a DNN output layer.
Here, we assume that the GAP layer is immediately followed by a fully connected (FC) classification layer. We will consider the distribution of this FC layer at {\em pre-activation} (i.e. before the softmax or sigmoid activation).

The use of an FC layer after the GAP layer means that the elemennts of the GAP feature vector will be linearly combined together. Let us now represent each element of the GAP feature as a random variable following a distribution of the form $f_S(x)$ (Eq. \ref{eq:gap_pdf}). Then, the R.V. of each output layer node can is a linear combination of GAP random variables.

We typically find that the number of GAP features ($F$, i.e. the number of filters in the last convolutional layer) is typically high, ranging from hundreds to thousands. Thus, we can use the Central Limit theorem to approximate the distribution of a DNN output node as a Normal distribution. This is extremely convenient since we only need to know the mean, variance and covariance matrix of the GAP elements.


Suppose we have $N$ deactivated GAP features, each represented by random variables $G_i, i = 1,2,...,N$, with their respective mean and variances denoted as $\mu_{G,i}$ (Eq. \ref{eq:deact_mu}) and $\sigma^2_{G,i}$ (Eq. \ref{eq:deact_sigma}). 

Suppose the fully connected classification layer has the following weights on these $N$ deactivated GAP features: $\{\alpha_i\}_{i=1}^N$. Then, let $O$ denote the random variable of the output layer before sigmoid activation, where:
$O = \sum^N_{i=1} \alpha_iG_i$.

Then, by the linearity of the expectation operator, the mean of $O$ is:
\begin{equation}
\mu_O = \sum^N_{i=1} \alpha_i\mu_{G,i}
\label{eq:out_mu}
\end{equation}

Similarly, it can be shown that the variance is:
\begin{equation}
\sigma^2_O = \sum^N_{i=1} \alpha^2_i \sigma_{G,i}^2 + 2\sum_{1\leq i < j \leq N} \alpha_i\alpha_jCov_D(i,j)
\label{eq:out_sigma}
\end{equation}
where $Cov_D(i,j)$ is the covariance of the $i^{th}$ and $j^{th}$ deactivated GAP features given in Eq. \ref{eq:deact_cov}.
Thus, by the central limit theorem, the distribution of $O$ can be approximated by the normal distribution:
\[
f(x) = \frac{1}{\sqrt{2\sigma_O}}\exp\left(
- \frac{(x - \mu_O)^2}{2\sigma_O^2}
\right)
\]

\subsection{KL-Divergence of Positive and Negative Class Output Values}
Using the normal distribution approximation from the previous section, we can now directly find the KL divergence between the positive and negative classes for the pre-sigmoid outputs of a DNN. 
In general, suppose we are given two normal distributions, say $p_-(x)$ and $p_+(x)$, with mean and standard deviation from Eq. \ref{eq:out_mu} and \ref{eq:out_sigma}, $(\mu_-,\sigma_-)$ and $(\mu_+,\sigma_+)$ respectively, their KL divergence has a particularly simple form:
\begin{eqnarray}
    K( p_+, p_- ) &=& \log( \sigma_- ) - \log( \sigma_+ ) + \nonumber \\ 
    &&\frac{\sigma_+^2 + (\mu_+ - \mu_-)^2}{2\sigma_-} - \frac{1}{2}
\label{eq:kl_div_func}
\end{eqnarray}


\section{Experimental Results and Analysis}
\label{sec:exp}
In this section, we describe experiments carried out that were aimed at validating the correctness of our proposed theoretical description of the output value distributions for the different layers in the aggregation block. This was done on two datasets: CIFAR10 and Flowers. 
They were chosen to show that our proposed distributions is valid across a range of different problems and dataset sizes. We first give details of the datasets, DNN models and training parameters used in Section \ref{sec:exp_dataset}. Following this, we use the trained DNN and its last convolutional layer responses from the dataset images to predict the distributions of the different aggregation layers. This allows us to predict the KL-divergence values of each classification output node, as described in Section \ref{sec:exp_kl_pred}. Following this, we further analyse how accurate the predicted distributions are in Section \ref{sec:exp_agg_layer_stats}. Our formulation also allows us to see how the correlation between deep features affects the KL divergence of a DNN in Section \ref{sec:exp_cov_deep_feat}.

\subsection{Datasets and DNN Training}
\label{sec:exp_dataset}

\subsection*{CIFAR10 Dataset}
The CIFAR10 dataset\cite{cifar} consists of 60000 small colour images of resolution 32x32 pixels. There are 10 classes, with 6000 images per class. The DNN used consisted of four convolutional blocks followed by a non-linear aggregation block and finally a fully connected classification layer (10 outputs) with softmax output. The convolutional block consisted of convolution->maxpool layers, with batch normalisation used at the end of each block. The number of filters used in the four convolutional blocks are  32,32,64 and 64 (first block to last block). The exponential activation layer was used and initialised with parameters $\alpha=1, \beta=0.01$ and power deactivation layer initialised with $\gamma = 0.8$.

The Adam algorithm with categorical loss was first used to warm up the convolutional layers and classification layer of the DNN. This was achieved  by temporarily connecting the last convolutional layer directly to the final fully connected classification layer. 
In the warmup stage, 40 epochs was used. Following this, the above temporary connection was removed. It was found that a further 5 epochs using Stochastic Gradient Descent (SGD) with learning rate of 0.0001 was sufficient to finetune the activation block parameters in an end-to-end manner.

\subsection*{Flowers dataset}
The Flowers dataset contains 3700 images belonging to five classes. We divide the dataset into 80\%-20\% training and validation split. 

The Activation CNN consists of baseline ImageNet trained EfficientNetB0 network \cite{EFFNET}. We proceed by removing the last pooling layer, prediction layer and loss layer from EfficientNetB0. The last convolutional layer output is passed to the Exponential based activation layer, followed by Global Average Pooling (GAP) layer, de-activation layer and fully-connected (FC) layer. 
The Activation CNN is fine-tuned on flowers dataset using cross-entropy loss. The images are resized to 224 $\times$ 224 pixels before passing through the network.  Optimization is performed by the Stochastic Gradient Descent (SGD) algorithm with momentum $0.9$, learning rate of $10^{-2}$ and weight decay of $5\times 10^{-4}$. The exponential activation layer is initialised with parameters $\alpha=8, \beta=0.1$ and and power deactivation layer initialised with $\gamma = 0.5$.



\subsection{Prediction of DNN KL-Divergence}
\label{sec:exp_kl_pred}
We have carried out experiments to predict the KL-divergence of DNNs. The following experimental setup was used: 
For each dataset, all of which are multi-class problems, we cycle through each class and set it as the ``positive'' class and group the remaining classes as ``negative''. The outputs of the last convolutional layer when the DNN is presented with the positive class is recorded and the mixed zero-Gamma distribution fitted. This allows us to predict the output value distributions for the activation, GAP, deactivation layers as well as the Gaussian distribution of the positive class corresponding output node. The same is also done for the negative class examples. Following this, the {\em predicted} KL divergence between the DNN output for the positive and negative classes was calculated. 

For comparison purposes, we also computed the {\em observed} KL divergence. This was achieved by recording the DNN output node corresponding to the positive class when positive examples are presented. A Gaussian was then be fitted to these observations. Similarly, another Gaussian was fitted to the output values when negative examples are presented to the DNN. Finally the observed KL divergence was calculated using these observation-based Gaussians.

We can see the predicted and observed Gaussians for the classification layer output nodes in Figure \ref{fig:cifar_out_gauss} for the CIFAR10 dataset and Figure \ref{fig:flowers_gauss} for the Flowers dataset. We can see that for the CIFAR10 dataset, the predicted Gaussians shown in solid lines is almost exactly the same as the observed Gaussians, shown in broken lines of the same colour. Additionally, the histogram obtained from corresponding observed data also shows a close match to the Gaussian distribution. This was the case across all 10 classes of the CIFAR10 dataset. We find similar results for the Flowers dataset. However, we find here there is a slight deviation between the predicted Gaussian variances compared with the corresponding observed Gaussians. 

We can then compare how well our predicted KL divergence matches to the observed KL divergence by means of a scatter plot, as shown in Figure \ref{fig:cifar_flower_kl_div}. Here the x-axis represents the observed KL-divergence values and the y-axis represents the predicted KL-divergenvce values. We can see in Figure \ref{fig:cifar_flower_kl_div} a), for the CIFAR10 dataset, there is good agreement between the predicted and observed KL divergence across all the different positive classes. In the case of the Flowers dataset, in general there is an over prediction of the KL divergence. However, larger observed KL-divergence values correlate well with larger predicted KL divergence values. The reason for this overestimation of the predicted KL divergence values is due to the slight underestimation of the Gaussian variances, as will be detailed in the next section.

\begin{figure}
    \centering
    \begin{tabular}{cc}
    \includegraphics[width=0.5\linewidth]{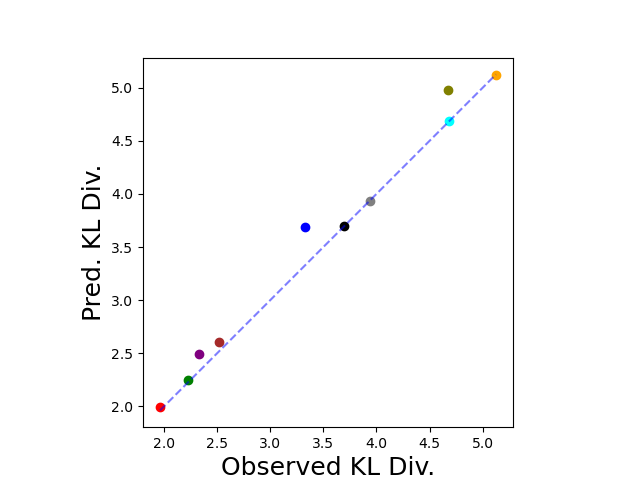} &
    \includegraphics[width=0.5\linewidth]{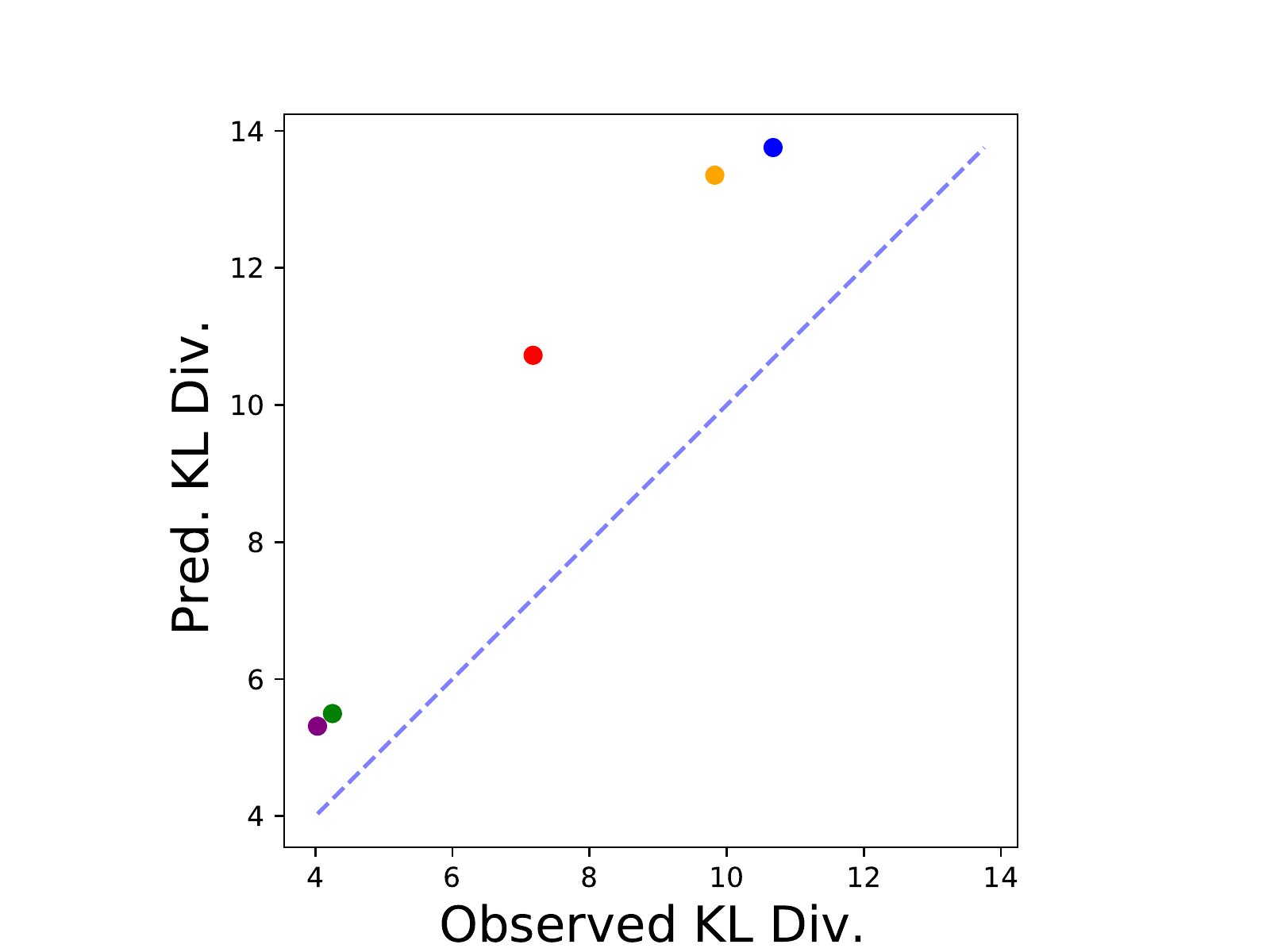}
    \\
    (a) CIFAR10 & (b) Flowers 
    \end{tabular}
    \caption{The KL-Divergence of the positive and negative distributions from the output values of output nodes for different classes. The different colours represent the different classes set as the positive class. The x-axis and y-axis reprsent the observed and predicted KL divergences respectively. Each point of the scatter plot shows the (observed KL, predicted KL) pair. The diagonal line shows how close our predicted KL divergence is to the observed KL. (a) shows the results for the CIFAR10 dataset and (b) for the Flowers dataset.}
    \label{fig:cifar_flower_kl_div}
\end{figure}

\begin{figure*}
\begin{tabular}{ccccc}
     \includegraphics[width=0.18\linewidth]{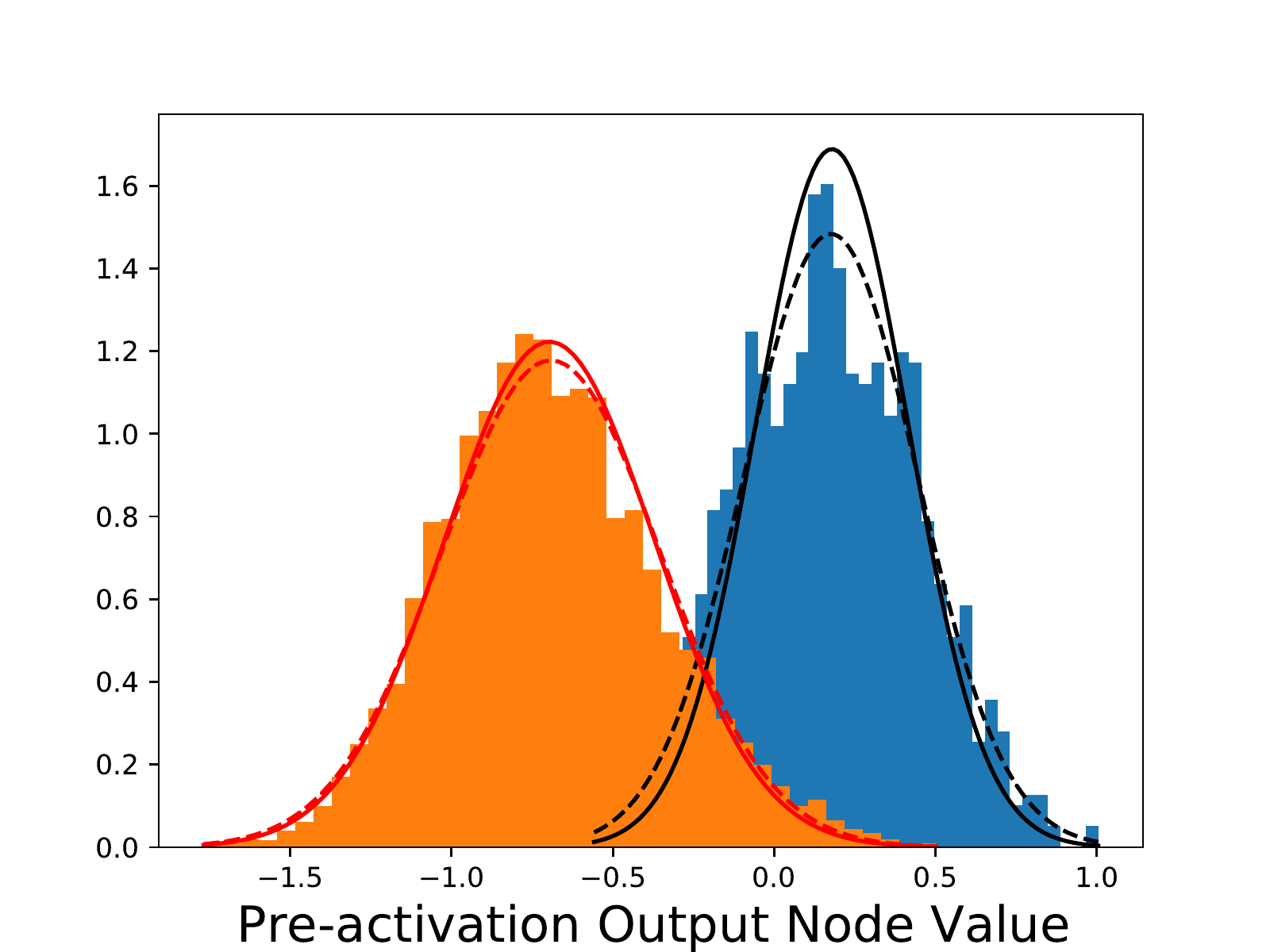} &
     \includegraphics[width=0.18\linewidth]{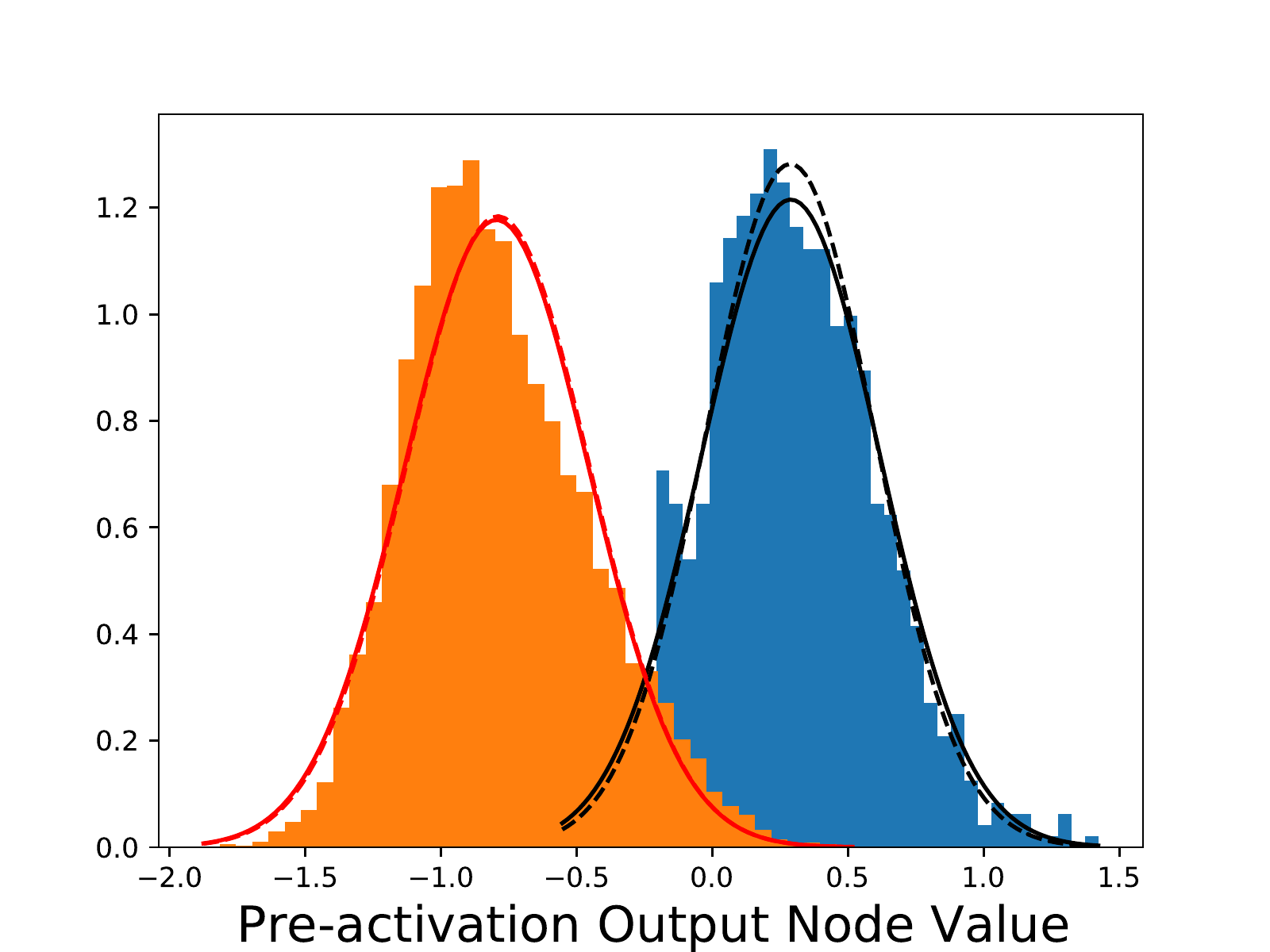}&
     \includegraphics[width=0.18\linewidth]{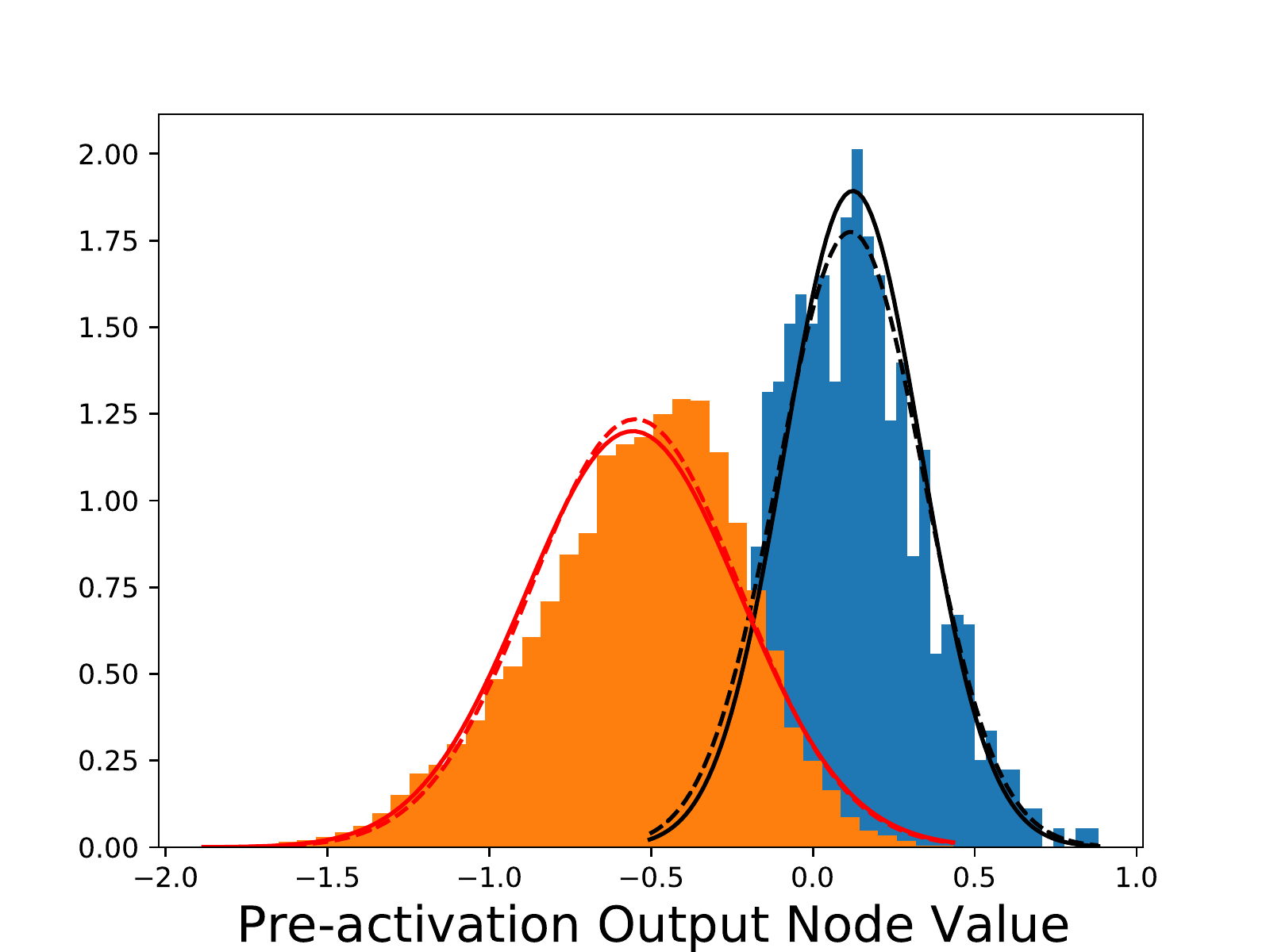}&
     \includegraphics[width=0.18\linewidth]{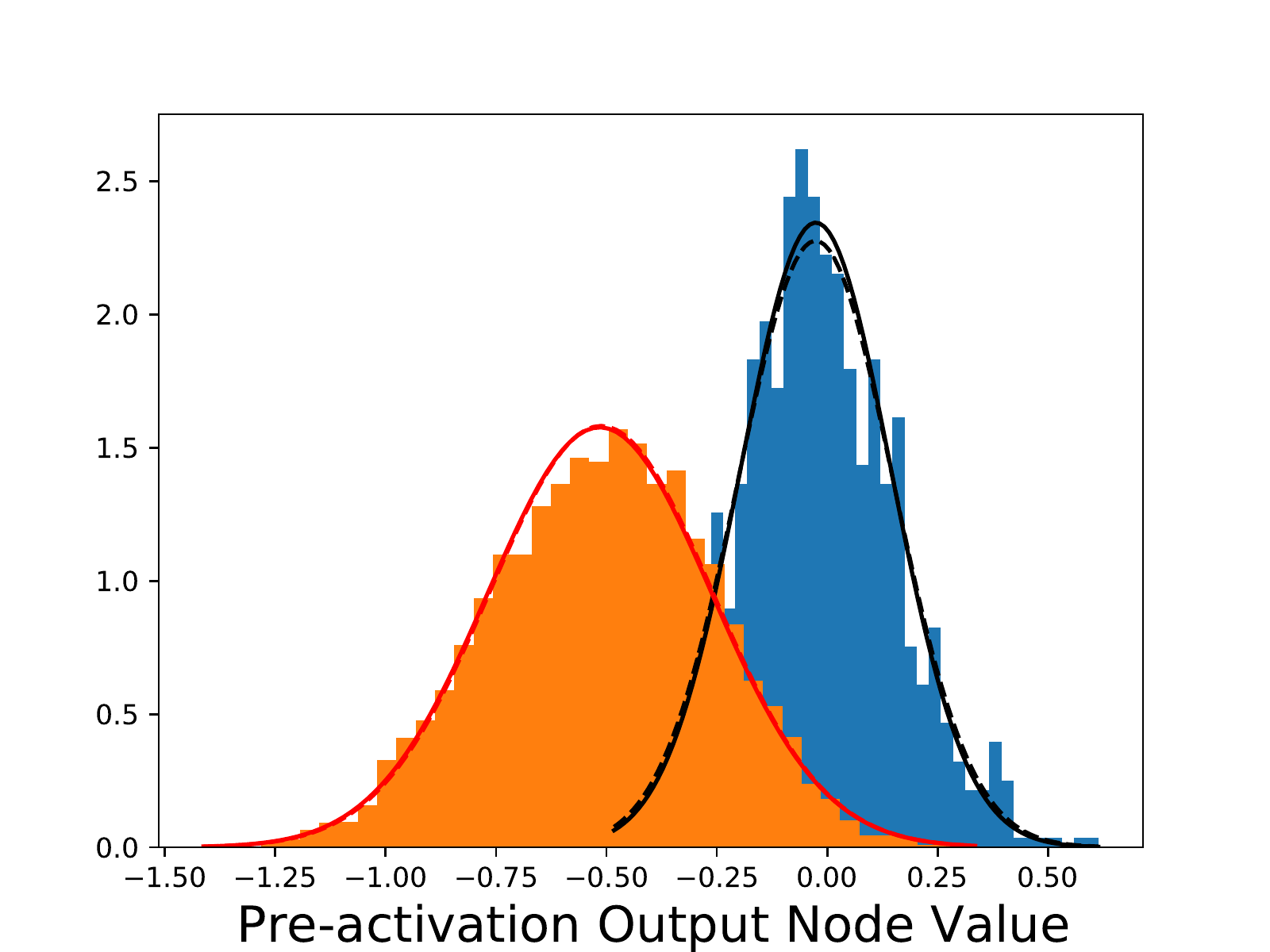}&
     \includegraphics[width=0.18\linewidth]{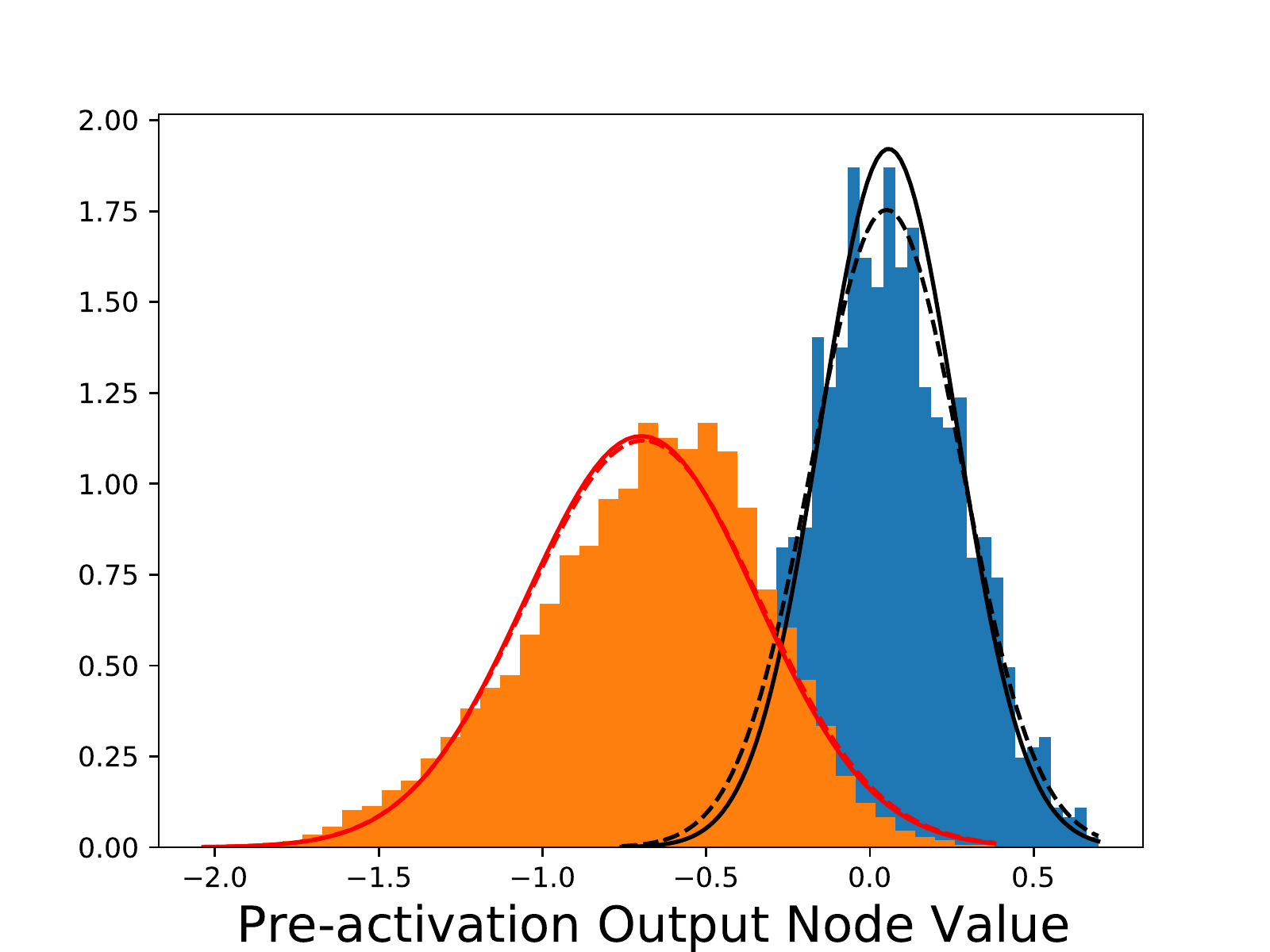}\\
     Cls. 1 & Cls. 2 & Cls. 3 & Cls. 4 & Cls. 5 \\
    \includegraphics[width=0.18\linewidth]{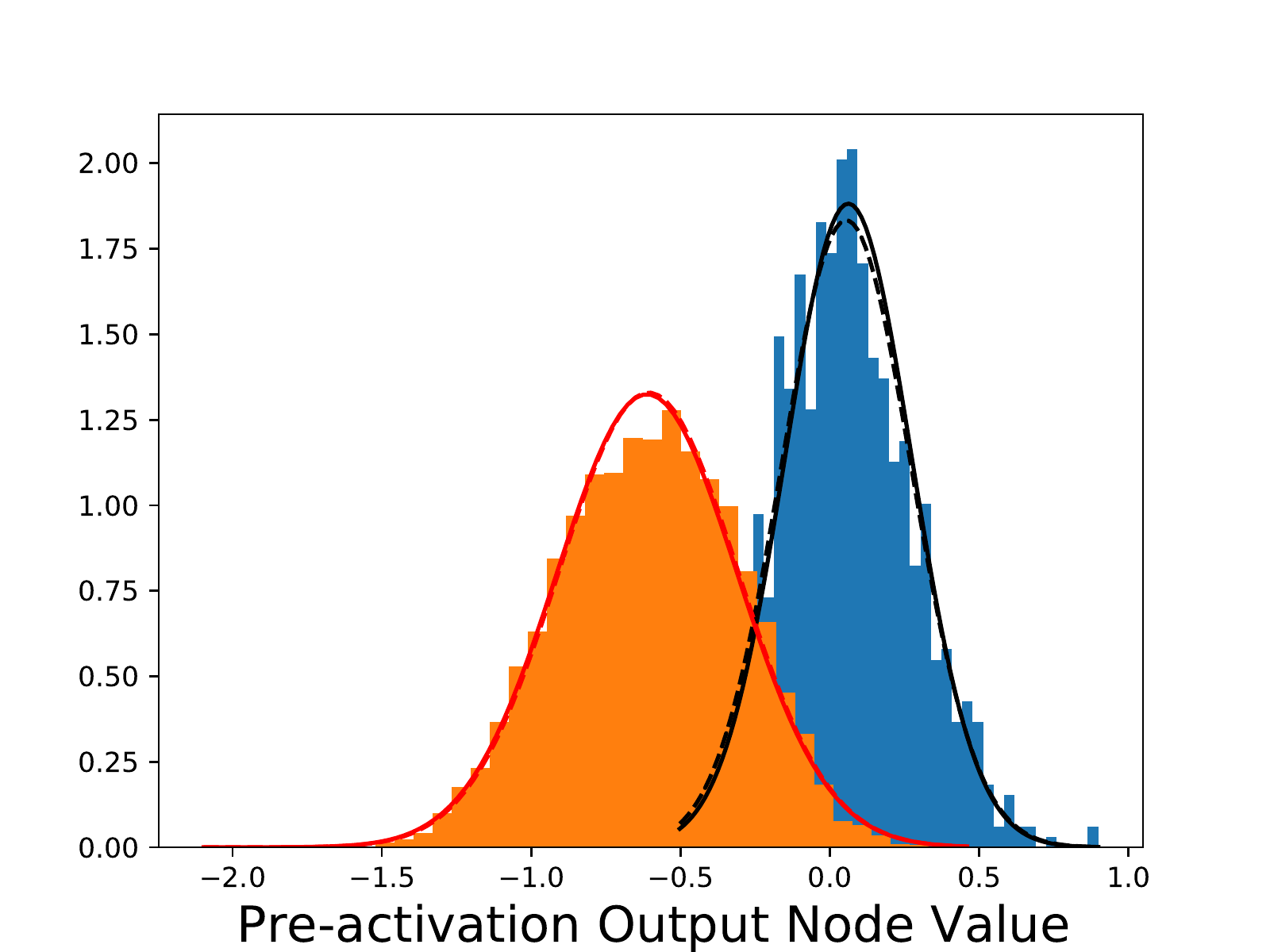} &
     \includegraphics[width=0.18\linewidth]{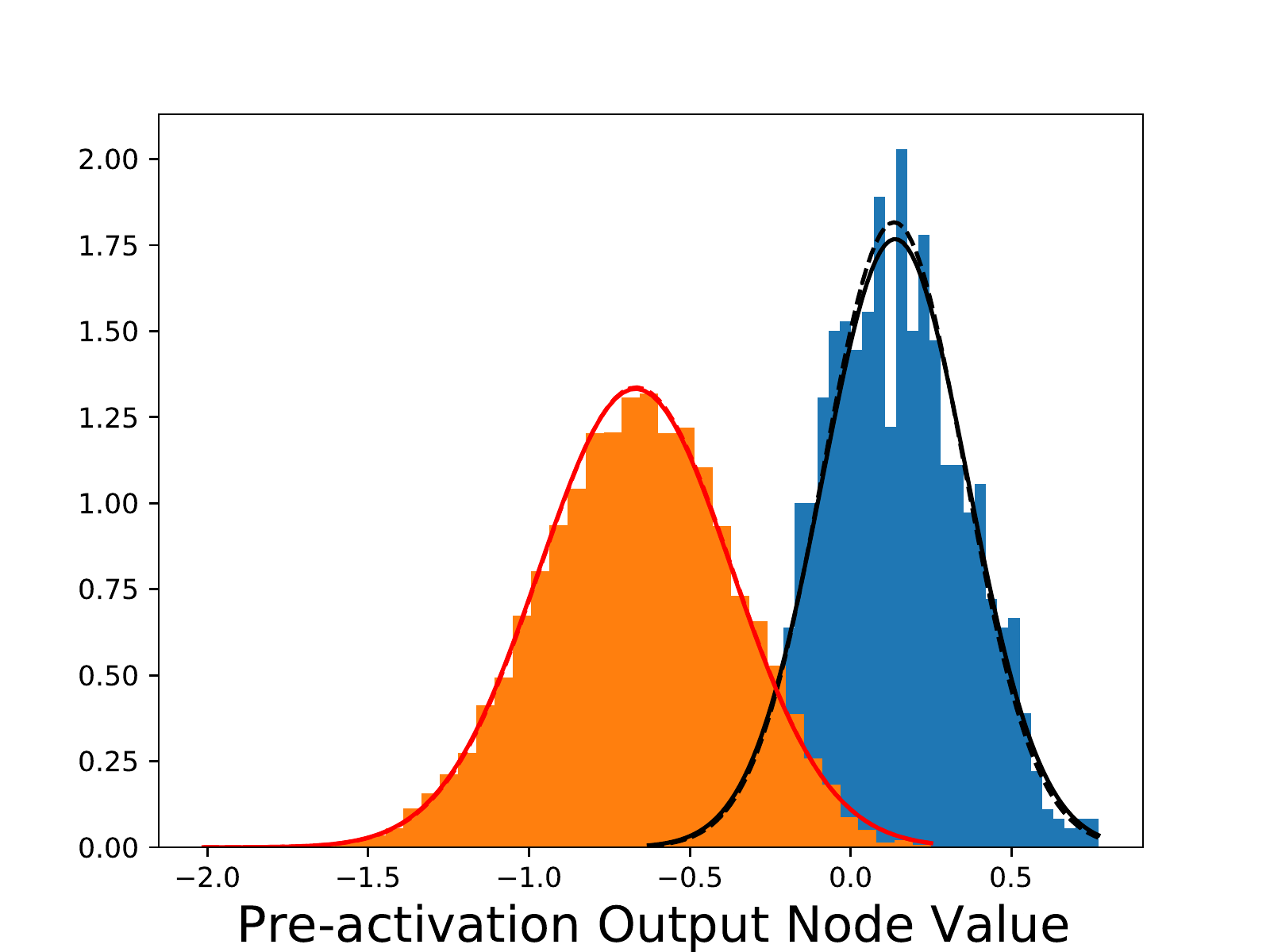}&
     \includegraphics[width=0.18\linewidth]{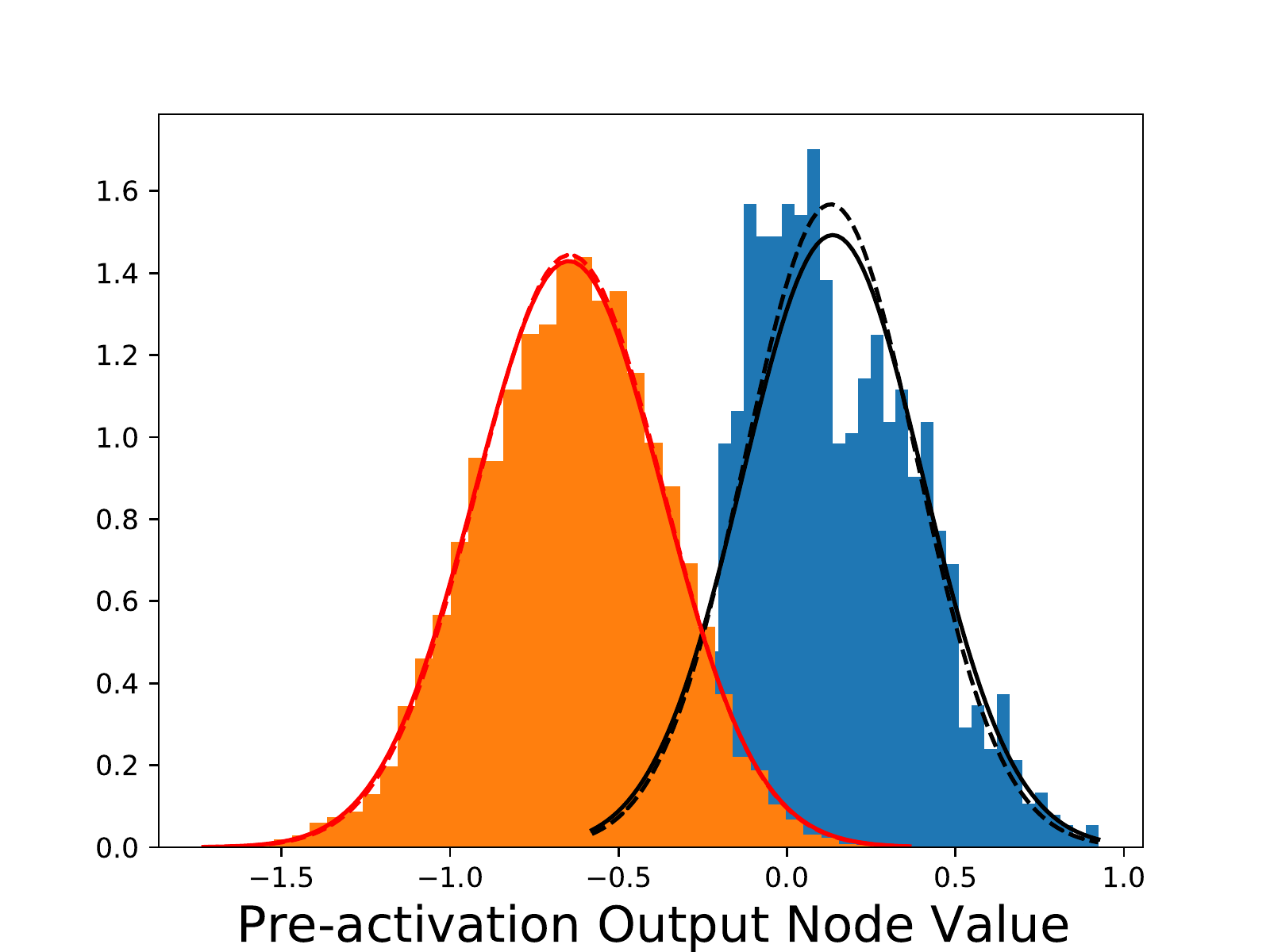}&
     \includegraphics[width=0.18\linewidth]{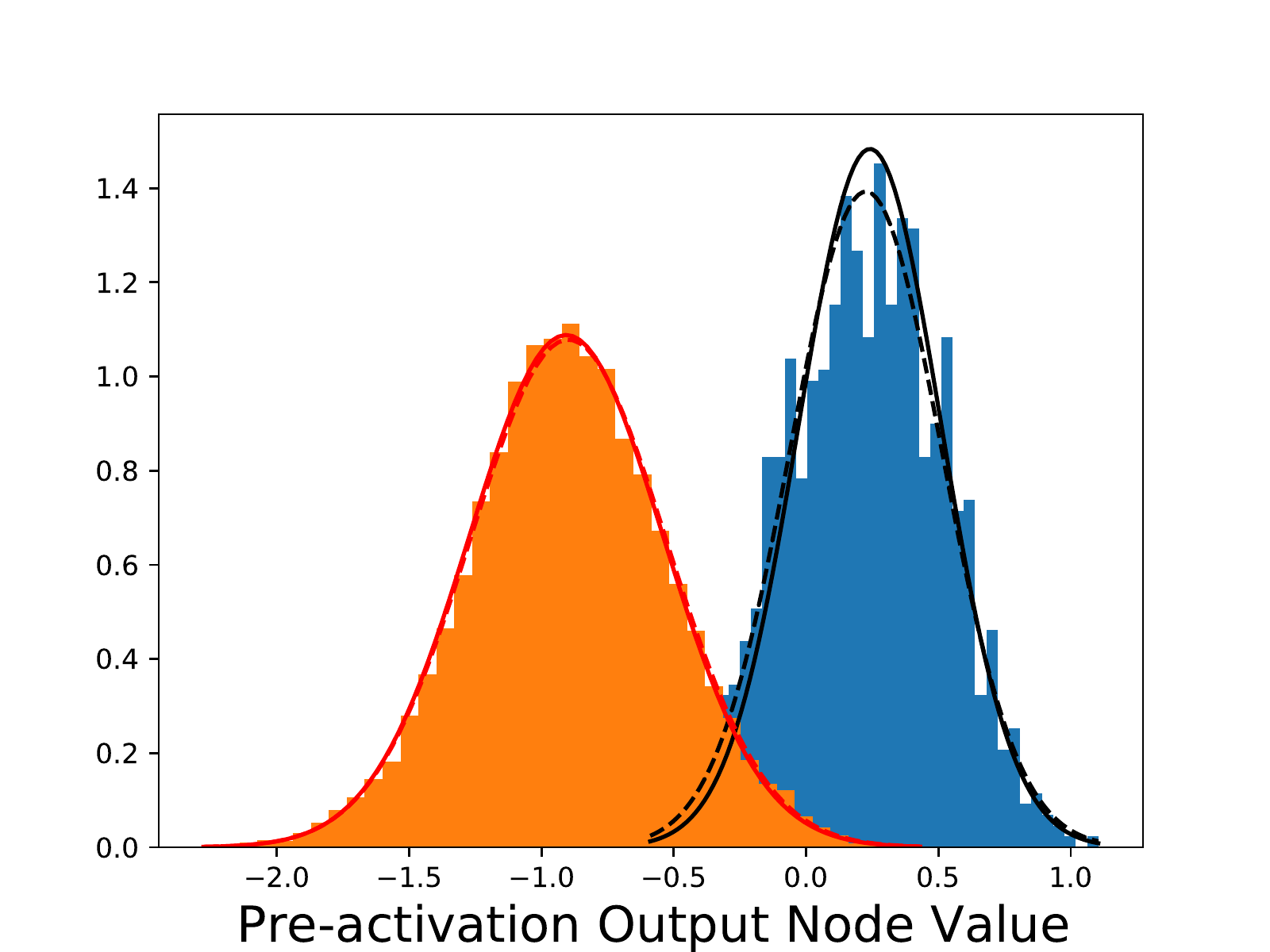}&
     \includegraphics[width=0.18\linewidth]{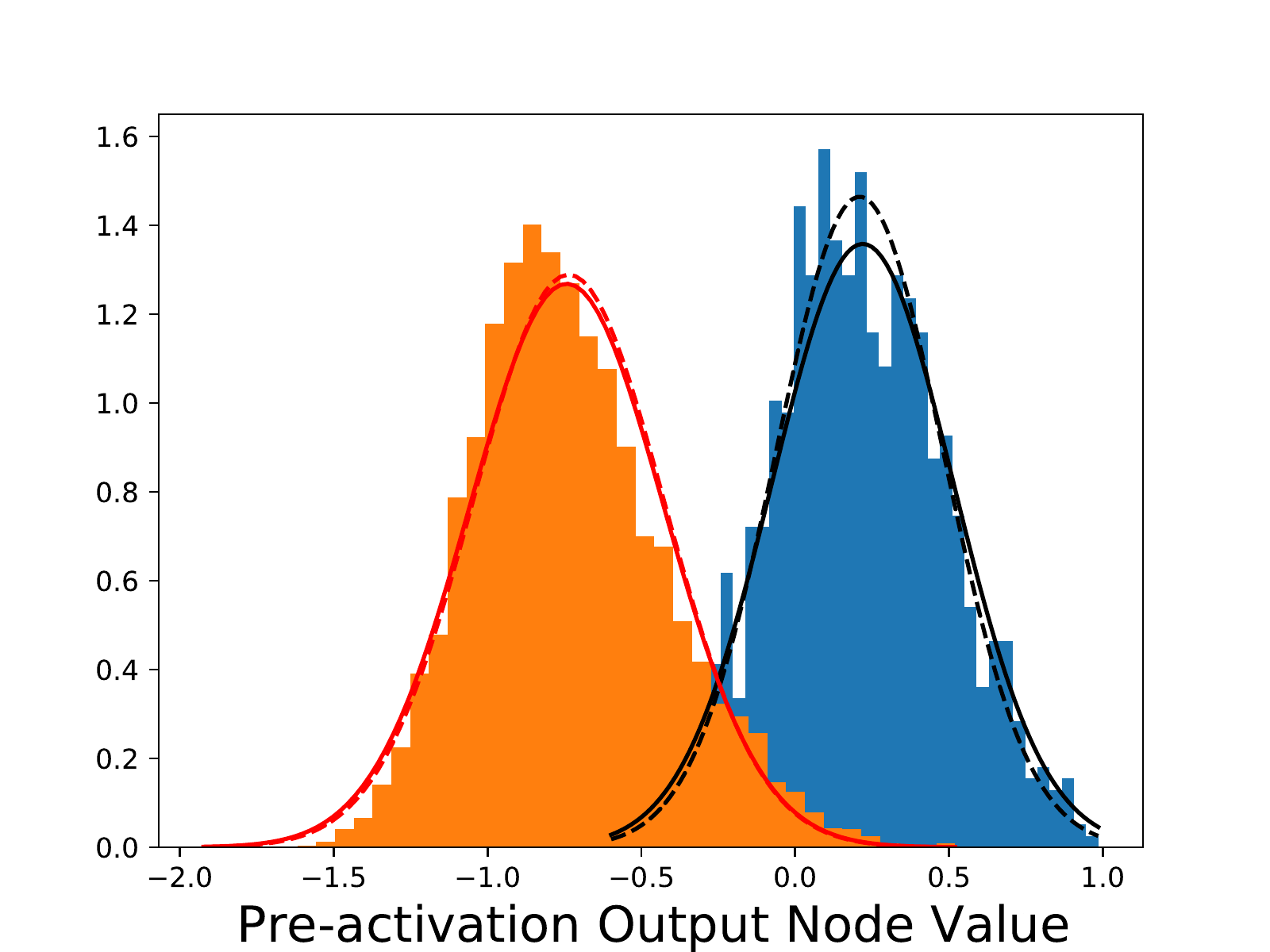} \\
     Cls. 6 & Cls. 7 & Cls. 8 & Cls. 9 & Cls. 10 \\
\end{tabular}
\caption{CIFAR10: Classification layer node (pre-softmax activation) distributions. The predicted distribution is shown as a solid curve and is shown against the histogram obtained by direct observation of the corresponding output node values. The Gaussian fitted to the observed output values is shown in dotted lines, whilst the solid line is the predicted Gaussian based solely on last convolutional layer values. The red curve shows the distribution for the responses from the ``negative examples'' (i.e. examples with labels that do not match the output node class) and black curve (blue histogram) for ``positive examples'' (images with label corresponding to output node class).}
\label{fig:cifar_out_gauss}
\end{figure*}

\begin{figure*}
    \centering
    \begin{tabular}{ccccc}
    \includegraphics[width=0.18\linewidth]{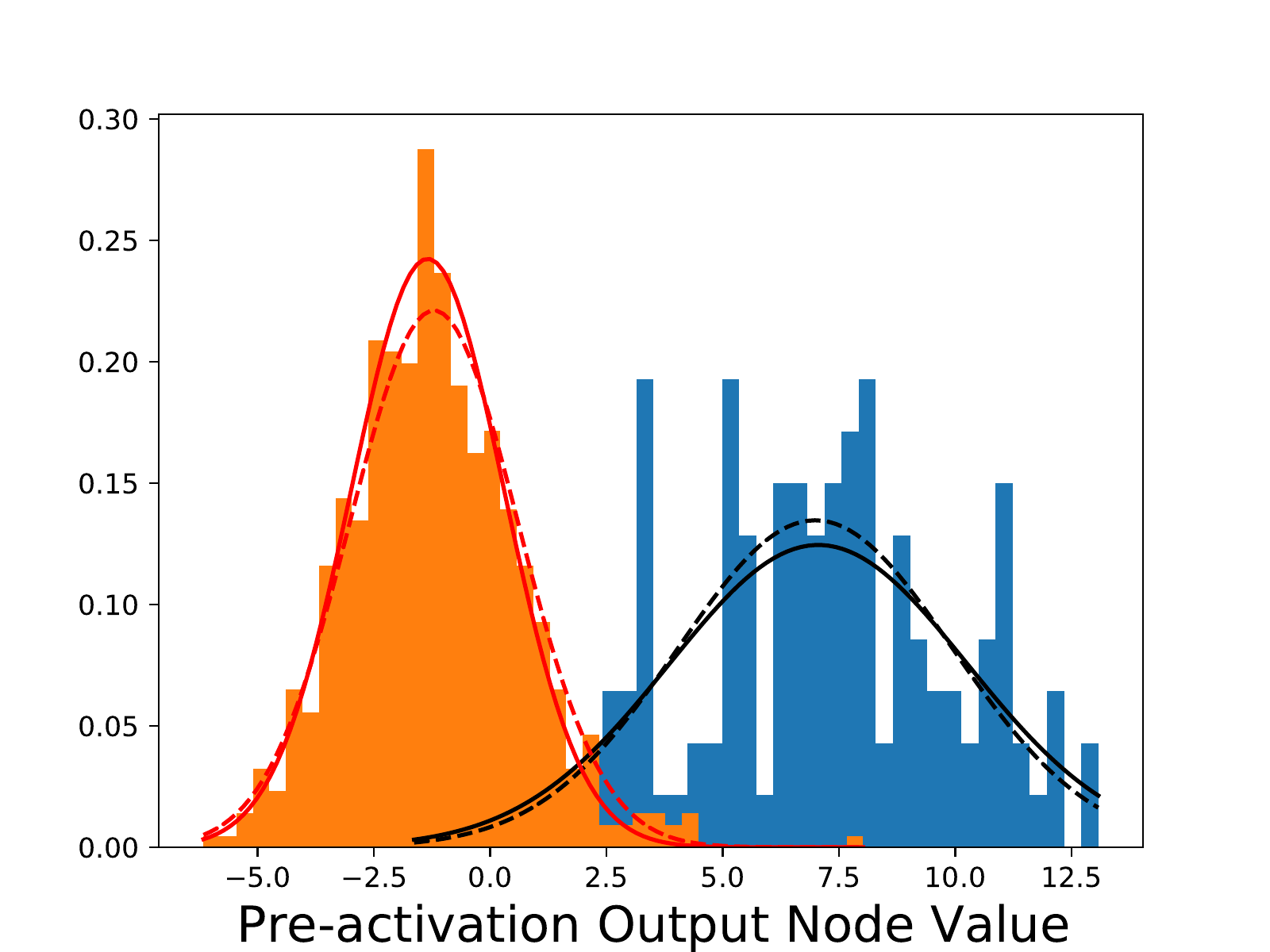}     &  
    \includegraphics[width=0.18\linewidth]{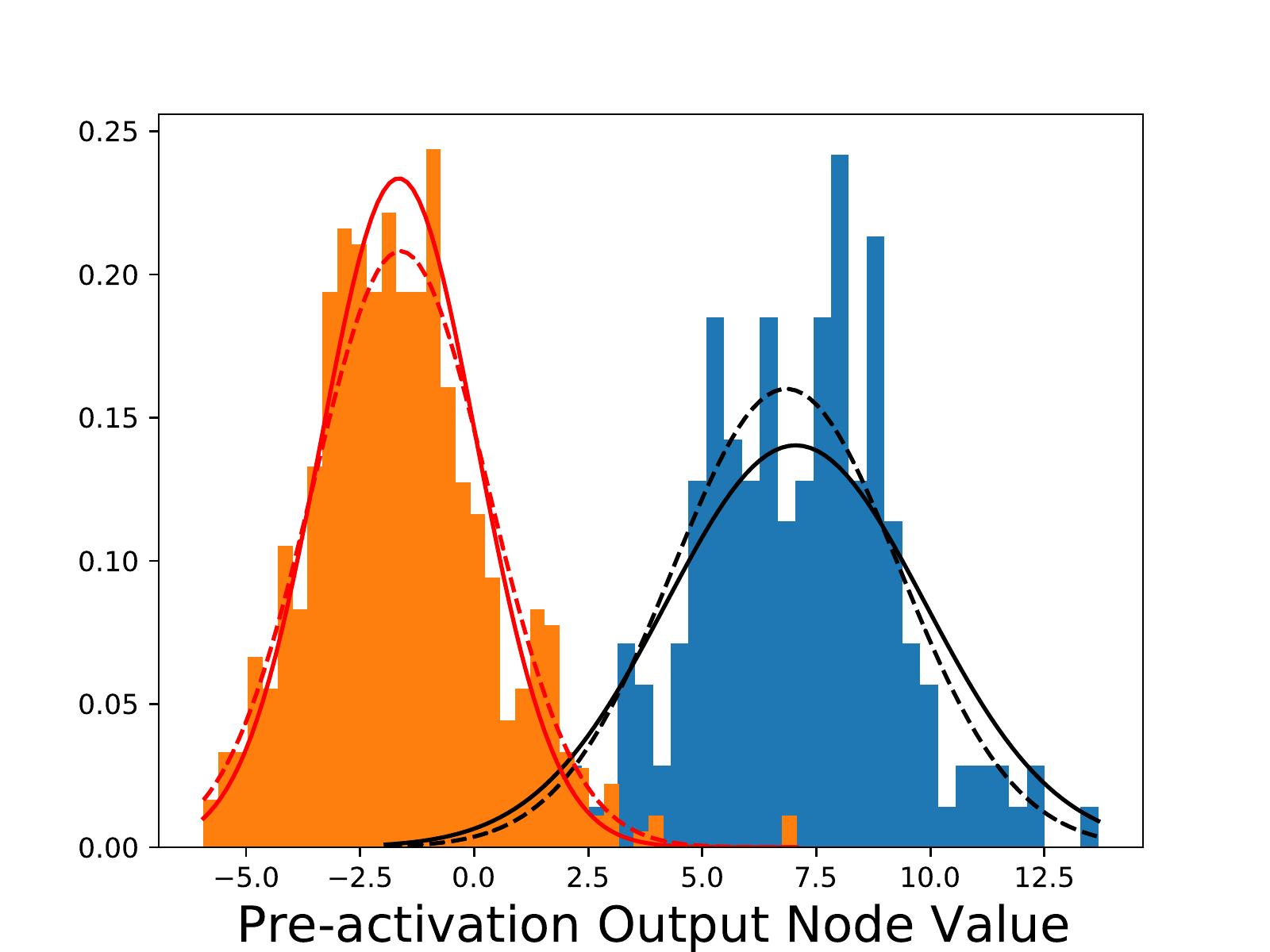}     &
    \includegraphics[width=0.18\linewidth]{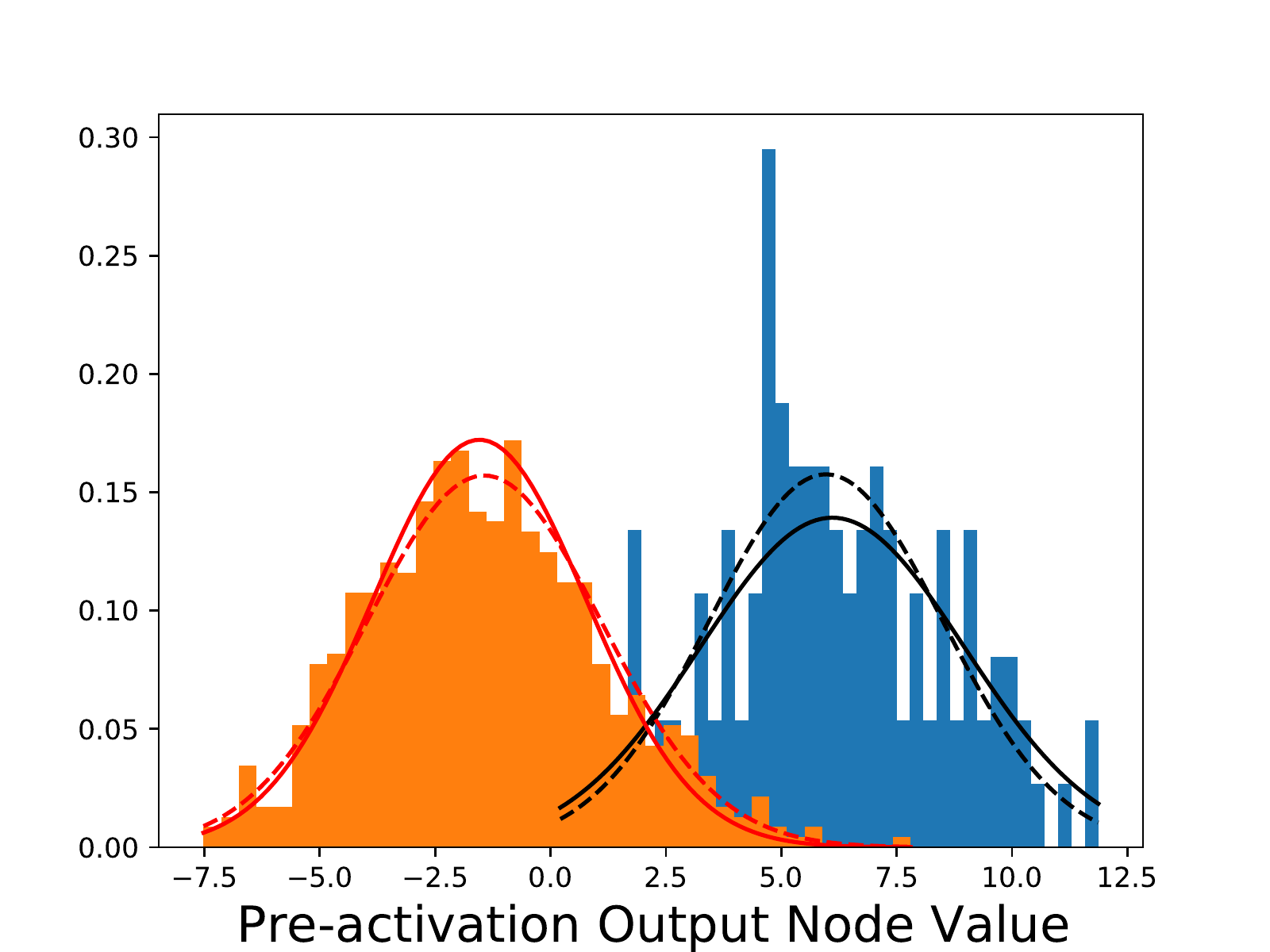}     &
    \includegraphics[width=0.18\linewidth]{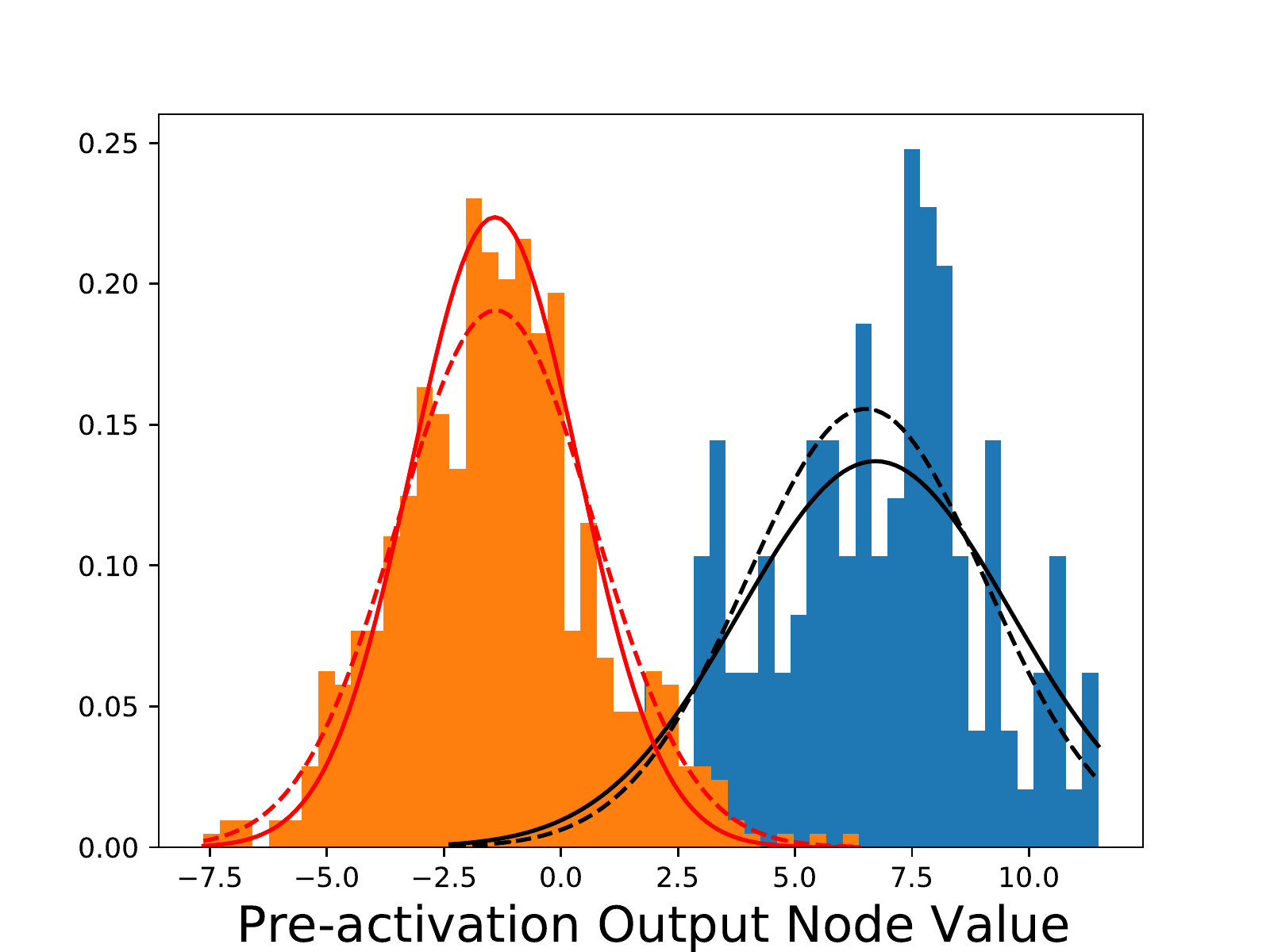}     &
    \includegraphics[width=0.18\linewidth]{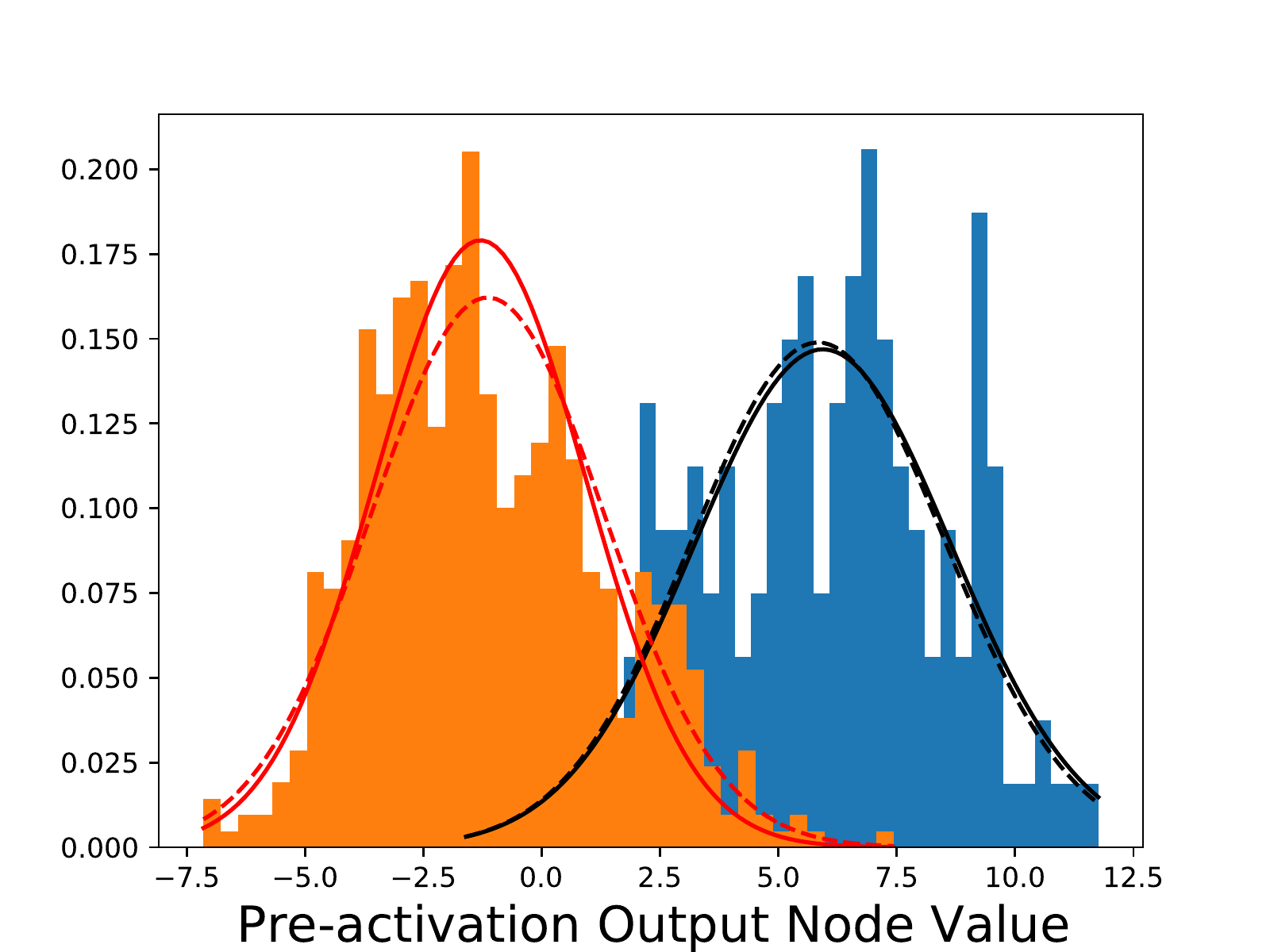} 
    \\
    Cls 1 & Cls 2 & Cls 3 & Cls 4 & Cls 5
    \end{tabular}
    \caption{Flowers: Classification layer node (pre-softmax activation) distributions. The predicted distribution is shown as a solid curve and is shown against the histogram obtained by direct observation of the corresponding output node values. The Gaussian fitted to the observed output values is shown in dotted lines, whilst the solid line is the predicted Gaussian based solely on last convolutional layer values. The red curve shows the distribution for the responses from the ``negative examples'' (i.e. examples with labels that do not match the output node class) and black curve (blue histogram) for ``positive examples'' (images with label corresponding to output node class).}
    \label{fig:flowers_gauss}
\end{figure*}



\subsection{Aggregation Layer Distribution Statistics}
\label{sec:exp_agg_layer_stats}
In this section, we use our proposed framework to analyse the underlying distributions of output values from the aggregation block layers. We visually show how the predicted distributions detailed in Section \ref{sec:agg_blocks} compare with the observed output value histograms for different layers in Figure \ref{fig:cifar_act_dist}. Here, the CIFAR10 dataset is used, and we have selected class 6 as the positive class. Shown are the output value distributions for the first filter of the last convolutional layer, exponential activation layer, GAP layer and deactivation layer. We can see how our proposed distribution formulations closely matches the corresponding histogram of observed values. Interestingly, we see how the initial Gamma distributed values of the last convolutional layer are squeezed into a smaller range by the exponential activation. The GAP and subsequent deactivation layer has the effect of pushing the mean value outwards. 

In order to comprehensively evaluate how close our proposed distributions are to the actual distributions, we compare the mean, variances and where applicable, covariances of the output values of each aggregation layer against those obtained by direct observed measurements. This can be visualised as scatter plots of predictions (x-axis) against observed values (y-axis). More specifically, for each dataset, each class set as ``positive'', every filter in each aggregation layer, we obtain a prediction of their respective output value distribution and subsequently predict the corresponding mean and variances. We also record these output values, and compute the corresponding observed mean and standard deviations. This enables us to produce a point (observed vs prediction) on the scatter plot described above. We next show and analyse these comparison scatter plots for each dataset separately.

\subsubsection{CIFAR10 Dataset}
The resulting scatter plots for the CIFAR10 dataset can be seen in Figure \ref{fig:cifar_mean_std}. Each point on a scatter plot represents the (observation,prediction) pair for a single filter. The different colours represent different classes set as positive. In all the plots, the diagonal line is also shown. Any point lying on the diagonal line indicates that the prediction and observation have equal values. 

The scatter plot for the mean and standard deviation of the exponential activation layer can be seen in Fig. \ref{fig:cifar_mean_std}a,b respectively. We see that there is very good agreement between the predicted values of the mean of exponentially activated responses, as shown by the points lying very close or on the diagonal line. In terms of the standard deviations, we find that there is a slight over prediction compared with the observed standard deviation. One possible reason is that the Gamma distribution is only an approximate (albeit closely fitting) of the last convolutional layer output distribution. 

The mean and standard deviation scatter plots for the GAP feature layer is shown in Figure \ref{fig:cifar_mean_std}c,d respectively. We find that the predicted and observed mean and standard deviation values match very well. This is the case across all the different classes and GAP features. Additionally, our proposed formulation is also able to estimate the covariance values between different GAP features as shown in Figure \ref{fig:cifar_cov_mat}a.


The predicted and observed mean and standard deviations for the deactivation layer scatter plots is shown in Figure \ref{fig:cifar_mean_std}e,f. We see a similar level of agreement between the predictions and observations of the deactivation layer output mean and standard deviation as those for the GAP layer. The covariance matrix between deactivated features is shown in Figure \ref{fig:cifar_cov_mat}b. We see that in general, the predicted and observed covariance matrix elements are in good agreement. Some variation is present in the predictions compared with its corresponding observed values. This is due to the error present in the use of a first order approximation for the covariance matrix equation on the deactivation layer. 

Finally, we can see the agreement between the output node Gaussians mean and standard deviations in Figure \ref{fig:cifar_act_dist}g,h respectively. The mean values for both the positive (circles) and negative class (thick cross) Gaussians lie on the diagonal line, showing very good agreement between the predicted and observed values. This can also be seen in the output Gaussian distributions shown in Figure \ref{fig:cifar_out_gauss}. The predicted Gaussian standard deviation exhibit some variation around the observed values due to the small inaccuracies present in the deactivated features covariance matrix estimation (as shown in Fig. \ref{fig:cifar_cov_mat}).


\subsubsection{Flowers Dataset}
The scatter plots for the Flowers dataset can be seen in Figure \ref{fig:flower_mean_std} and is shown in a similar format to that for the CIFAR10 dataset. The scatter plots for the predicted-vs-observed mean and standard deviation values of the exponential activated layers can be seen in Figure \ref{fig:flower_cov_mat}a,b. We find that there is also good agreement between the predicted and observed mean of exponentially activated features. However, there is a greater amount of variation present for the standard deviation predictions. 

As with the CIFAR dataset, the GAP layer shows much better agreement between our predicted mean and standard deviation values when compared with the corresponding observed values, as shown in Figure \ref{fig:flower_mean_std}c,d respectively. The scatter plot of the predicted-vs-observed GAP features covariance matrix elements is shown in Figure \ref{fig:flower_cov_mat}.

In the deactivation layer, the predicted mean again matches well to the observed values. However, the errors in the predicted GAP layer covariance matrix have caused the under-prediction of the deactivated features standard deviation. We also see that the first order approximation for estimating the covariance matrix has caused it to over predict the deactivation covariance matrix element values, as shown in Figure \ref{fig:flower_cov_mat}.

Despite the above issues, the predicted mean values for the output node Gaussians still match very well to the observed mean values,
as they mainly depends on the mean values of the deactivated features. However, we can see that the standard deviation values of the output node Gaussians is affected to a greater extent. However, the predicted standard deviation values still agree in general to those observed.

\begin{figure}
    \centering
    \begin{tabular}{cc}
             \includegraphics[width=0.35\linewidth]{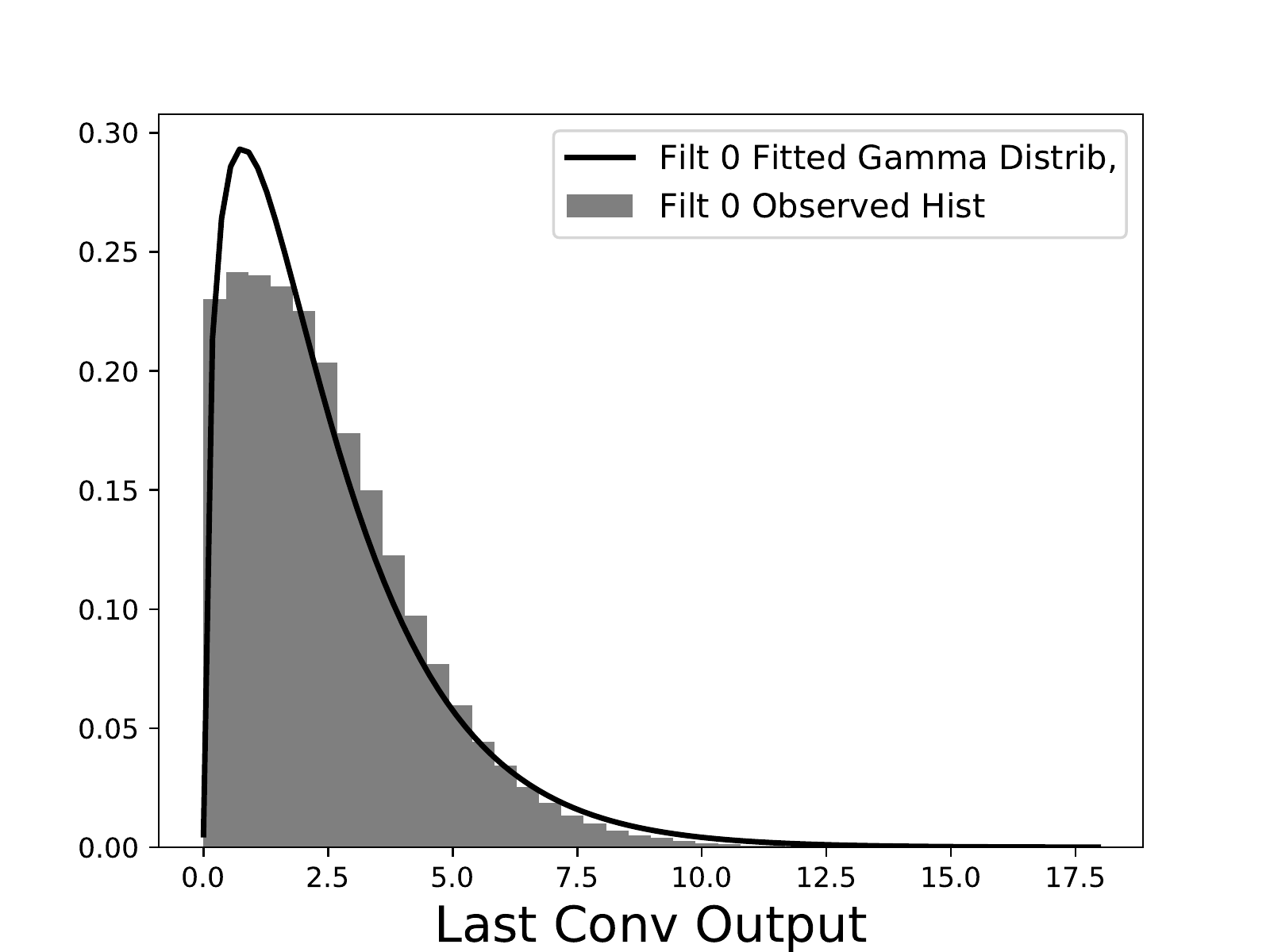}&  \includegraphics[width=0.45\linewidth]{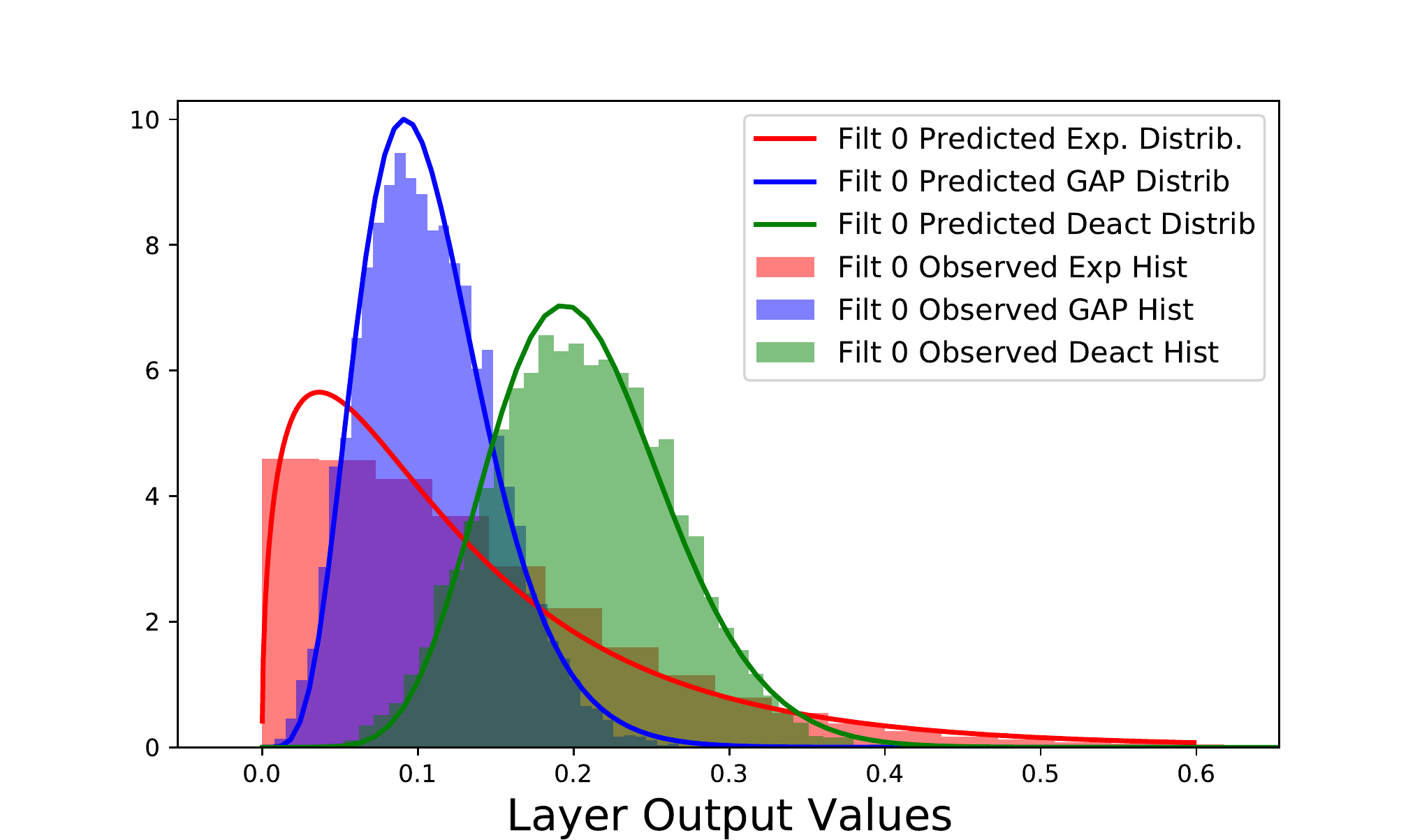}\\
             (Cls 6) Last Conv. Layer & Aggr. Layers \\
             (a) & (b) \\
        \hline
         \includegraphics[width=0.35\linewidth]{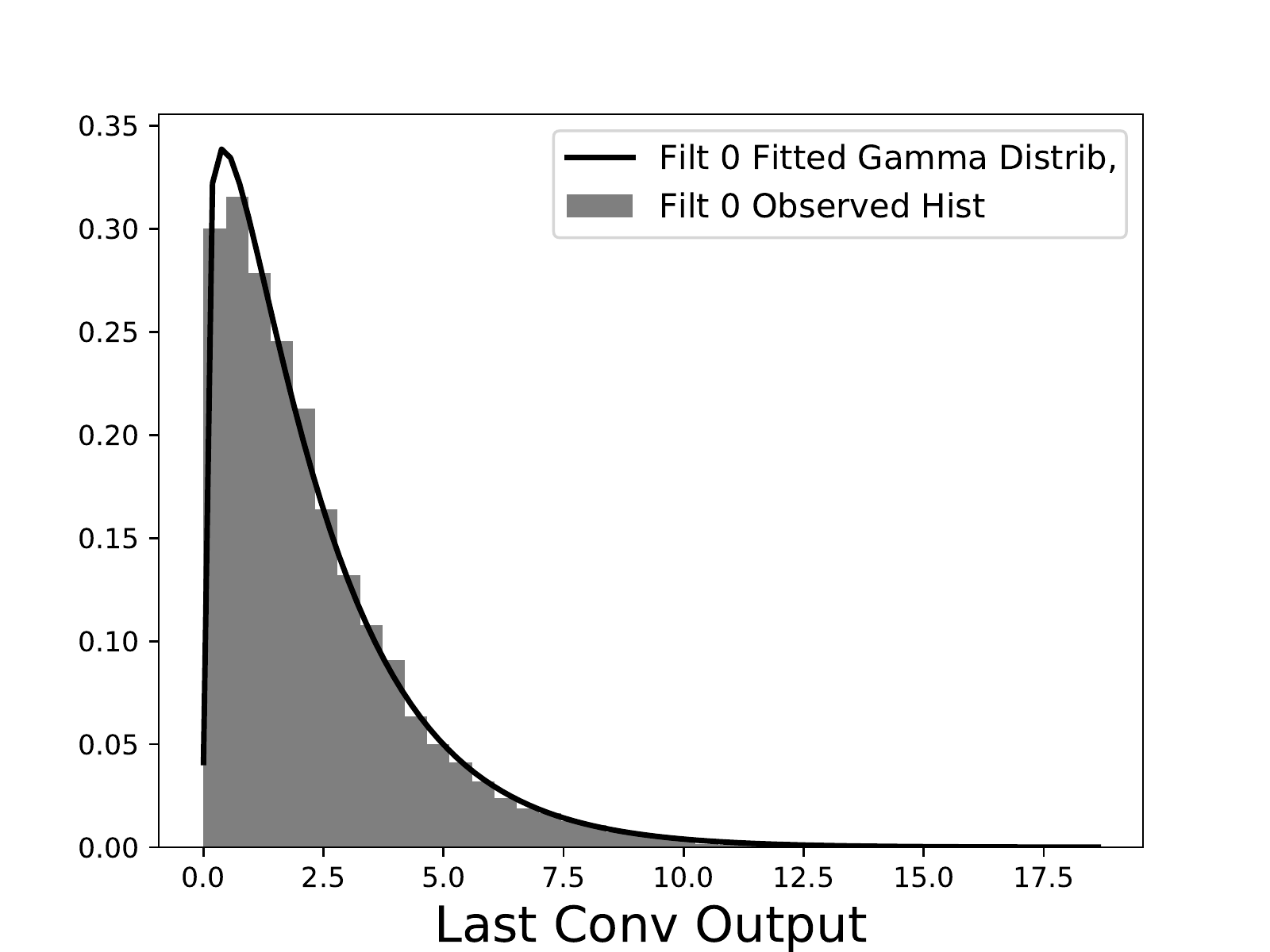}&
         \includegraphics[width=0.45\linewidth]{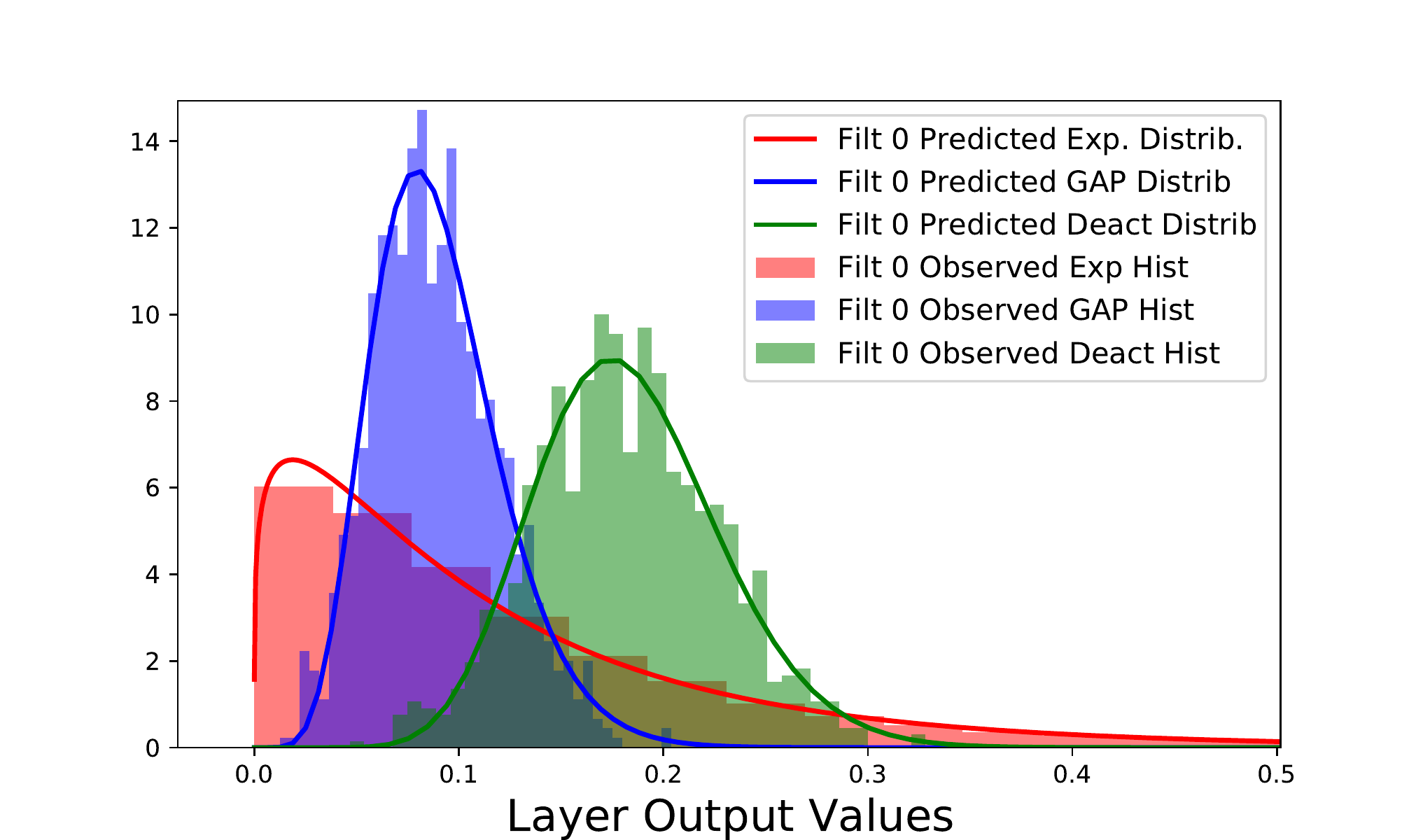} \\

          (Not Cls 6) Last Conv. Layer & Aggr. Layers \\
          (c) & (d)
    \end{tabular}

    \caption{Predicted and observed distributions of output values of different layers in the activation block for the CIFAR10 dataset. The curves show the predicted distribution of output values for corresponding activation block layers using the equations from Section \ref{sec:agg_blocks}. We also show the histogram of observed output values from corresponding layers for comparison. (a) and (b) show the distributions when class 6 examples are given to the DNN. (c) and (d) shows the layer output distributions when examples from other classes are given to the DNN.  }
    \label{fig:cifar_act_dist}
\end{figure}

\begin{figure}
    \centering
    \begin{tabular}{cc}
        \includegraphics[width=0.5\linewidth]{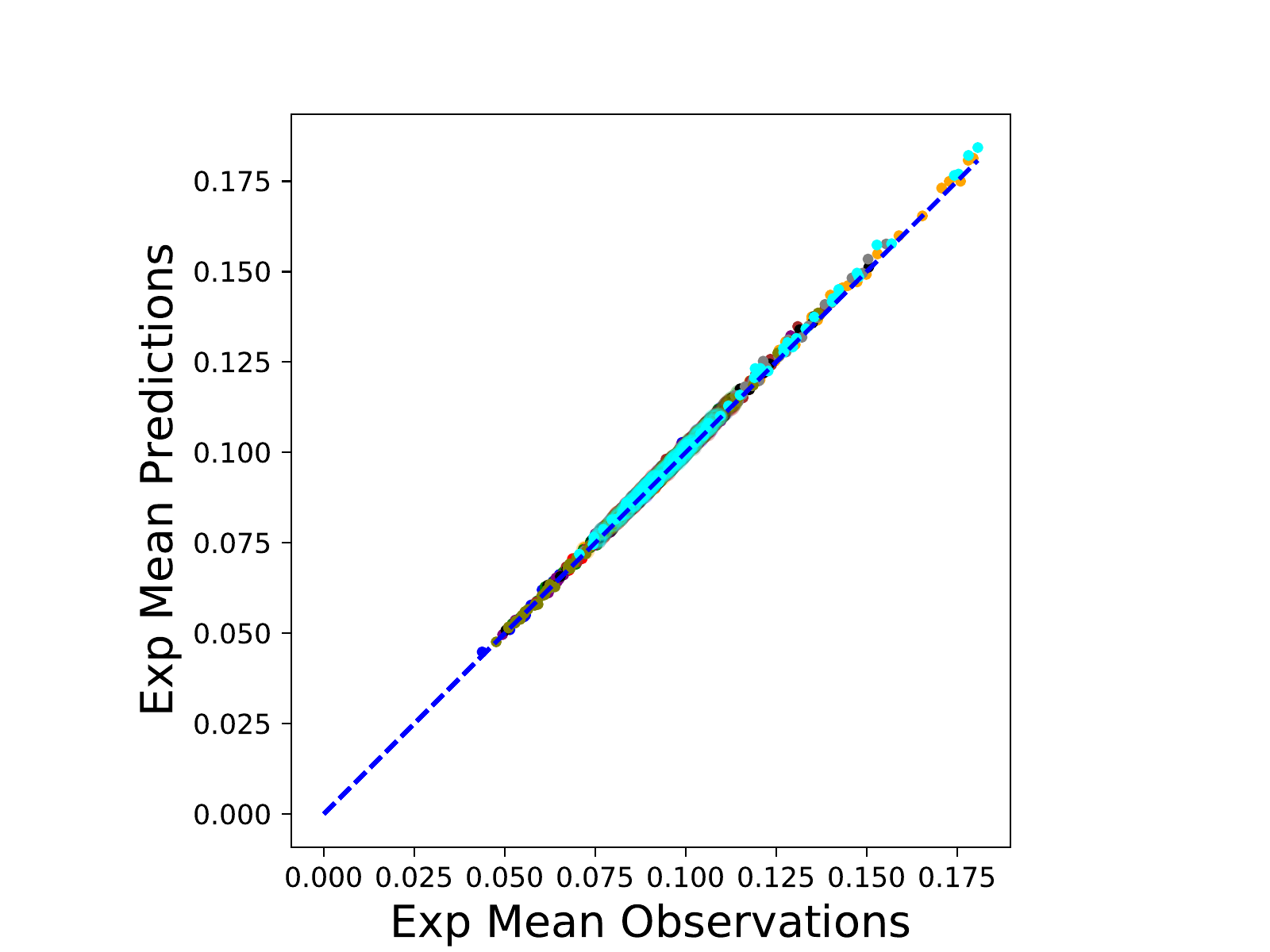}
     &  \includegraphics[width=0.5\linewidth]{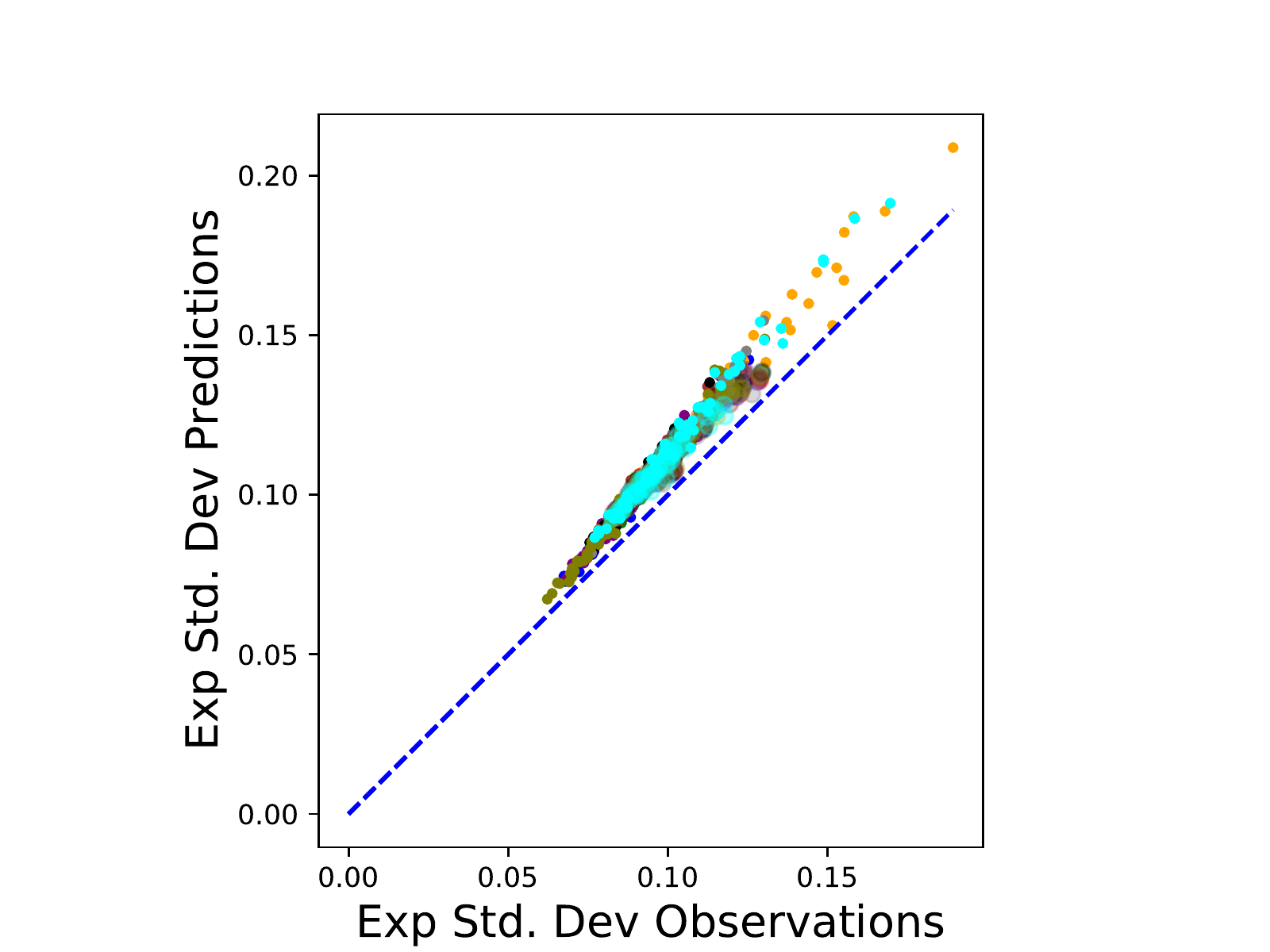}\\
       (a)  & (b)\\
       \includegraphics[width=0.5\linewidth]{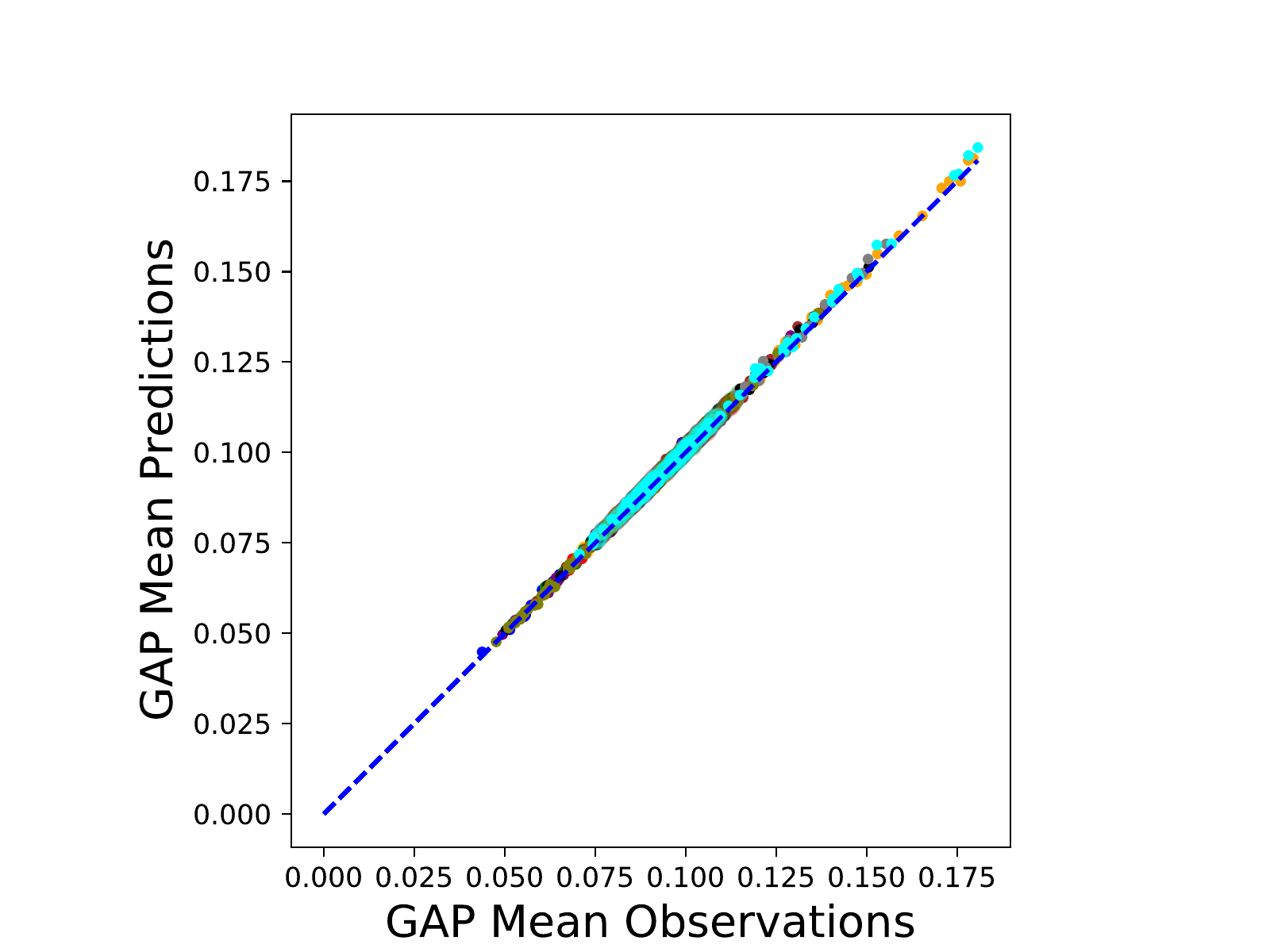}
     &  \includegraphics[width=0.5\linewidth]{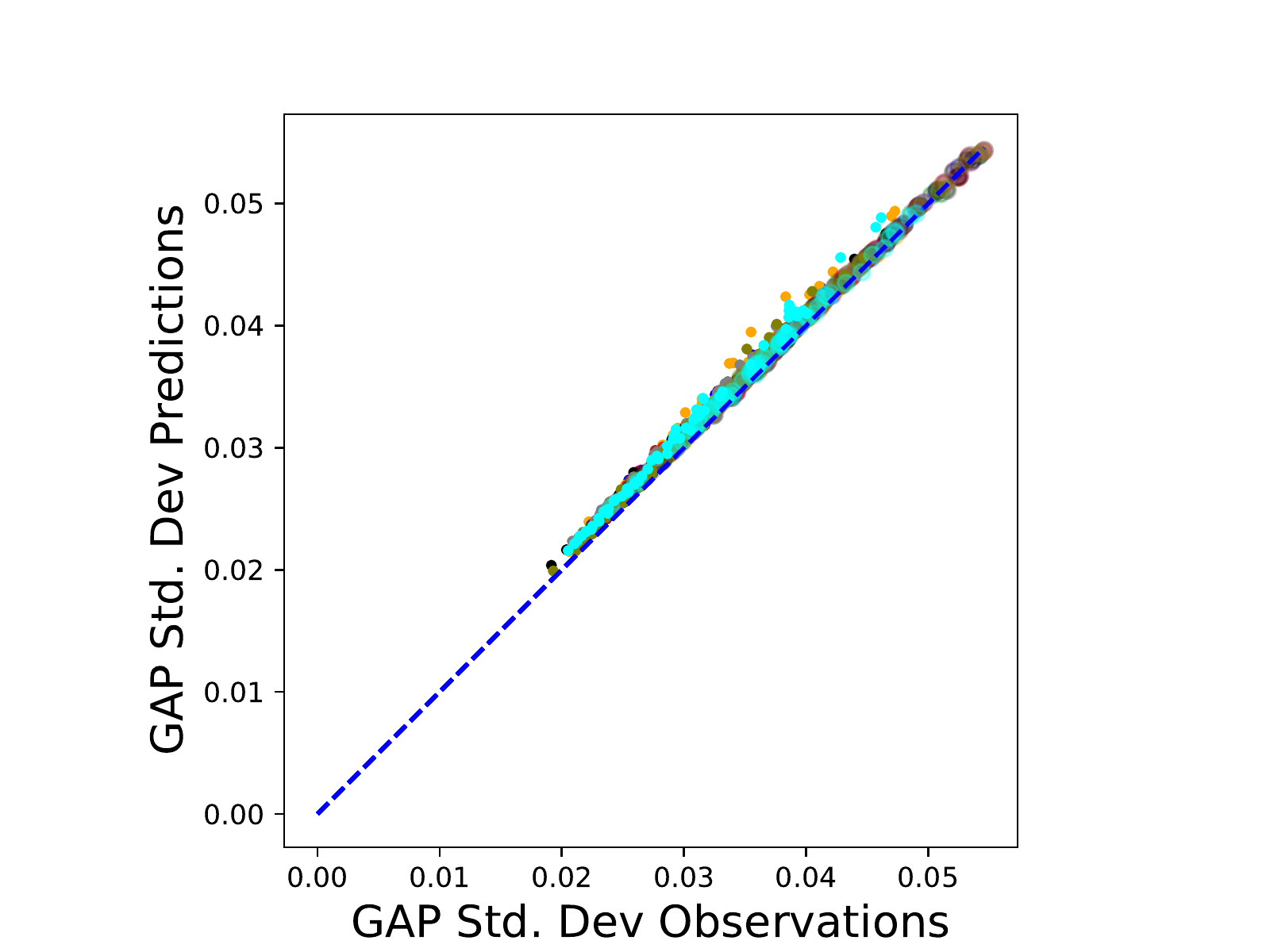}\\
       (c)  & (d) \\
       \includegraphics[width=0.5\linewidth]{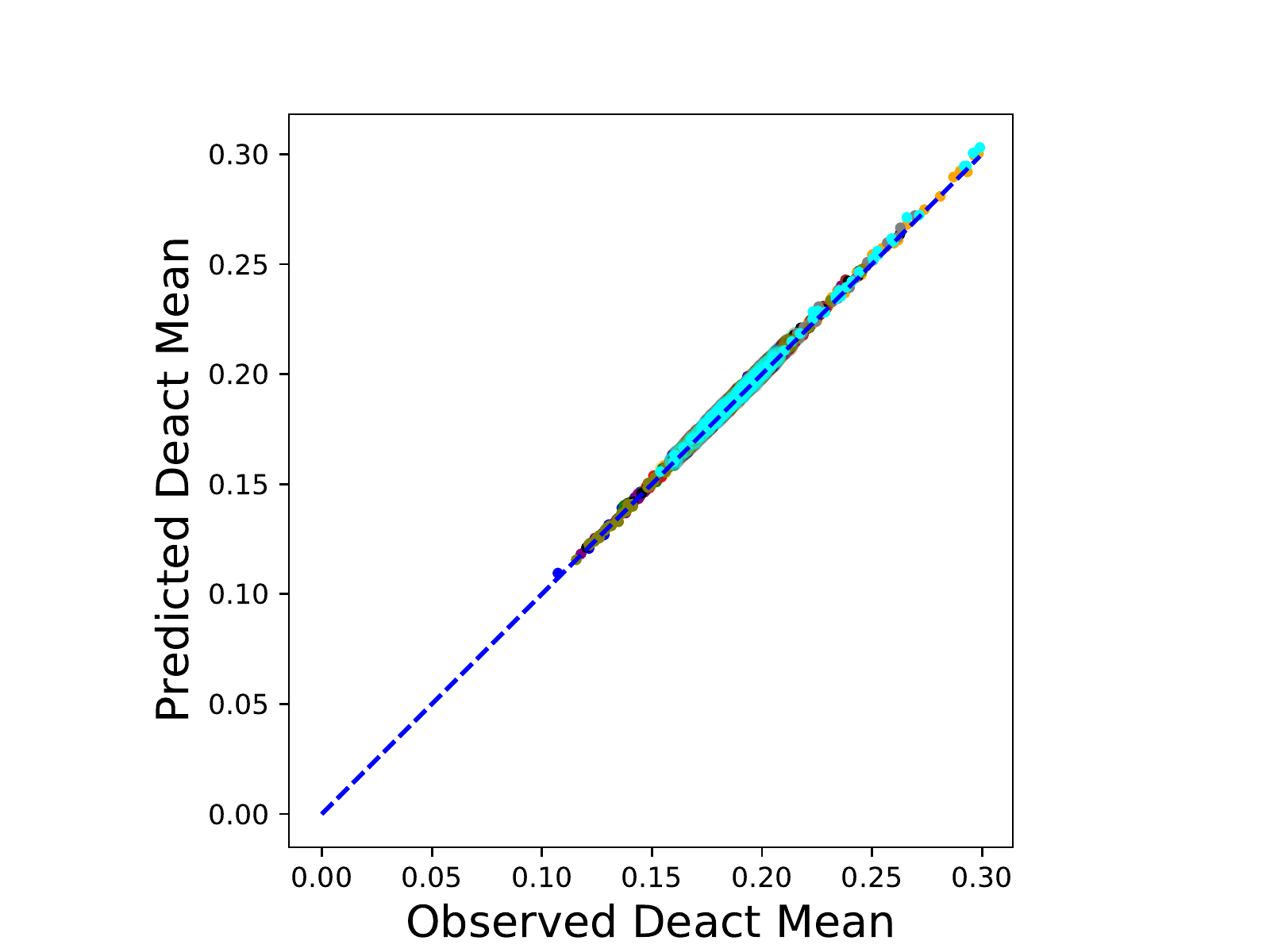}
     &  \includegraphics[width=0.5\linewidth]{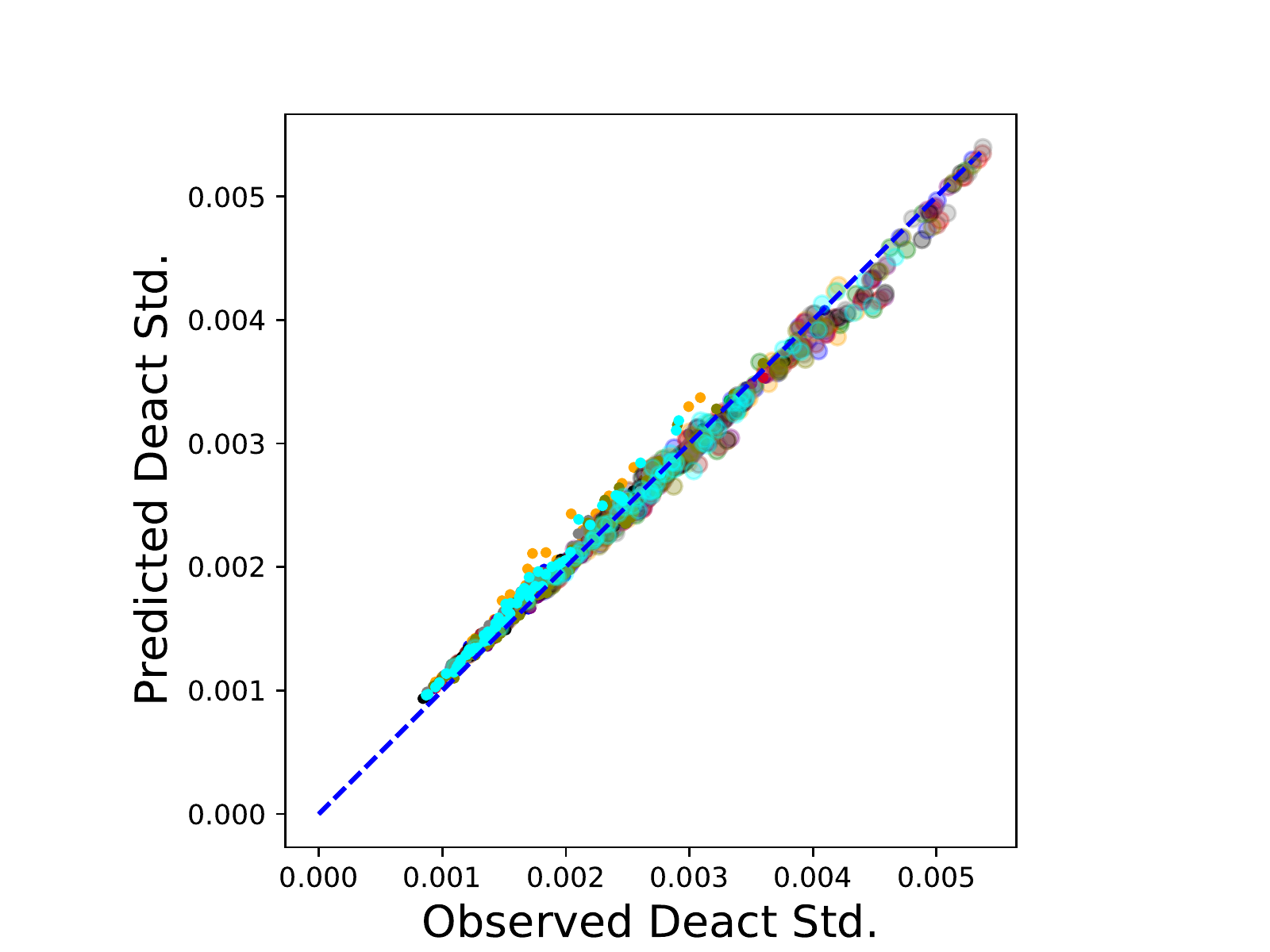}\\
       (e)  & (f) \\
       \includegraphics[width=0.5\linewidth]{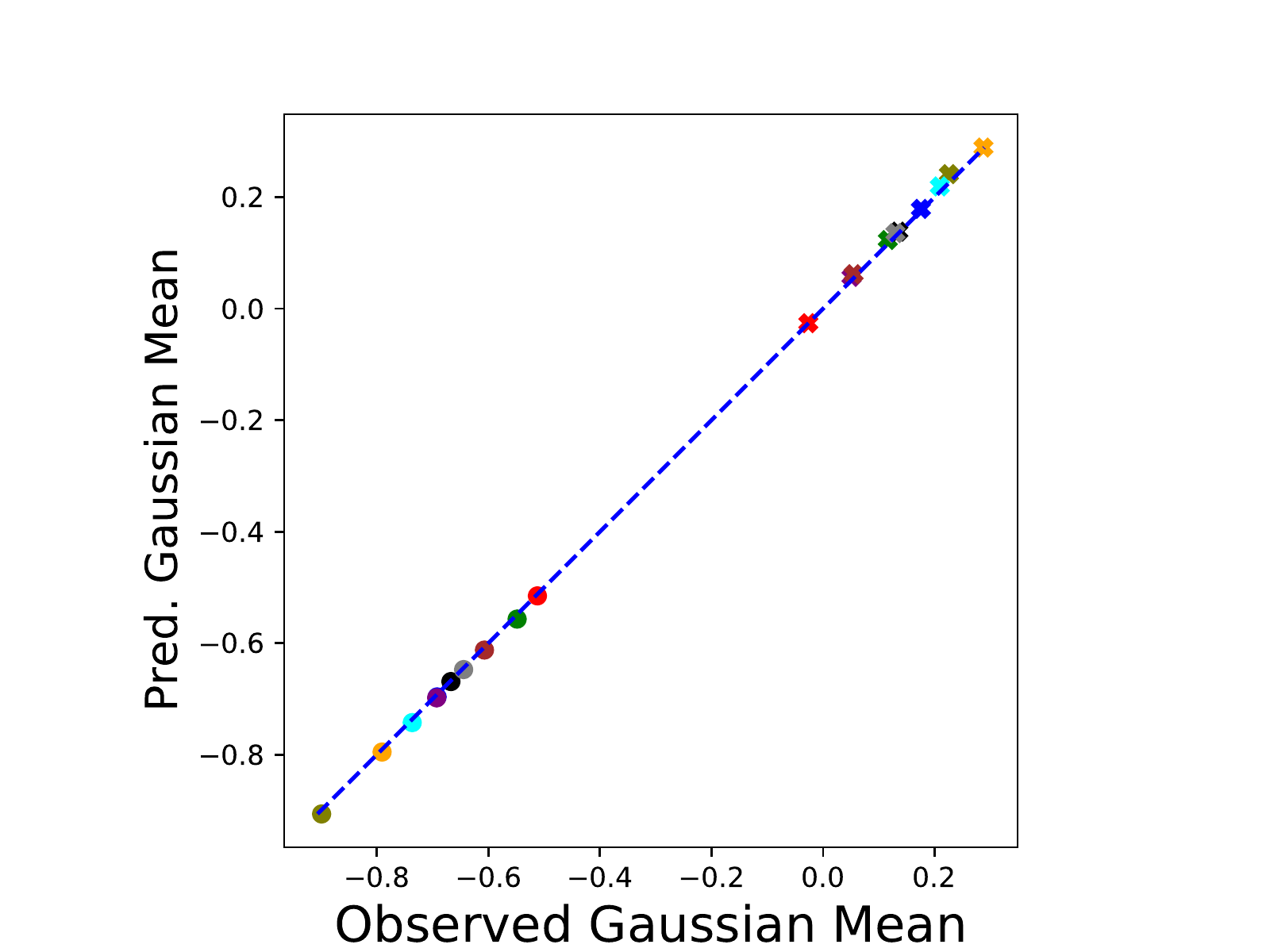}
     &  \includegraphics[width=0.5\linewidth]{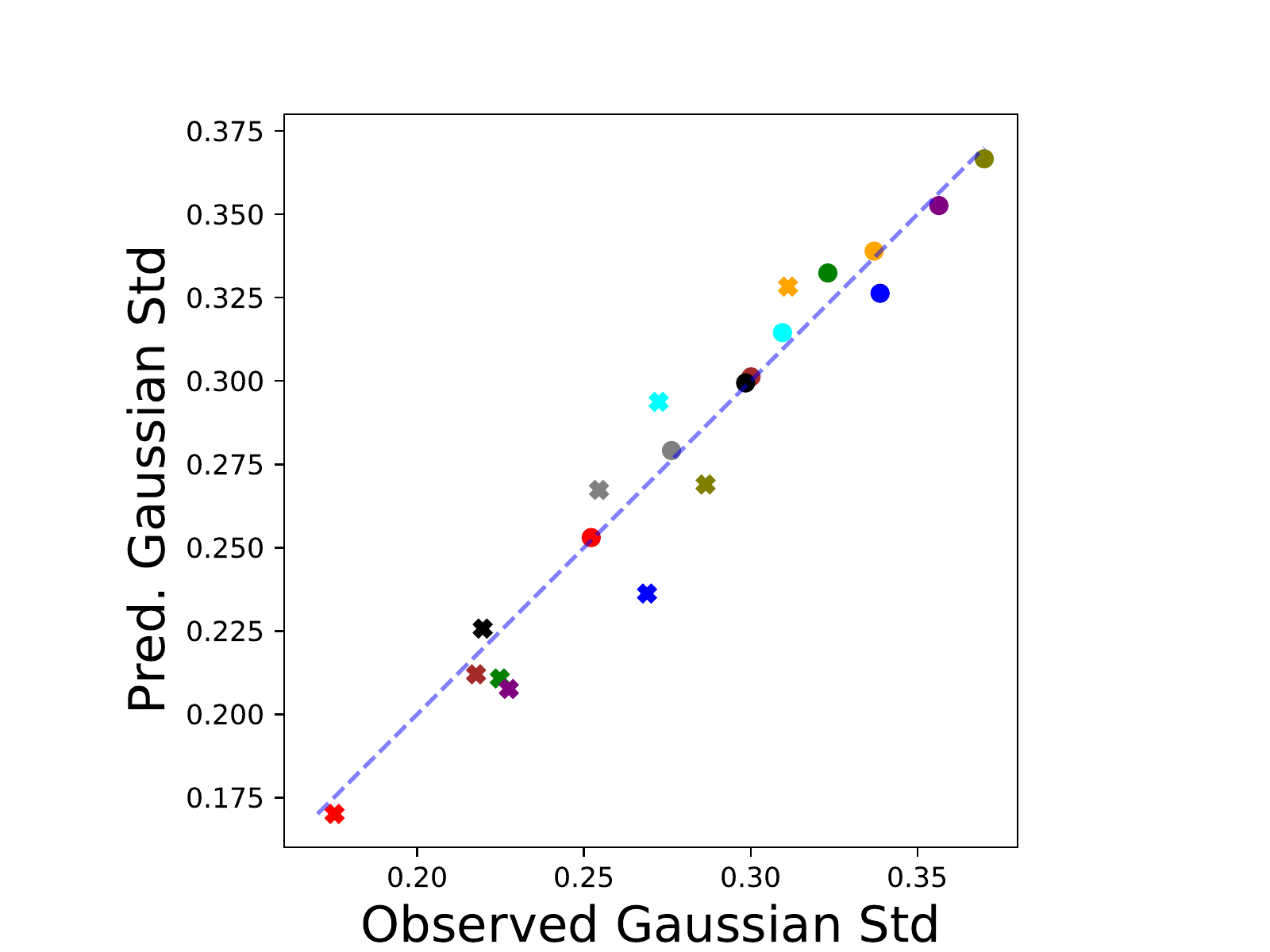}\\
       (g)  & (h)
    \end{tabular}
    \caption{CIFAR10: All plots show the observed statistics (x-axis) against predicted statistics (y-axis) of output value distributions for the different layers in the aggregation block. The left column shows the distribution mean values, whilst the right column shows the standard deviation values. Different colours represent different classes treated as positive. For reference, the diagonal line is shown as a dotted line.
    (a) and (b) for the Exponential Layer Activation Layer, (c) and (d) for the GAP layer, (e) and (f) for the deactivation layer and (g),(h) for the fully connected classification layer. }
    \label{fig:cifar_mean_std}
\end{figure}

\begin{figure}
    \centering
    \begin{tabular}{cc}
         \includegraphics[width=0.45\linewidth]{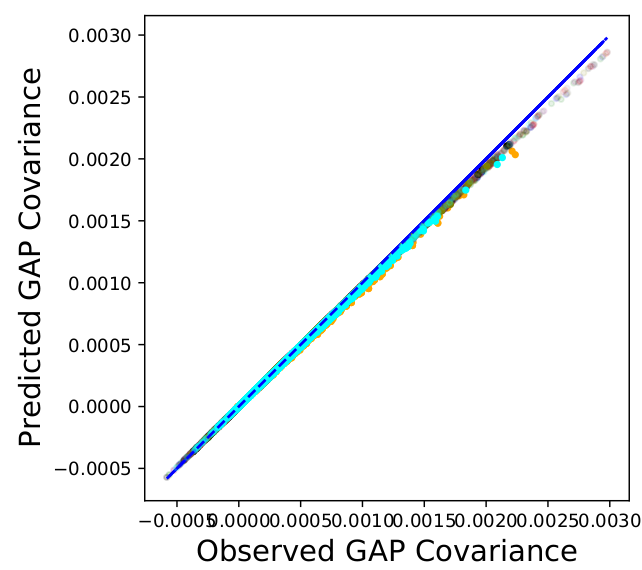} & 
         \includegraphics[width=0.45\linewidth]{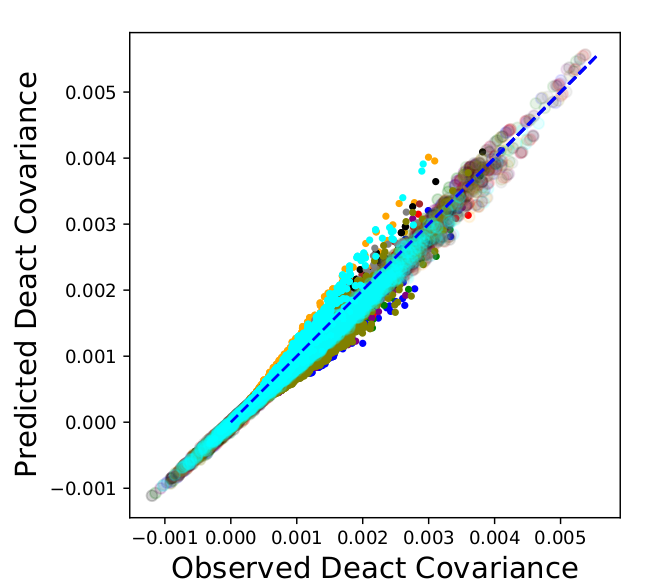}  \\
         (a) GAP Layer & (b) Deact. Layer 
    \end{tabular}
    \caption{CIFAR10: Aggregation layer covariance matrix scatter plot showing value of elements of the observed covariance matrix (x-axis) against corresponding element values of the predicted covariance matrix (y-axis).  }
    \label{fig:cifar_cov_mat}
\end{figure}

\begin{figure}
    \centering
    \begin{tabular}{cc}
        \includegraphics[width=0.5\linewidth]{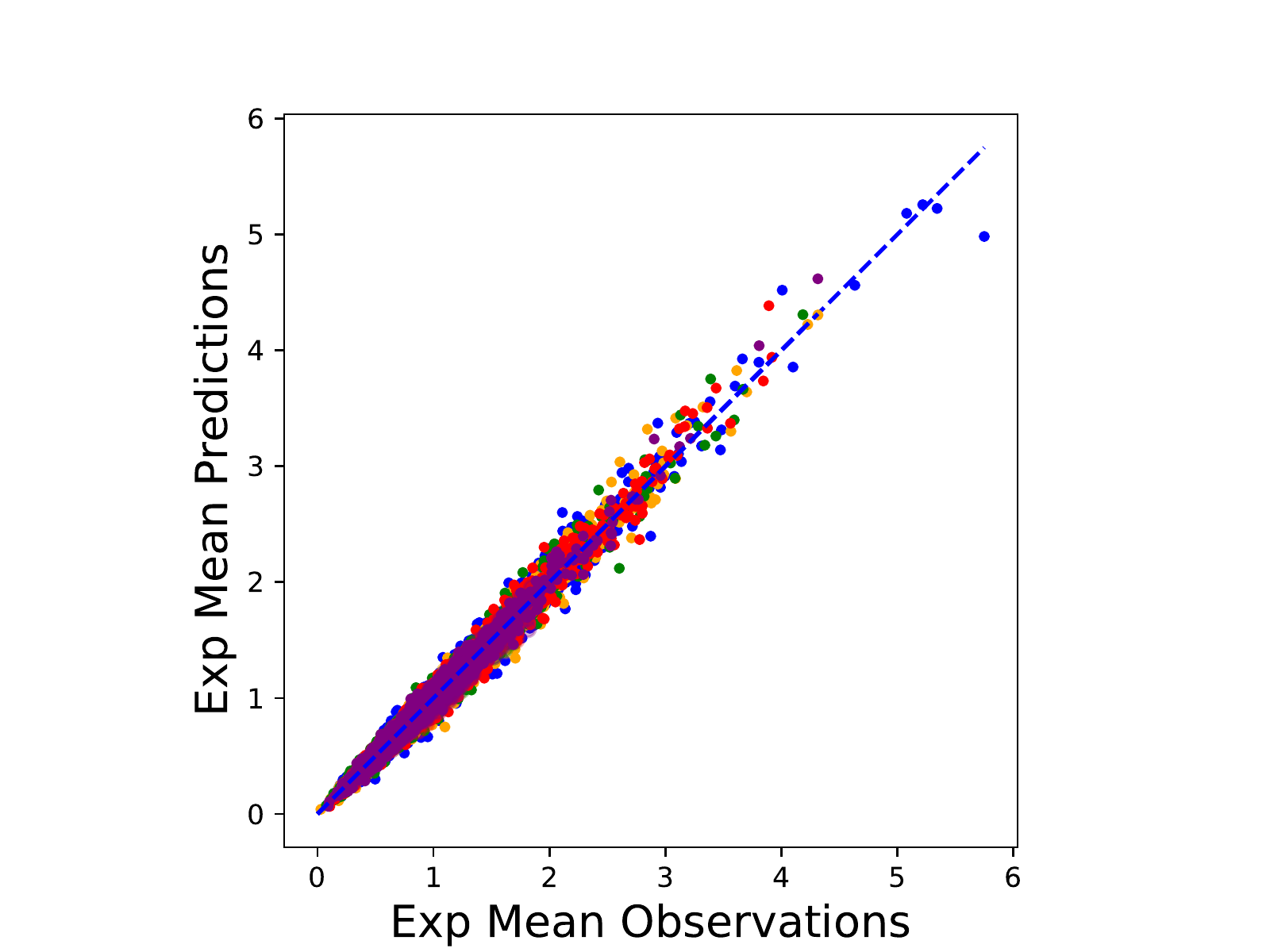}
     &  \includegraphics[width=0.5\linewidth]{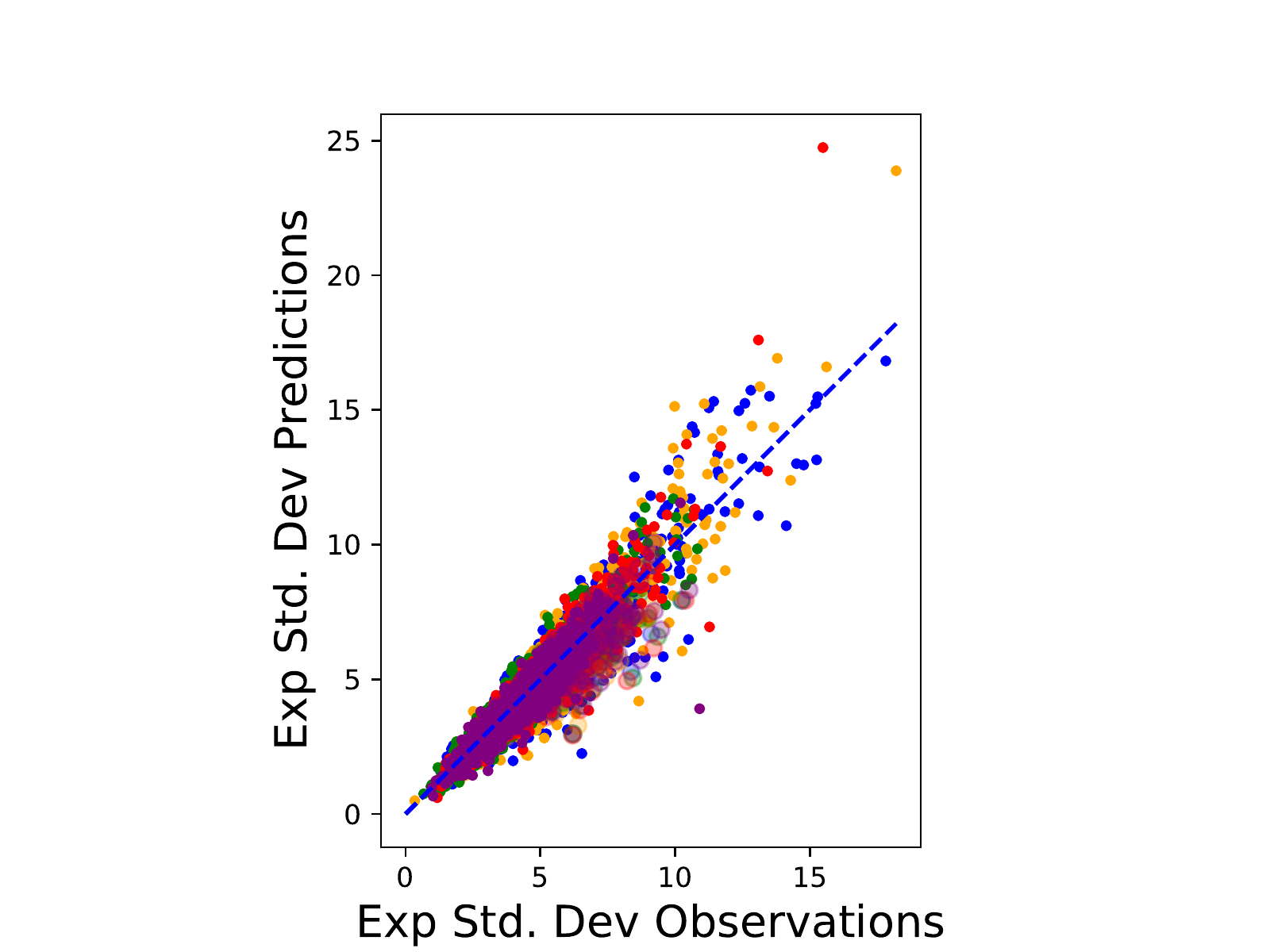}\\
       (a)  & (b)\\
       \includegraphics[width=0.5\linewidth]{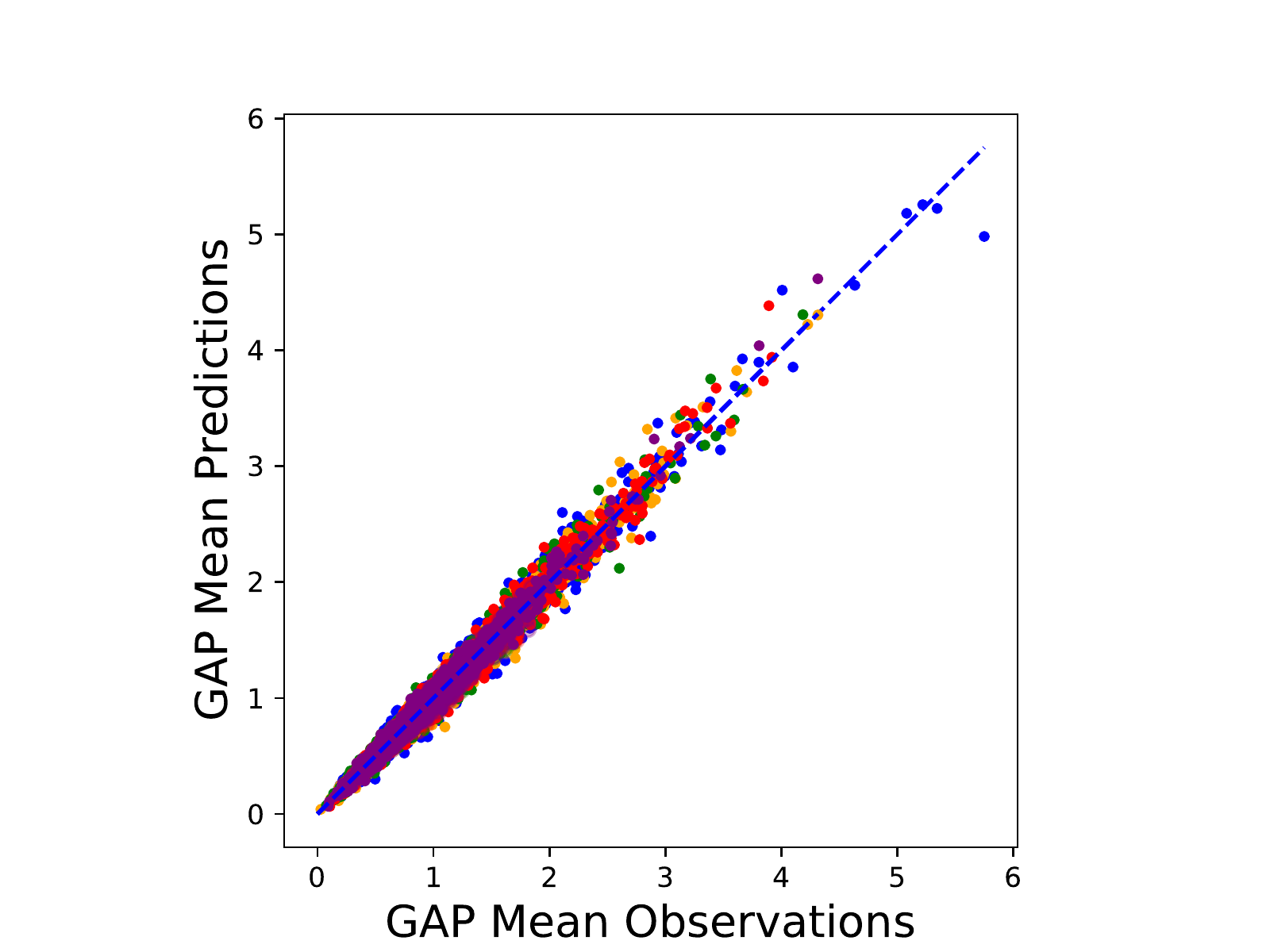}
     &  \includegraphics[width=0.5\linewidth]{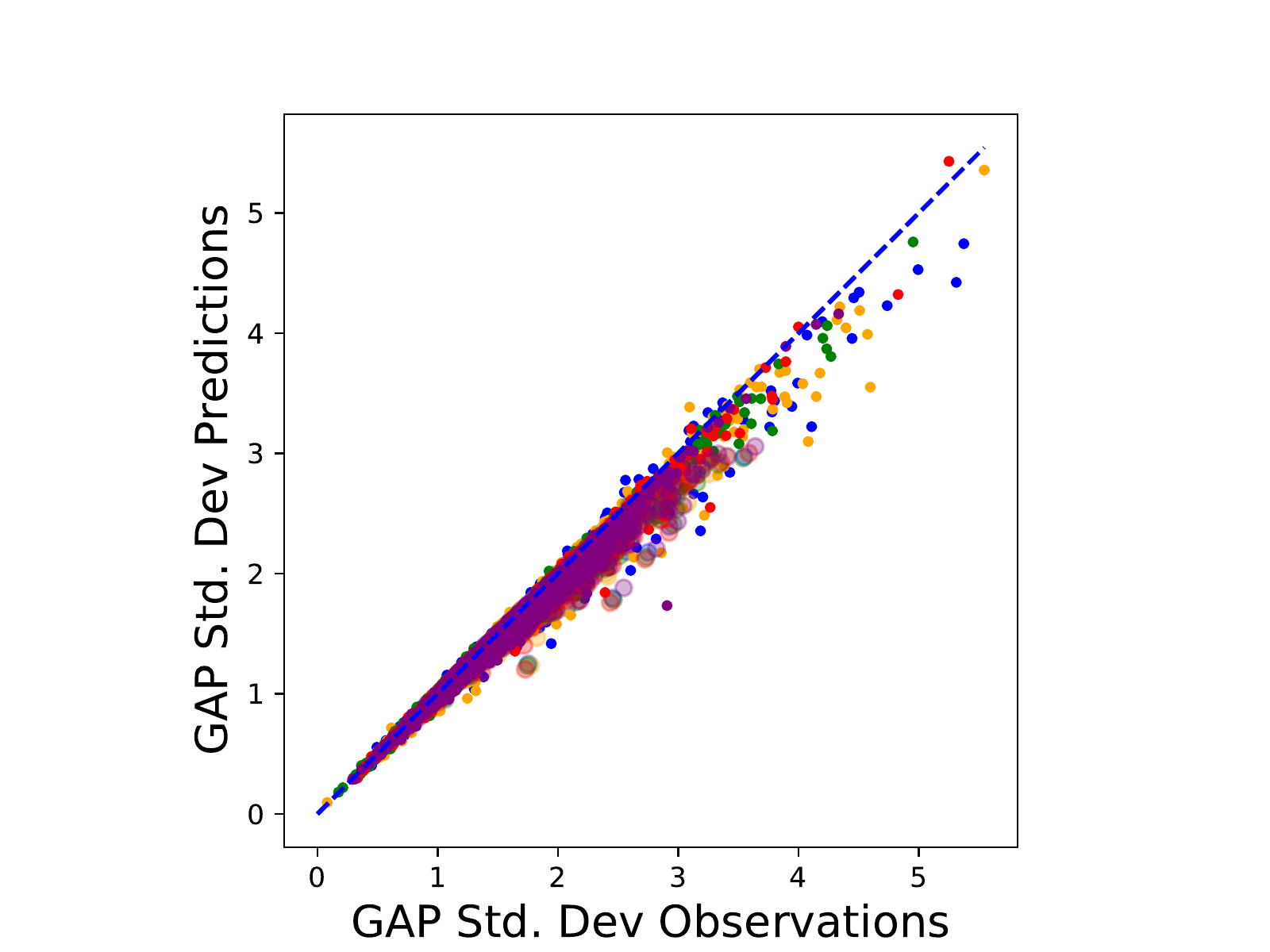}\\
       (c)  & (d) \\
       \includegraphics[width=0.5\linewidth]{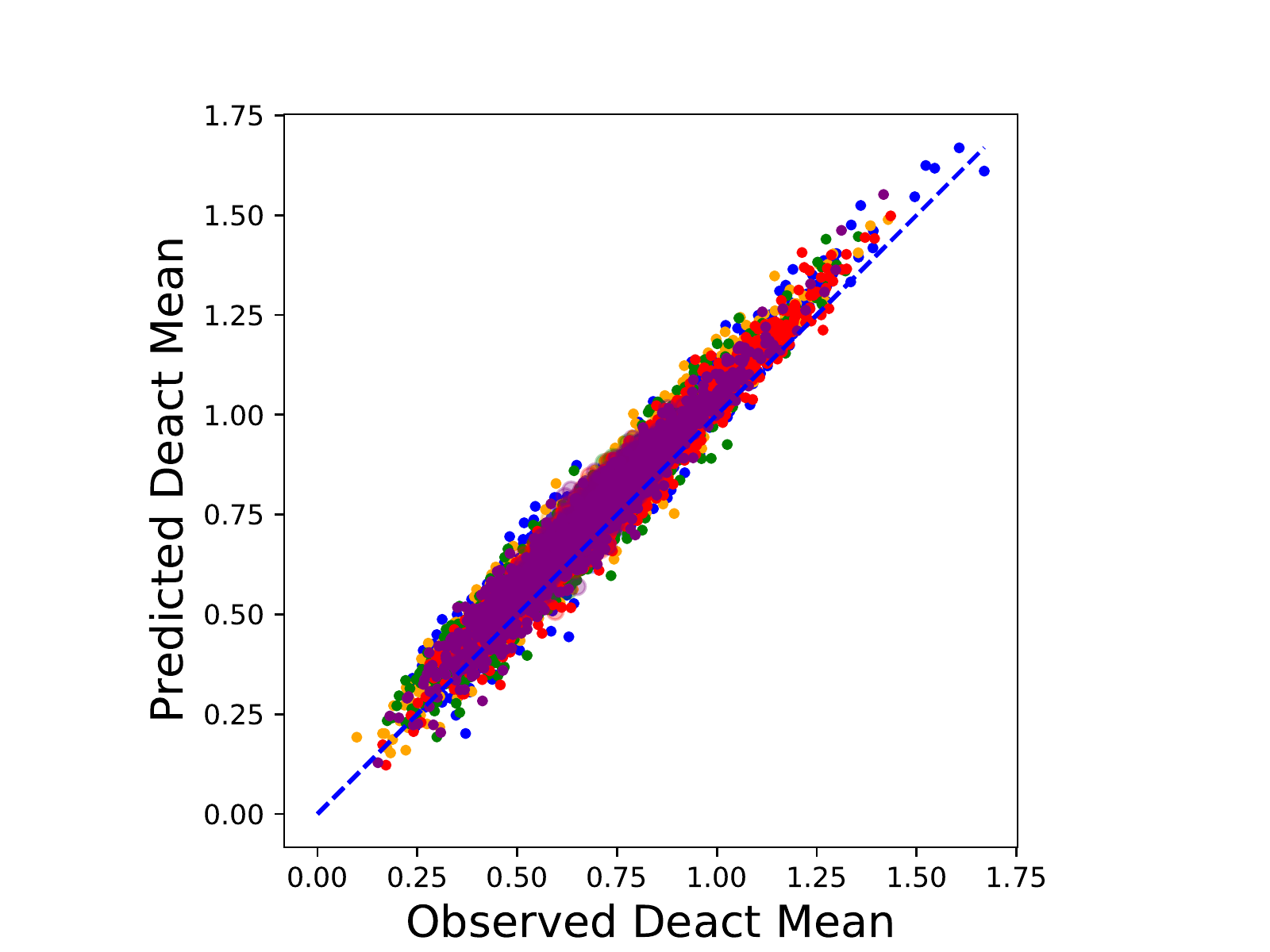}
     &  \includegraphics[width=0.5\linewidth]{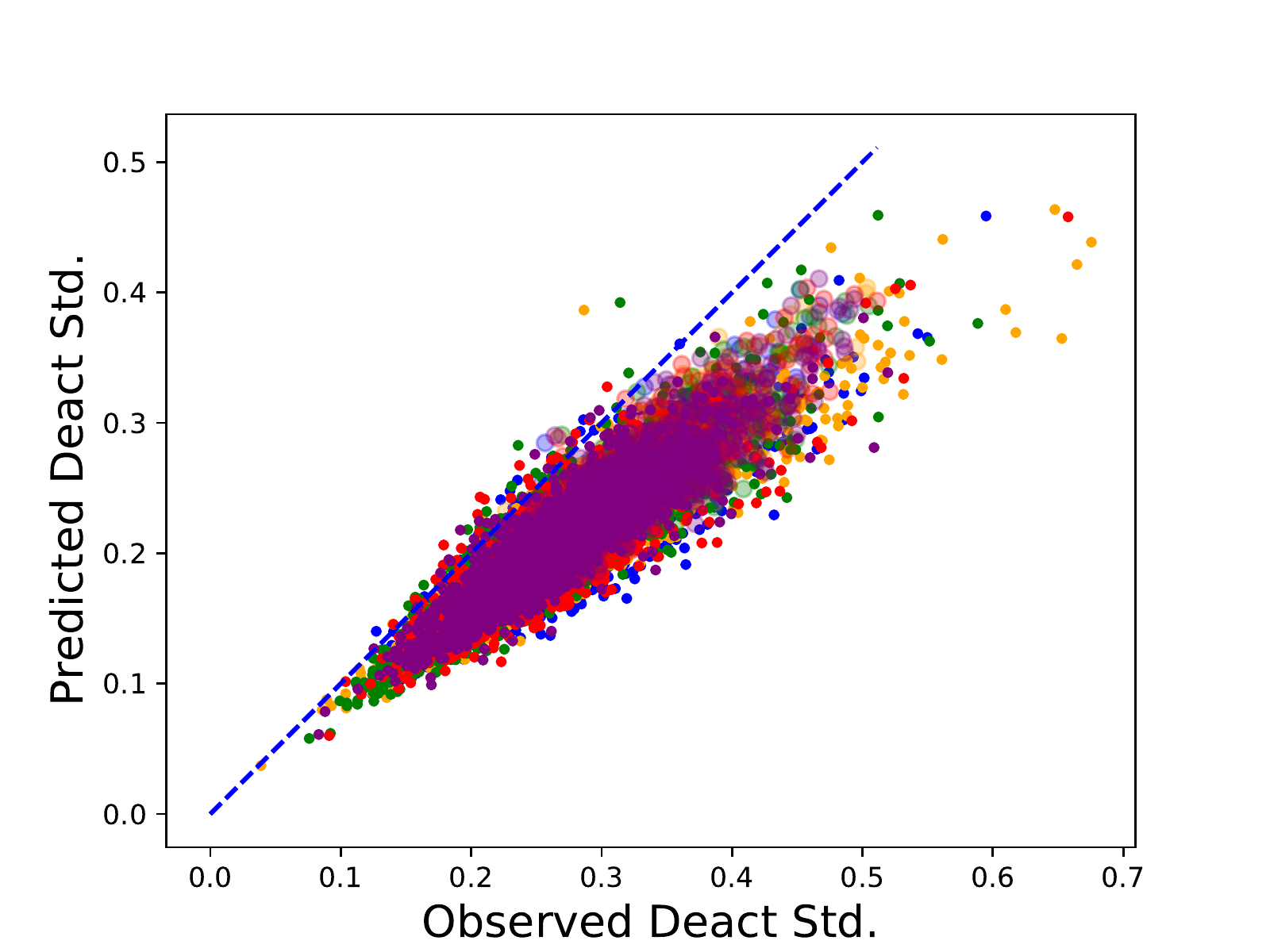}\\
       (e)  & (f) \\
       \includegraphics[width=0.5\linewidth]{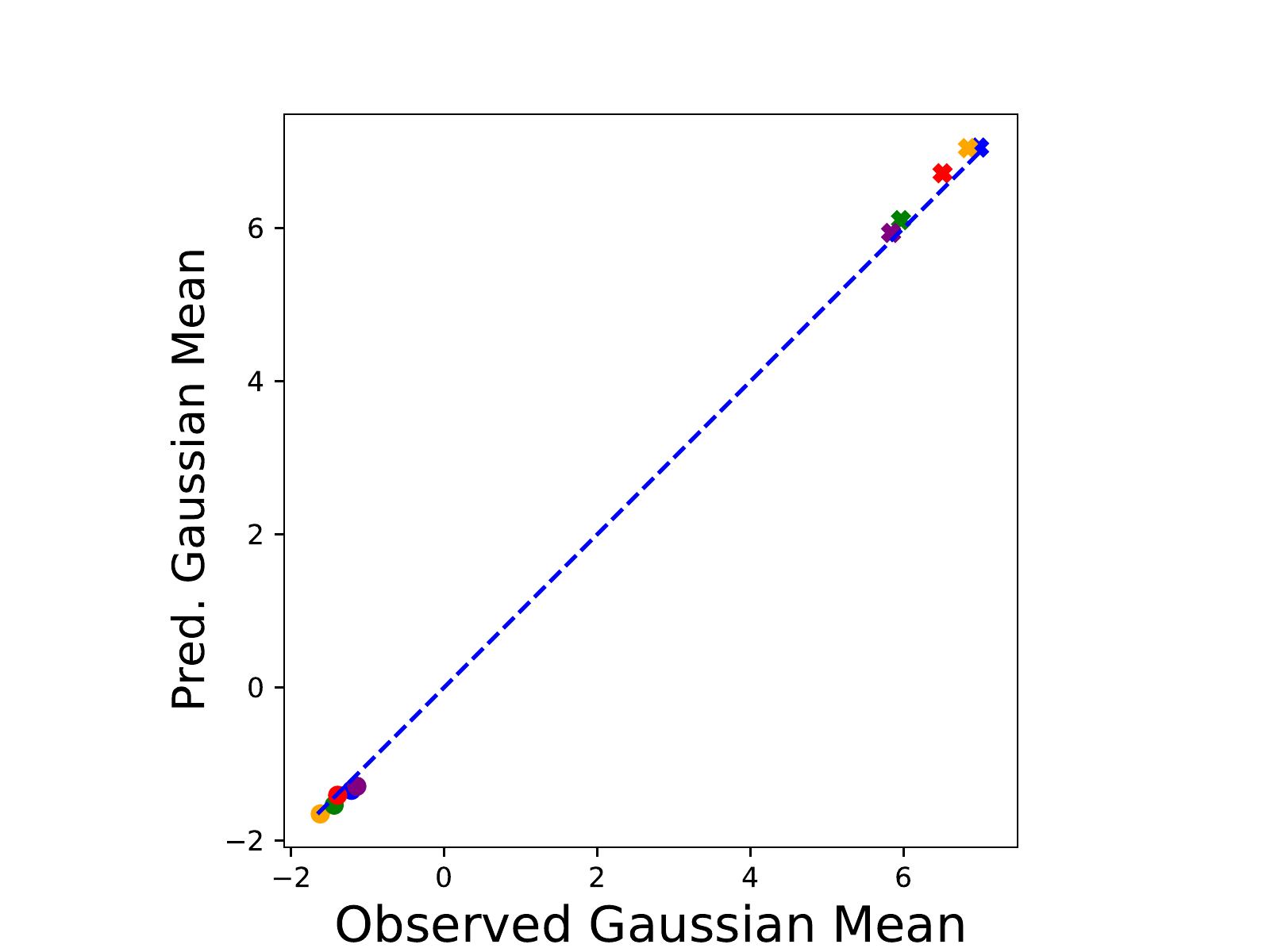}
     &  \includegraphics[width=0.5\linewidth]{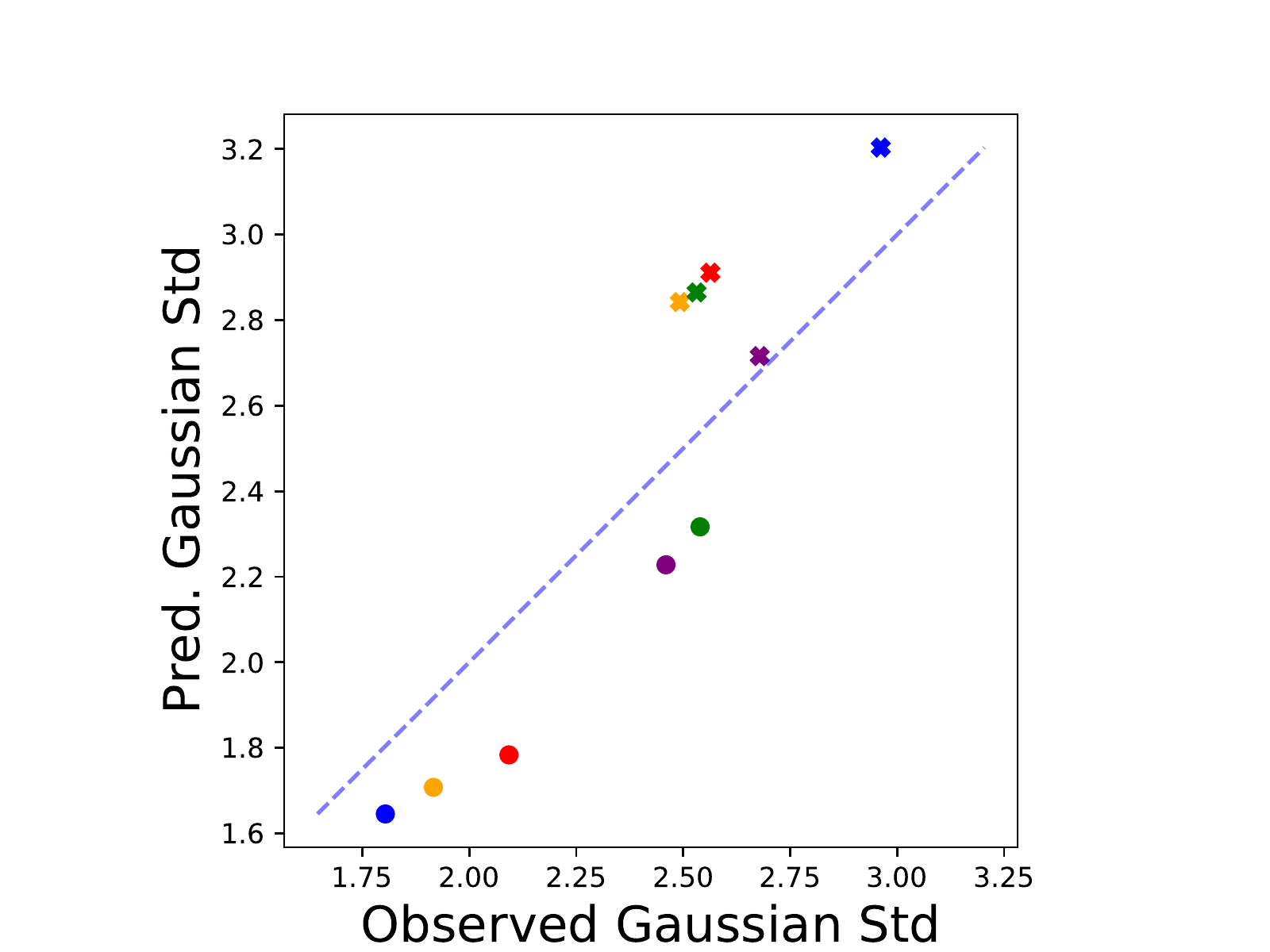}\\
       (g)  & (h)
    \end{tabular}
    \caption{Flowers: All plots show the observed statistics (x-axis) against predicted statistics (y-axis) of output value distributions for the different layers in the aggregation block. The left column shows the distribution mean values, whilst the right column shows the standard deviation values. Different colours represent different classes treated as positive. For reference, the diagonal line is shown as a dotted line.
    (a) and (b) for the Exponential Layer Activation Layer, (c) and (d) for the GAP layer, (e) and (f) for the deactivation layer and (g),(h) for the fully connected classification layer. }
    \label{fig:flower_mean_std}
\end{figure}

\begin{figure}
    \centering
    \begin{tabular}{cc}
         \includegraphics[width=0.45\linewidth]{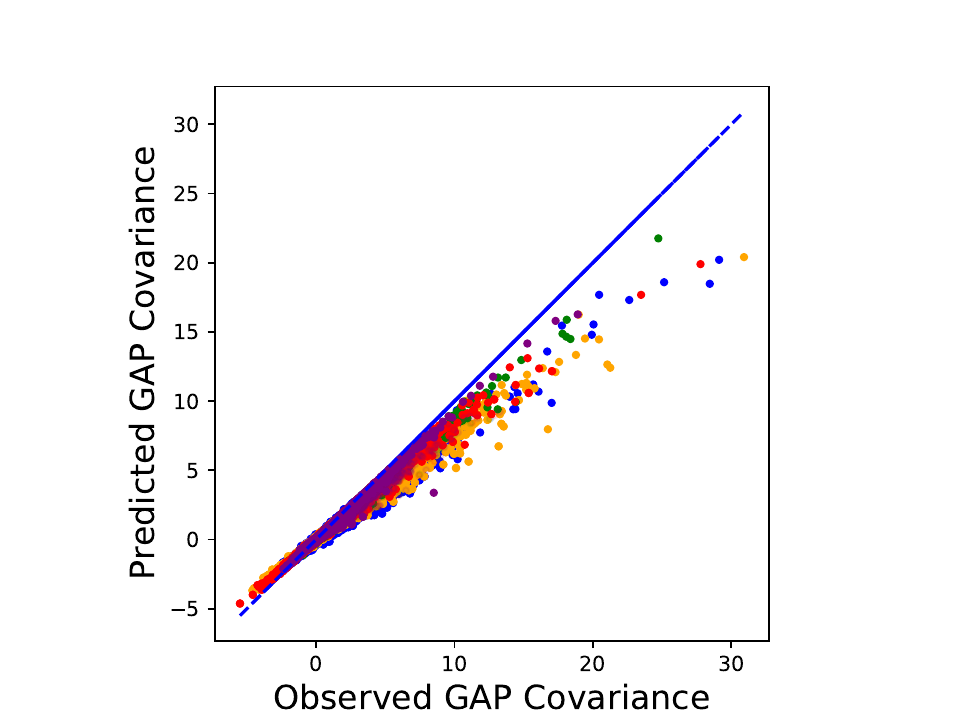} & 
         \includegraphics[width=0.75\linewidth]{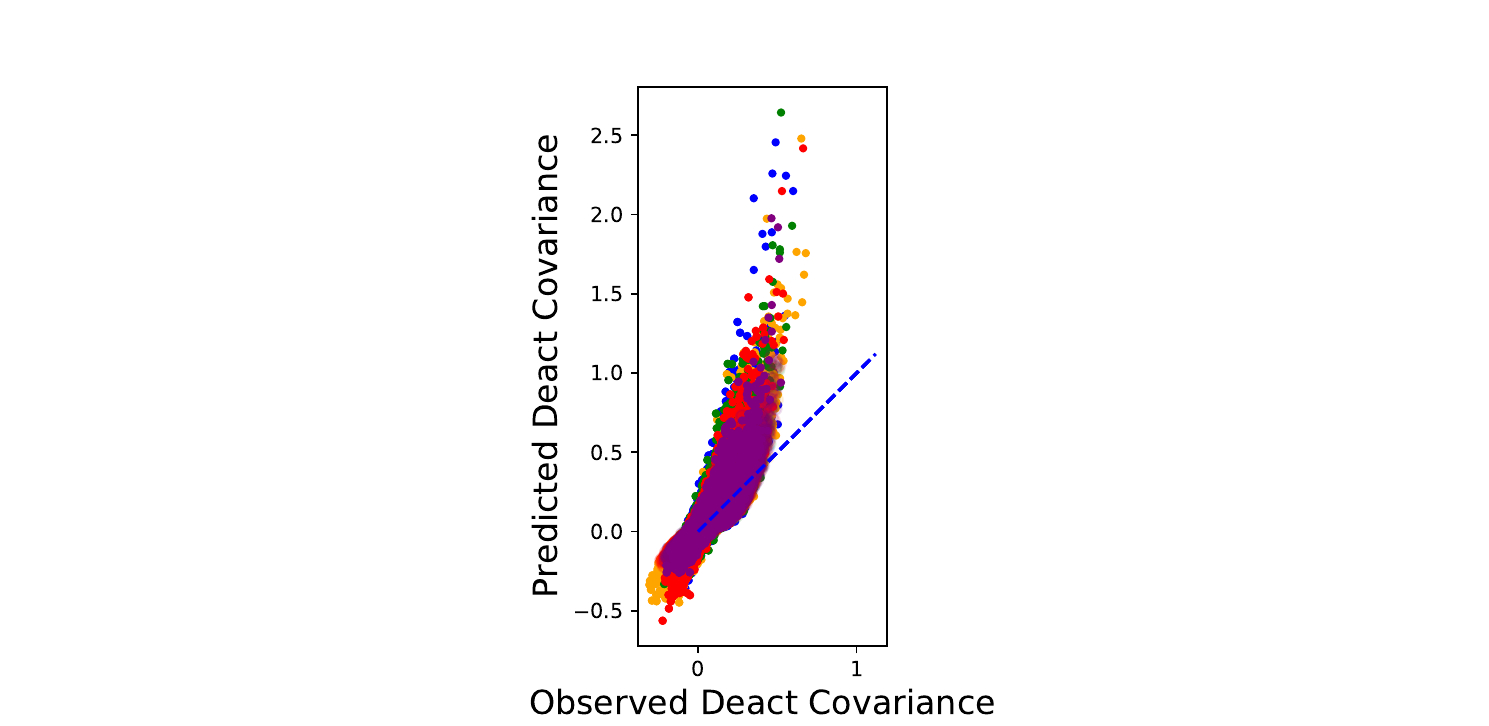}  \\
         (a) GAP Layer & (b) Deact. Layer 
    \end{tabular}
    \caption{Flowers: Aggregation layer covariance matrix scatter plot showing value of elements of the observed covariance matrix (x-axis) against corresponding element values of the predicted covariance matrix (y-axis).  }
    \label{fig:flower_cov_mat}
\end{figure}


\subsection{The Effect of Correlated Deep Features}
\label{sec:exp_cov_deep_feat}
We find that different pixels of the last convolutional layer output image are correlated, both within and between filters. Although the correlation is reduced as training progresses, it is still present and has a significant influence on the KL-divergence of the DNN classifier nodes.

Secondly, the covariance between deactivated GAP features plays a significant factor in determining the KL divergence of the DNN. This can be seen from the variance equations of Eq. \ref{eq:gapsum_var} for the GAP features and Eq. \ref{eq:out_sigma} for the output Gaussian. Here, we see that the covariance matrix directly impacts on the size of the variance of the final output distribution. The greater the covariance between different filters in the last convolutional layer, the bigger the classification node output value variances will be. This causes the KL divergence between the corresponding positive and negative distributions of an output classifier node to decrease. A consequence of ignoring the covariance matrix contributions results in a severe under-estimation of the variances. This can be seen in Figure \ref{fig:nocov}a. We can also visually see this issue in Figure \ref{fig:nocov}c-e. The resulting incorrect increase in the predicted KL divergence can be seen in Figure \ref{fig:nocov}b. This was found to be the case for all positive classes in the CIFAR10 dataset.

\begin{figure}
    \centering
    \begin{tabular}{cc}
    \includegraphics[width=0.45\linewidth]{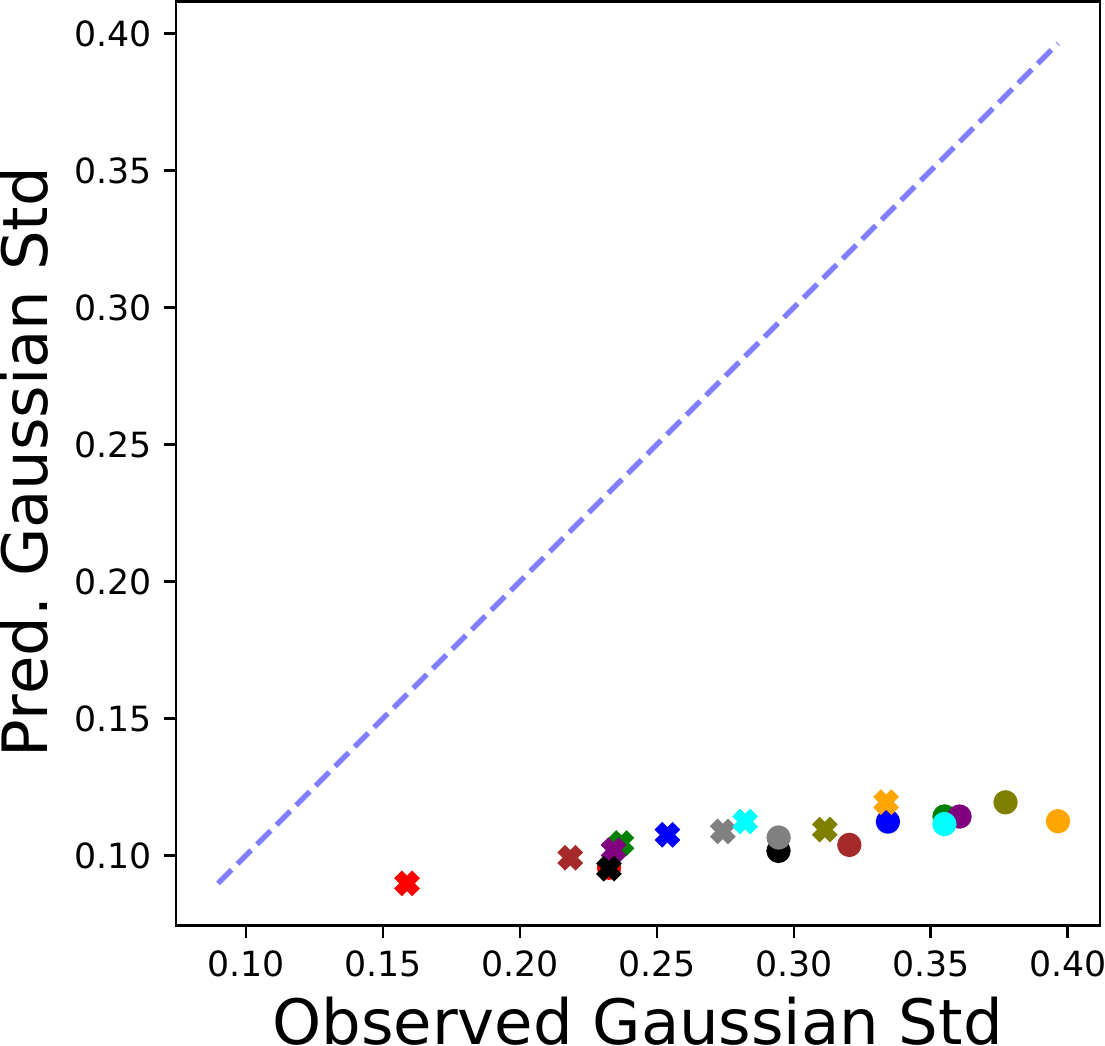}     &  \includegraphics[width=0.45\linewidth]{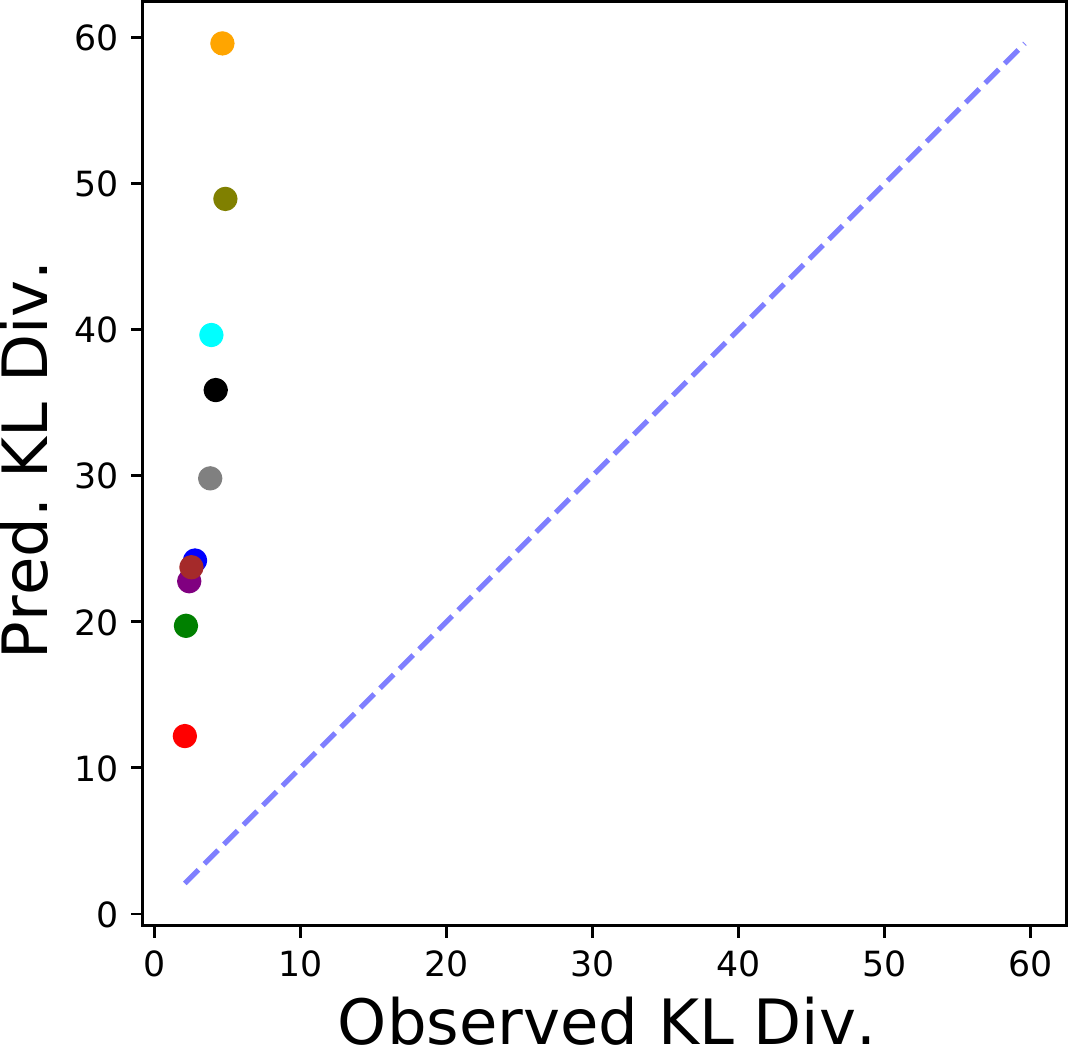}\\
       (a)  & (b)
    \end{tabular}
    \begin{tabular}{ccc}
         \includegraphics[width=0.3\linewidth]{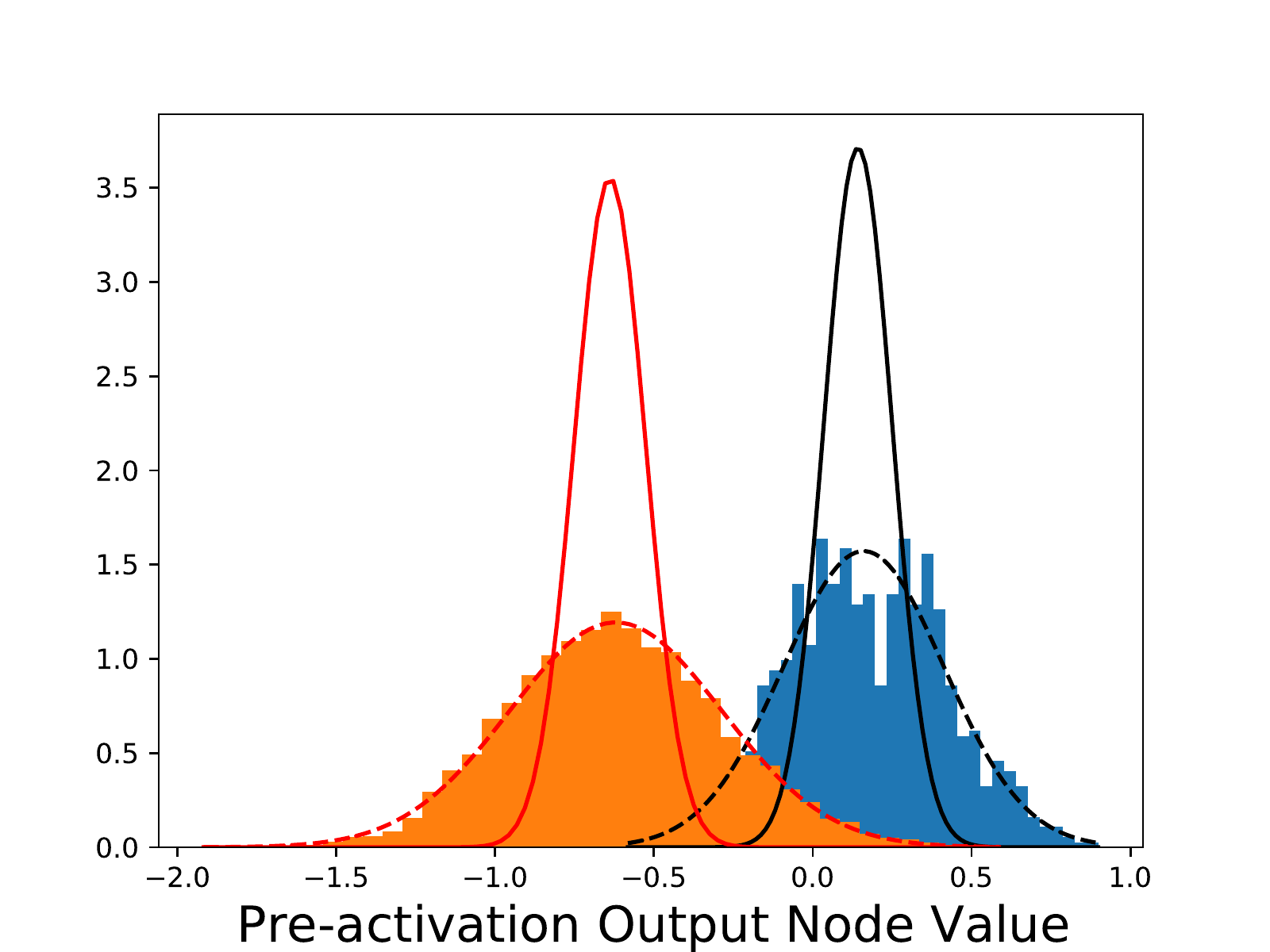} & \includegraphics[width=0.3\linewidth]{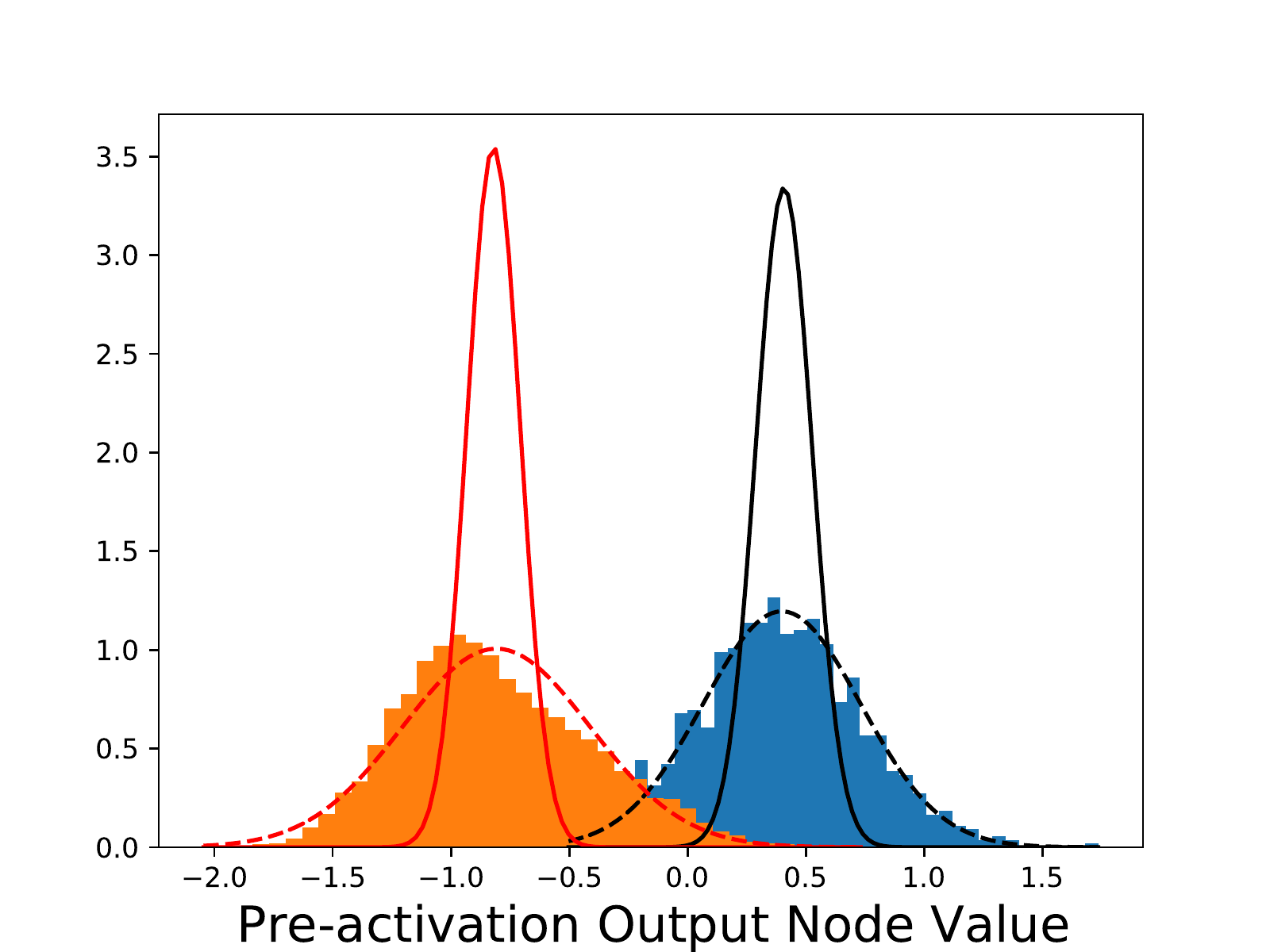}&\includegraphics[width=0.3\linewidth]{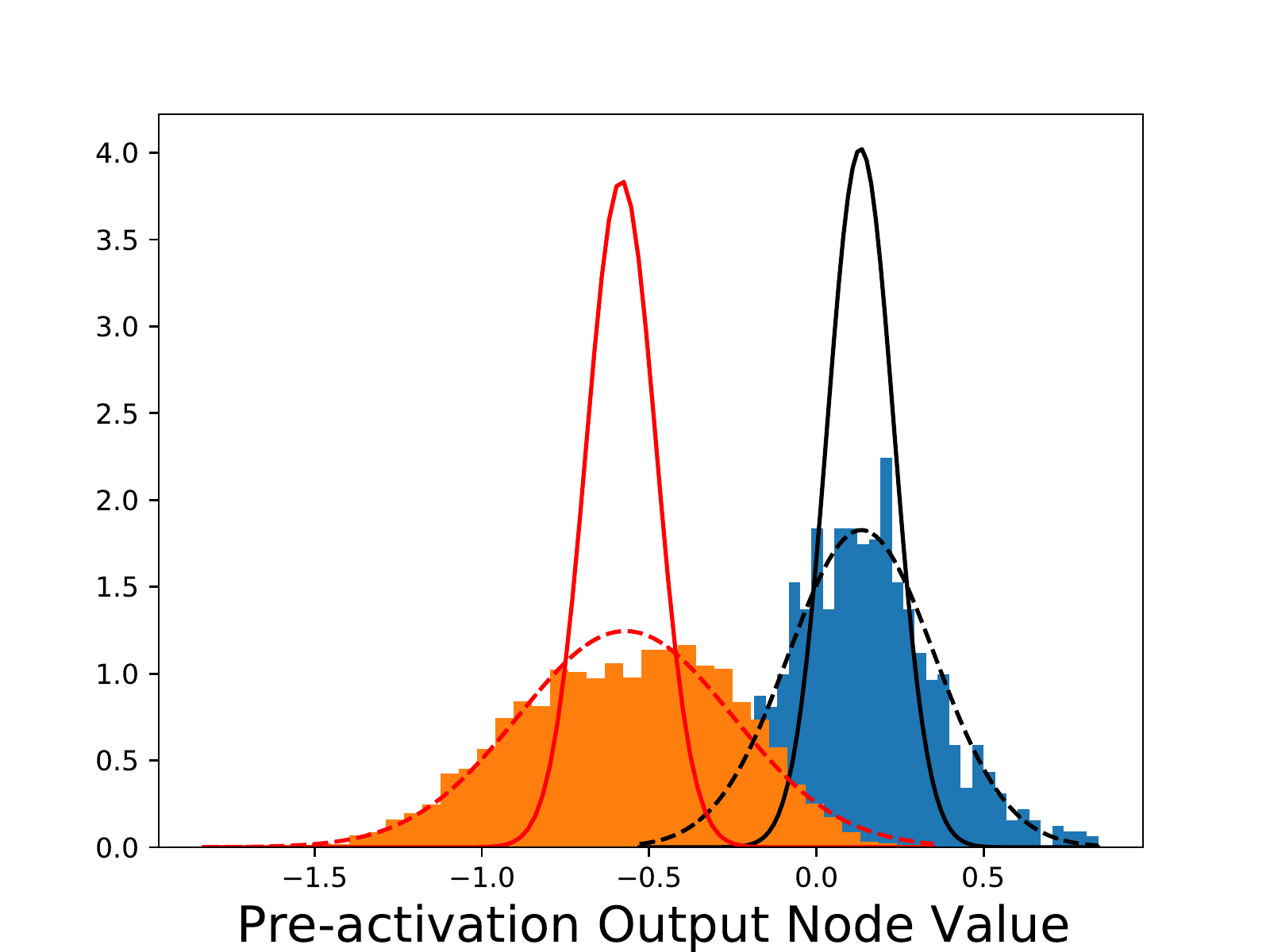}  \\
         (c) Cls. 1& (d) Cls. 6 & Cls. 9 (f) 
    \end{tabular}
    \caption{The consequence of ignoring the covariances between different GAP feature and deactivated features on the CIFAR10 dataset. The result are significant under estimation of the variances of the output node Gaussian distributions as shown in (a). Consequently, the predicted DNN KL divergence is much larger than its actual value (b) (compare with Fig. \ref{fig:cifar_flower_kl_div}a). Examples of the effect of underestimating the variances of the output node distributions can be seen in (c), (d) and (f). }
    \label{fig:nocov}
\end{figure}
\section{Conclusions}
\label{sec:conclusions}

In this paper, we have proposed a novel mathematical formulation that allowed for the analytical modelling of probability distributions of layer activations involved with nonlinear deep feature aggregation. 
In this paper, we have carried out a theoretical study of the distributions of activations in the aggregation process of a deep net. To achieve this, we proposed a novel mathematical formulation that allowed for the analytical modelling of probability distributions of layer activations involved with nonlinear deep feature aggregation.  
Using a mixed discrete-continuous distribution for modelling the output values of convolutional layers, we derived the output distributions for each layer in the aggregation block.
We then showed how these distributions can be used to model the output value distributions of a classification DNN. This immediately allowed us to obtain the KL divergence of the DNN. Importantly, we now have a direct link between the learnable parameters for layers starting from the aggregation block and the DNN KL divergence. From this, we made a theoretical observation of how the covariance of pre-aggregation deep features can influence the KL divergence and thus, classification performance of a deep net. In particular, we have found that an increase in covariance of convolutional layer activations leads to a decreased KL divergence, confirming existing intuition on the need for decorrelated deep features for good DNN accuracy. The model predictions were then verified on classification tasks from two different datasets and DNN models, where our models were shown to have good agreement with the experimental results.
\appendices

\section{Derivation of $\mathbb{E}[X^2]$ of Mixed-Gamma Distribution}
\label{app:mix_gamma_scnd_mnt_deriv}
The second moment of $X$ with distribution $f(x)$ (Eq. \ref{eq:zero_gamma_pdf}) is defined as:
\begin{eqnarray*}
\mathbb{E}[X^2] & = & \int^\infty_0 f(x) x^2 dx \\
& = & \int^\infty_0 (1-p)\frac{1}{\Gamma(a)s^a}x^{a-1}\exp(-x/s)x^2 dx
\end{eqnarray*}
The second equation above says that the second moment of $X$ is simply a weighted ($1-p$) 2nd moment of a Ganma distribution. To find the second moment, we can use the moment generating function (MGF) of a Gamma distribution. We find that the MGF of the Ganma distribution (with shape and scale parameters $a$ and $s$ respectively) is:
\[
M_X(t) = (1-ts)^{-a}
\]
The $n^{th}$ moment is the $n^{th}$ derivative of $M_X$, that is, $d^{(n)}M_X/dt$ evaluated at $t=0$. Thus, differentiating the above twice gives:
\[
\frac{d^2M}{dt^2} = a(a+1)s^2(1-ts)^{-(a-2)}
\]
which when evaluated at $t=0$ gives the second moment of a Gamma distribution as $a(a+1)s^2$. Thus, 
\begin{equation}
\mathbb{E}[X^2] = (1-p)a(a+1)s^2
\end{equation}

\section{Proof of Proposition \ref{thm:exp_gamma}}
\label{app:exp_gamma_pdf}
\begin{proposition}
Let $X$ be a random variable with distribution $f_\Gamma(x)$ in Eq. \ref{eq:gamma_pdf}. Let $g(x) = \alpha( \exp(\beta x) - 1)$ with $\alpha,\beta \in \mathbb{R}$, then $g(X)$ follows the following distribution:
\[
h(x) = \frac{1}{\alpha\Gamma(a)\beta^as^a}(\ln (1+x/\alpha) )^{a-1}(1+x/\alpha)^{-\left(\frac{\beta s +1}{\beta s}\right)}
\]
\end{proposition}
\begin{proof}
We will use the transformation of variables method, where:
\begin{equation}
h(x) = f_\Gamma(g^{-1}(x)) \frac{dg^{-1}(x)}{dx}
\label{eq:pdf_trans}
\end{equation}
We can use the above form since $g(x)$ has an inverse that is differentiable. Thus, $g^{-1}(x) = \ln (1 + x/\alpha )/\beta$ and so $dg^{-1}/dx = 1/(\alpha\beta (1 + x/\alpha))$. Substituting the above into $h(x)$ gives:
\begin{eqnarray*}
h(x) & = & f_\Gamma\left( \frac{\ln (1+x/\alpha)}{\beta}\right) \frac{1}{\alpha\beta (1+ x/\alpha)} \\
& = & \frac{1}{s\Gamma(a)}\left( \frac{\ln (1+x/\alpha)}{\beta s} \right)^{a-1}\exp\left(- \left(\frac{\ln (1+x/\alpha)}{\beta s} \right)\right) \times \\
&  & \left(\frac{1}{\alpha\beta(1+ x/\alpha)}\right) \\
& = & \frac{1}{\alpha \Gamma(a)\beta^a s^a}(\ln (1+x/\alpha)^{a-1}\exp\left(\ln (1+x/\alpha)^{-\frac{1}{\beta s}}\right) \times \\
&  & \left(\frac{1}{(1+ x/\alpha)}\right) \\
& = & \frac{1}{\alpha\Gamma(a) \beta^a s^a}(\ln (1+x/\alpha)^{a-1}(1+x/\alpha)^{-\frac{1}{\beta s}} (1+ x/\alpha)^{-1} \\
& = &\frac{1}{\alpha\Gamma(a)\beta^as^a}(\ln (1+x/\alpha) )^{a-1}(1+x/\alpha)^{-\left(\frac{\beta s +1}{\beta s}\right)}
\end{eqnarray*}
as required.
\end{proof}

\section{Derivation of Mean and Variance of Zero-ExpGamma}
\label{app:zero_expgamma_meanvar}
We will need to use the following integral result for the derivation of the mean and variance of the zero-ExpGamma distribution.
\begin{lemma}
Let $b, B \in \mathbb{R}$, with $B > 1$, then we have:
\begin{equation}
    \int^\infty_1 (\ln x )^bx^{-B} dx = \frac{\Gamma(1 + b)}{(B-1)^{b+1}}
\label{eq:int_lemma}
\end{equation}
\label{lem:int_lemma}
\end{lemma}
\begin{proof}
We first perform a change of variables: let $y = \ln x$, so $dy/dx = 1/x$, $x = e^y, y(1) = 0$ and $x^{-B} = e^{-By}$. The integral becomes:
\begin{eqnarray*}
\int^\infty_1 (\ln x)^bx^{-B} dx & = & \int^\infty_0 y^b x^{-B+1}dy \\
& = & \int^\infty_0 y^b e^{-(B-1)y} dy
\end{eqnarray*}
Performing change of variables again, with $w = (B-1)y$, so $dw/dy = B-1$, $y^b = w^b/(B-1)^b$ and the above integral becomes:
\begin{eqnarray*}
\int^\infty_0 y^b e^{-(B-1)y} dy & = & \int^\infty_0 (B-1)^{-b}w^be^{-w}dw \\
& = & (B-1)^{b-1}\int^\infty_0 w^b e^-w dw \\
& = & (B-1)^{b-1}\Gamma(1+b) \\
& = & \frac{\Gamma(1 + b)}{(B-1)^{b+1}}
\end{eqnarray*}
\end{proof}

To find the mean of a random variable $X_e$ following the exponentiated Gamma distribution, we have:
\begin{eqnarray*}
\mathbb{E}[X_e] & = & \int^\infty_0 h(x)x dx \\
& = & \int^\infty_0 \frac{\left(\ln(1 + \frac{x}{\alpha})\right)^{a-1}\left(1 + \frac{x}{\alpha}\right)^{-(1+\frac{1}{\beta s})}}{\alpha\Gamma(a) s^a\beta^a} x\; dx
\end{eqnarray*}
Let $w = 1+x/\alpha$, so $dw/dx = 1/\alpha$, $w(0) = 1$, $x = \alpha w - \alpha$
and so the above integral becomes:
\begin{eqnarray*}
\mathbb{E}[X_e] & = & \int^\infty_1 \frac{\left(\ln w\right)^{a-1}w^{-(1+\frac{1}{\beta s})}}{\alpha\Gamma(a) s^a\beta^a} (\alpha w - \alpha)\; \alpha dw \\
& = & \alpha \int^\infty_1 \frac{\left(\ln w\right)^{a-1}w^{-(1+\frac{1}{\beta s})}}{\Gamma(a) s^a\beta^a} w\; dw - \\
&&\alpha \int^\infty_1 \frac{\left(\ln w\right)^{a-1}w^{-(1+\frac{1}{\beta s})}}{\alpha\Gamma(a) s^a\beta^a}  dw
\end{eqnarray*}
The integral in the last line is over the entire exponentiated Gamma distribution, so will be 1, resulting in:
\begin{eqnarray*}
\mathbb{E}[X_e] & = & \alpha\int^\infty_1 \frac{\left(\ln w\right)^{a-1}w^{-(1+\frac{1}{\beta s})}}{\Gamma(a) s^a\beta^a}  w\; dw - \alpha \\
 &= & \frac{\alpha}{\Gamma(a)s^a\beta^a}\int^\infty_1 (\ln w)^{a-1}w^{-\frac{1}{\beta s}} dw - \alpha \\
\end{eqnarray*}
The integral in the last equation above can be evaluated using Eq. \ref{eq:int_lemma} by setting $b = a-1$ and $B = 1/(\beta s)$. This results in:
\begin{eqnarray*}
\mathbb{E}[X_e] & = & \frac{\alpha}{ \Gamma(a)s^a\beta^a}\left(\frac{\Gamma(a)}{\left(\frac{1}{\beta s} - 1\right)^a}\right) - \alpha \\
& = & \frac{\alpha}{ s^a\beta^a}\left( \frac{\beta s}{1 - \beta s}\right)^a - \alpha \\
 & = & \alpha\left(\frac{1}{(1-\beta s)^a} - 1\right)
\end{eqnarray*}

In order to find the variance, we first find the second moment:
\begin{eqnarray*}
\mathbb{E}[X_e^2] &= & \int^\infty_0 h(x)x^2 dx \\
& = & \int^\infty_0 \frac{\left(\ln(1 + \frac{x}{\alpha})\right)^{a-1}\left(1 + \frac{x}{\alpha}\right)^{-(1+\frac{1}{\beta s})}}{\alpha\Gamma(a) s^a\beta^a} x^2\; dx 
\end{eqnarray*}

As with the mean derivation, we perform a change of variable, where $w = 1 + x/\alpha$, so that the integral becomes:
\begin{eqnarray*}
\mathbb{E}[X_e^2] &= & \int^\infty_1 \frac{\left(\ln w\right)^{a-1}w^{-(1+\frac{1}{\beta s})}}{\alpha\Gamma(a) s^a\beta^a} (\alpha w - \alpha)^2\; \alpha dw \\
& = & \int^\infty_1 \frac{\left(\ln w\right)^{a-1}w^{-(1+\frac{1}{\beta s})}}{\alpha\Gamma(a) s^a\beta^a} \alpha^2 w^2\; \alpha dw - \\
&& \int^\infty_1 \frac{\left(\ln w\right)^{a-1}w^{-(1+\frac{1}{\beta s})}}{\alpha\Gamma(a) s^a\beta^a} 2(\alpha^2 w )\; \alpha dw + \\
&&\int^\infty_1 \frac{\left(\ln w\right)^{a-1}w^{-(1+\frac{1}{\beta s})}}{\alpha\Gamma(a) s^a\beta^a} (\alpha)^2\; \alpha dw \\
& = & \frac{\alpha^2}{\Gamma(a) s^a\beta^a}\int^\infty_1 \left(\ln w\right)^{a-1}w^{1-\frac{1}{\beta s}}\; dw - \\
&& 2\alpha^2\mathbb{E}[X_e] + \alpha^2
\end{eqnarray*}
The last integral above can be evaluated using Eq. \ref{eq:int_lemma} by setting $b = a-1$ and $B = 1 - 1/(\beta s)$, to give a result of $1/(1-2\beta s)^a$, and by substituting the mean into above we find that the second moment is:
\[
\mathbb{E}[X_e^2] = \alpha^2\left(
\frac{1}{(1-2\beta s)^a} - \frac{2}{(1-\beta s)^a} + 1
\right)
\]

Finally, the variance is found by substituting the mean and second moment equations into Eq. \ref{eq:gen_var}, resulting in:
\begin{eqnarray*}
Var(X_e) & = & \alpha^2\left(
\frac{1}{(1-2\beta s)^a} - \frac{2}{(1-\beta s)^a} + 1
\right) - \\
&& \left( \alpha\left(\frac{1}{(1-\beta s)^a} - 1\right) \right)^2 \\
& = & \alpha^2\left[ 
\frac{1}{(1-2\beta s)^a} - \frac{2}{(1-\beta s)^a} + 1 - \right.\\
&&\left. \frac{1}{(1-\beta s)^{2a}} + \frac{2}{(1-\beta s)^a} - 1
\right] \\
&= & \frac{\alpha^2}{(1-2\beta s)^a} - \frac{\alpha^2}{(1-\beta s)^{2a}}
\end{eqnarray*}

\section{Proof of Proposition \ref{thm:deact_pdf}}
\label{app:deact_pdf_proof}
\begin{proposition}
Let $X$ be a random variable with distribution $f(x)$ in Eq. \ref{eq:gap_pdf}. Let $g(x) = (x+\epsilon)^\gamma$ with $\gamma \in \mathbb{R}$, where $\epsilon>0$ is a small real constant, then $g(X)$ follows the following distribution:
\[
w(x) = \frac{e^{\epsilon/s}}{\gamma\Gamma(a)s^a}\left(\sum^\infty_{k=0} \binom{a-1}{k}x^{\frac{a - 1}{\gamma}-k}(-\epsilon)^k\right) \exp\left(-\frac{x^{1/\gamma}}{s}\right)
\]
\end{proposition}
\begin{proof}
We will follow the same steps as used in the proof of Thm. \ref{thm:exp_gamma} since $g(x)$ here is has a differentiable inverse. Therefore, $g^{-1}(x) = x^{1/\gamma} - \epsilon$ and $dg^{-1}/dx = (1/\gamma)x^{(1/\gamma) - 1}$. Then, Eq. \ref{eq:pdf_trans} ($w(x)$ taking the place of $h(x)$) becomes:
\begin{eqnarray*}
w(x) & = & \frac{1}{\Gamma(a)s^a}\left(x^{1/\gamma}- \epsilon \right)^{a-1} \exp\left(-\frac{x^{1/\gamma} - \epsilon}{s}\right)\left(\frac{1}{\gamma}\right)x^{\frac{1}{\gamma} - 1} 
\end{eqnarray*}
Taking the binomial expansion of $(x^{1/\gamma} - \epsilon)^{a-1}$ and moving the term $x^{\frac{1}{\gamma}-1}$ into this expansion gives:
\[
w(x) = \frac{e^{\epsilon/s}}{\gamma\Gamma(a)s^a}\left(\sum^\infty_{k=0} \binom{a-1}{k}x^{\frac{a - 1}{\gamma}-k}(-\epsilon)^k\right) \exp\left(-\frac{x^{1/\gamma}}{s}\right)
\]
\end{proof}

\section{Deactivation Distribution Mean and Variance}
\label{app:deact_mean_var}
The mean of $w(x)$ (Thm. \ref{thm:deact_pdf}), denoted as $\mu_w$ is found as follows:
\begin{eqnarray*}
\mu_w & = & \frac{1}{\gamma\Gamma(a)s^a}\int^\infty_0
x^{\frac{a-\gamma}{\gamma}}\exp\left(-\frac{x^{1/\gamma}}{s}\right) x\; dx\\
& = & \frac{1}{\gamma\Gamma(a)s^a}\int^\infty_0
x^{\frac{a}{\gamma}}\exp\left(-\frac{x^{1/\gamma}}{s}\right) dx
\end{eqnarray*}
Performing change of variables with $y = x^{1/\gamma}/s$, so $dy/dx = (1/(\gamma s))x^{1/\gamma - 1}$ and $w(0) = 0$, resulting in the integral:
\begin{eqnarray*}
\mu_w & = & \frac{1}{\gamma\Gamma(a)s^a}\int^\infty_0 \exp(-y) x^{a/\gamma} s x^{1 - 1/\gamma} dy \\
& = & \frac{1}{\gamma\Gamma(a)s^a}\int^\infty_0 \exp( -y) x^{1 + (a-1)/\gamma} dy \\
& = & \frac{1}{\gamma\Gamma(a)s^a}\int^\infty_0 \exp( -y) y^{\gamma + a - 1} dy \\
& = & \frac{s^\gamma\Gamma(\gamma + a)}{\Gamma(a)}
\end{eqnarray*}

To find the variance, we first find the second moment (denoting the random variable as $X_w$):
\begin{eqnarray*}
\mathbb{E}[X_w^2] & = & \frac{1}{\gamma\Gamma(a)s^a}\int^\infty_0
x^{\frac{a-\gamma}{\gamma}}\exp\left(-\frac{x^{1/\gamma}}{s}\right) x^2\; dx
\end{eqnarray*}
Similar to the mean, changing the variable with $w = x^{1/\gamma}/s$ results in the integral becoming:
\begin{eqnarray*}
\mathbb{E}[X_w^2] & = & \frac{s^{2\gamma}}{\Gamma(a)} \int^\infty_0 y^{a -1 + 2\gamma} e^{-y} dy \\
& = & \frac{s^{2\gamma}\Gamma(a + 2\gamma)}{\Gamma(a)}
\end{eqnarray*}
where the last line again used Eq. \ref{eq:int_lemma} from Lemma \ref{lem:int_lemma}.

 Finally, substituting the above second moment and mean into Eq. \ref{eq:gen_var} results in the required variance, denoted as $\sigma^2_e$:
\begin{eqnarray*}
\sigma^2_e & = & \frac{s^{2\gamma}\Gamma(a + 2\gamma)}{\Gamma(a)} - \left( \frac{s^\gamma\Gamma(\gamma + a)}{\Gamma(a)} \right)^2 \\
& = & \frac{s^{2\gamma}\Gamma(a + 2\gamma)}{\Gamma(a)} - \frac{s^{2\gamma}\Gamma( a + \gamma)^2}{\Gamma(a)^2} \\
& = & \frac{s^{2\gamma}}{\Gamma(a)}\left( \Gamma(a + 2\gamma) - \frac{\Gamma(a+\gamma)^2}{\Gamma(a)}\right)
\end{eqnarray*}

\section{Variance of Zero-ExpGamma}
\label{app:zero_expgamma_var}
To find the variance, we first find the second moment of $X_m$:
\begin{eqnarray}
\mathbb{E}[X^2_m] & = & \int^\infty_0 h_m(x)x^2 dx \nonumber\\
& = & (1-p)\int^\infty_0 h(x)x^2 dx \nonumber\\
& = & \alpha^2(1-p)\left(
\frac{1}{(1-2\beta s)^a} - \frac{2}{(1-\beta s)^a} + 1
\right) \nonumber
\end{eqnarray}
where the last line used the second moment obtained in the derivation of the variance for the exponentiated Gamma distribution.
Therefore, the variance of the mixed zero-ExpGamma distribution is:
\begin{eqnarray}
\sigma^2_m & = & \mathbb{E}[X^2_m] - \mathbb{E}[X_m]^2 \nonumber \\
& = & \alpha^2(1-p)\left(
\frac{1}{(1-2\beta s)^a} - \frac{2}{(1-\beta s)^a} + 1
\right) - \nonumber \\
&& \left( \alpha(1-p)\left(\frac{1}{(1-\beta s)^a} - 1\right)\right)^2 \\
& = & \alpha^2(1-p)\left[ \frac{1}{(1-2\beta s)^a} - \frac{1-p}{(1-\beta s)^{2a}} - \right. \\
&& \left. \frac{2p}{(1-\beta s)^a } + p \right]
\end{eqnarray}


%

%
%

\ifCLASSOPTIONcaptionsoff
  \newpage
\fi



%

%

\bibliographystyle{ieee_fullname}
\bibliography{egbib}

\begin{thebibliography}{10}\itemsep=-1pt

\bibitem{NetVLAD}
R. Arandjelovic, P. Gronat, A. Torii, T. Pajdla, and J. Sivic.
\newblock Net\textsc{VLAD}: \textsc{CNN} architecture for weakly supervised
  place recognition.
\newblock {\em IEEE Transactions on Pattern Analysis and Machine Intelligence},
  40(6):1437--1451, June 2018.

\bibitem{MAC}
H. {Azizpour}, A.~S. {Razavian}, J. {Sullivan}, A. {Maki}, and S. {Carlsson}.
\newblock From generic to specific deep representations for visual recognition.
\newblock In {\em 2015 IEEE Conference on Computer Vision and Pattern
  Recognition Workshops (CVPRW)}, pages 36--45, 2015.

\bibitem{brock2021characterizing}
Andrew Brock, Soham De, and Samuel~L Smith.
\newblock Characterizing signal propagation to close the performance gap in
  unnormalized resnets.
\newblock In {\em International Conference on Learning Representations}, 2021.

\bibitem{Xception}
F. {Chollet}.
\newblock Xception: Deep learning with depthwise separable convolutions.
\newblock In {\em 2017 IEEE Conference on Computer Vision and Pattern
  Recognition (CVPR)}, pages 1800--1807, 2017.

\bibitem{matthews2018gaussian}
Alexander~G. de G.~Matthews, Mark Rowland, Jiri Hron, Richard~E. Turner, and
  Zoubin Ghahramani.
\newblock Gaussian process behaviour in wide deep neural networks, 2018.

\bibitem{Fukushima1980}
Kunihiko Fukushima.
\newblock Neocognitron: A self-organizing neural network model for a mechanism
  of pattern recognition unaffected by shift in position.
\newblock {\em Biological Cybernetics}, 36(4):193--202, Apr 1980.

\bibitem{resnet}
K. He, X. Zhang, S. Ren, and J. Sun.
\newblock Deep residual learning for image recognition.
\newblock In {\em IEEE Conference on Computer Vision and Pattern Recognition},
  June 2016.

\bibitem{HusainPAMI}
S.~S. Husain and M. Bober.
\newblock Improving large-scale image retrieval through robust aggregation of
  local descriptors.
\newblock {\em IEEE Transactions on Pattern Analysis and Machine Intelligence},
  39(9):1783--1796, Sept 2017.

\bibitem{REMAP}
S.~S. {Husain} and M. {Bober}.
\newblock \textsc{REMAP}: Multi-layer entropy-guided pooling of dense cnn
  features for image retrieval.
\newblock {\em IEEE Transactions on Image Processing}, 28(10):5201--5213, 2019.

\bibitem{cifar}
Alex Krizhevsky.
\newblock Learning multiple layers of features from tiny images.
\newblock {\em University of Toronto}, 05 2012.

\bibitem{krizhevsky2012imagenet}
Alex Krizhevsky, Ilya Sutskever, and Geoffrey~E Hinton.
\newblock Imagenet classification with deep convolutional neural networks.
\newblock In {\em Advances in neural information processing systems}, pages
  1097--1105, 2012.

\bibitem{lee18}
Jaehoon Lee, Yasaman Bahri, Roman Novak, Sam Schoenholz, Jeffrey Pennington,
  and Jascha Sohl-dickstein.
\newblock Deep neural networks as gaussian processes.
\newblock In {\em Proc. of ICLR}, 2018.

\bibitem{lin2014network}
Min Lin, Qiang Chen, and Shuicheng Yan.
\newblock Network in network.
\newblock In {\em International Conference on Learning Representations,
  {ICLR}}, 2014.

\bibitem{nealthesis}
Radford~M Neal.
\newblock {\em Bayesian Learning For Neural Networks}.
\newblock PhD thesis, University of Toronto, The address of the publisher, 7
  1995.
\newblock An optional note.

\bibitem{Novak19}
Roman Novak, Lechao Xiao, Jaehoon Lee, Yasaman Bahri, Greg Yang◦and~Jiri
  Hron, Daniel~A. Abolafia, Jeffrey Pennington, and Jascha Sohl-Dickstein.
\newblock Bayesian deep convolutional networks with many channels are gaussian
  processes.
\newblock In {\em Proc. of ICLR}, 2019.

\bibitem{GEM}
F. Radenovic, G. Tolias, and O. Chum.
\newblock Fine-tuning \textsc{CNN} image retrieval with no human annotation.
\newblock {\em IEEE Transactions on Pattern Analysis and Machine Intelligence},
  pages 1--1, 2018.

\bibitem{simonyan2014very}
Karen Simonyan and Andrew Zisserman.
\newblock Very deep convolutional networks for large-scale image recognition.
\newblock {\em arXiv preprint arXiv:1409.1556}, 2014.

\bibitem{EFFNET}
Mingxing Tan and Quoc Le.
\newblock {E}fficient{N}et: Rethinking model scaling for convolutional neural
  networks.
\newblock In {\em Proceedings of the 36th International Conference on Machine
  Learning}, volume~97, pages 6105--6114, 2019.

\bibitem{ROIP}
Giorgos Tolias, Ronan Sicre, and Herv{\'e} J{\'e}gou.
\newblock Particular object retrieval with integral max-pooling of \textsc{CNN}
  activations.
\newblock {\em CoRR}, 2015.

\bibitem{maxpool_weng}
J.J. Weng, N. Ahuja, and T.S. Huang.
\newblock Learning recognition and segmentation of 3-d objects from 2-d images.
\newblock In {\em 1993 (4th) International Conference on Computer Vision},
  pages 121--128, 1993.

\end{thebibliography}








\end{document}